\documentclass[a4paper,11pt,twoside]{article}
\usepackage[a4paper,hscale=0.7,vscale=0.75,centering]{geometry}
\usepackage{fullpage}

\RequirePackage[colorlinks,citecolor=blue,urlcolor=blue]{hyperref}
\usepackage[utf8]{inputenc}
\usepackage[T1]{fontenc}
\usepackage[english]{babel}
\usepackage{charter}
\usepackage{enumerate}
\usepackage{indentfirst}
\usepackage{xcolor}
\usepackage{authblk}
\usepackage{mathtools}
\usepackage{xfrac}
\usepackage{bm}

\usepackage{amsmath}
\usepackage{amsthm}	
\usepackage{amsfonts}
\usepackage{mathrsfs}	
\usepackage{amssymb}
\usepackage{dsfont}
\usepackage{bbm}
\usepackage[numbers]{natbib}
\usepackage{algorithmic,algorithm}

\theoremstyle{plain}
\newtheorem{theorem}{Theorem}[section]
\newtheorem{lemma}[theorem]{Lemma}

\theoremstyle{definition}

\newtheorem{example}[theorem]{Example}
\newtheorem{assumption}[theorem]{Assumptions}
\newtheorem{assumptionsingle}[theorem]{Assumption}

\theoremstyle{definition}

\numberwithin{equation}{section}
\numberwithin{figure}{section}
 \usepackage[nodayofweek]{datetime}

\usepackage[nodayofweek]{datetime}

\newcommand{\bE}{\mathbb{E}}

\newcommand{\bR}{\mathbb{R}}

\newcommand{\cA}{\mathcal{A}}

\def \a {\alpha}

\title{Multi-index Antithetic Stochastic Gradient Algorithm}
\author[1]{Mateusz B.\ Majka \thanks{m.majka@hw.ac.uk}}
\author[3]{Marc Sabate-Vidales \thanks{m.sabate-vidales@sms.ed.ac.uk}}
\author[2,3]{\L ukasz Szpruch\thanks{l.szpruch@ed.ac.uk}}

\affil[1]{School of Mathematical and Computer Sciences, Heriot-Watt University}
\affil[2]{The Alan Turing Institute, London}
\affil[3]{School of Mathematics, University of Edinburgh}

\date{ }

\begin{document}
\selectlanguage{english}
\maketitle

\begin{abstract}

Stochastic Gradient Algorithms (SGAs) are ubiquitous in computational statistics, machine learning and optimisation. Recent years have brought an influx of interest in SGAs, and the non-asymptotic analysis of their bias is by now well-developed. However, relatively little is known about the optimal choice of the random approximation (e.g mini-batching) of the gradient in SGAs as this relies on the analysis of the variance and is problem specific. 
While there have been numerous attempts to reduce the variance of SGAs, these typically exploit a particular structure of the sampled distribution 
by requiring a priori knowledge of its density's mode. It is thus unclear how to adapt such algorithms to non-log-concave settings.
In this paper, we construct a Multi-index Antithetic Stochastic Gradient Algorithm (MASGA) whose implementation is independent of the structure of the target measure and which achieves performance on par with Monte Carlo estimators that have access to unbiased samples from the distribution of interest. In other words, MASGA is an optimal estimator from the mean square error-computational cost perspective within the class of Monte Carlo estimators.
We prove this fact rigorously for log-concave settings and verify it numerically for some examples where the log-concavity assumption is not satisfied.
\end{abstract}

\vspace{-5mm}

\section{Introduction}

Variations of Stochastic Gradient Algorithms (SGAs) are central in many modern machine learning applications such as large scale Bayesian inference \cite{welling2011bayesian}, variational inference \cite{Hoffman2013}, generative adversarial networks \cite{goodfellow2014generative}, variational autoencoders \cite{kingma2013auto} and deep reinforcement learning \cite{mnih2015human}. Statistical sampling perspective provides a unified framework to study non-asymptotic behaviour of these algorithms, which is the main topic of this work. More precisely, consider a data set $D=(\xi_i)_{i=1}^{m} \subset \mathbb{R}^n$, with $m \in \mathbb N \cup \{ \infty \}$ and the corresponding empirical measure $\nu^m:=\frac{1}{m}\sum_{i=1}^{m}\delta_{\xi_i}$, where $\delta$ is a Dirac measure. Denote by $\mathcal{P}(\mathbb R^n)$ the space of all probability measures on $\mathbb{R}^n$ and consider a potential $V:\mathbb R^d \times \mathcal P(\mathbb R^n) \rightarrow \mathbb R$. We are then interested in the problem of sampling from the (data-dependent) probability distribution $\pi$ on $\mathbb{R}^d$, given by 
\begin{equation}\label{eq gibbs} \textstyle 
\pi(x) \propto \exp\left(-\frac{2}{\beta^2}V(x,\nu^m)\right)dx
\end{equation}  
for some fixed parameter $\beta > 0$. 
Under some mild assumptions on $V$, the measure $\pi$ is a stationary distribution of the (overdamped) Langevin stochastic differential equation (SDE) and classical Langevin Monte Carlo \cite{Dalalyan2017} algorithms utilise discrete-time counterparts of such SDEs to provide tools for approximate sampling from $\pi$, which, however, require access to exact evaluations of $\nabla V(\cdot, \nu^m)$. On the other hand, 
SGAs take as input a noisy evaluation $\nabla V(\cdot,\nu^s)$ for some $s \in \{1, \ldots m \}$.
The simplest example of $\nabla V(\cdot,\nu^s)$ utilizes the subsampling with replacement method.
Namely, consider a sequence of i.i.d.\ uniformly distributed random variables $\tau^k_i\sim \operatorname{Unif}(\{ 1,\cdots,m \})$ for $k \geq 0$ and $1\leq i \leq s$ and define a sequence of random data batches $D_s^k := (\xi_{\tau^k_i})_{i=1}^{s}$ and corresponding random measures $\nu_s^k:=\frac{1}{s}\sum_{i=1}^{s}\delta_{\xi_{\tau^k_i}}$ for $k \geq 0$. Fix a learning rate (time-step) $h>0$. The corresponding algorithm to sample from \eqref{eq gibbs} is given by
\begin{equation}\label{eq SGLDm} \textstyle 
X_{k+1} = X_k - h \nabla_x V(X_k,\nu_s^k) + \beta \sqrt{h} Z_{k+1},
\end{equation}
where $(Z_i)_{i=1}^{\infty}$ are i.i.d.\ random variables with the standard normal distribution. 
This method in its simplest form, without mini-batching 
(i.e., when we use the exact evaluation $\nabla V(\cdot,\nu^m)$) 
is known in computational statistics as the Unadjusted Langevin Algorithm (ULA) \cite{DurmusMoulines2017}, but it has numerous more sophisticated variants and alternatives \cite{ChatterjiBartlettJordan2018, CornishDoucet2019, BrosseDurmusMoulines2018, MaChenJinFlamarionJordan2019, MajkaMijatovicSzpruch2018}. 

Numerical methods based on Euler schemes with inaccurate (randomised) drifts such as (\ref{eq SGLDm}) have recently become an object of considerable interest in both the computational statistics and the machine learning communities \cite{Raginsky2017, Gao2018, ChauRasonyi2019, DenizAkyildiz2020, ZouXuGu2019, AicherMaFotiFox2019, NemethFearnhead2019, MaChenFox2015}. In particular, the Stochastic Gradient Langevin Dynamics (SGLD) method for approximate sampling from invariant measures of Langevin SDEs has been studied e.g.\ in \cite{welling2011bayesian, TehThieryVollmer2016, VollmerZygalakisTeh2016, Dubey2016, BrosseDurmusMoulines2018, Barkhagen2018, ChauMoulines2019, Zhang2019}.
Furthermore, recall that under some mild assumptions on $V$, when $\beta \rightarrow 0$, the measure $\pi$ concentrates on the set of minima of $V$, i.e., on $\{x\in\mathbb R^d: x = \arg\inf V(\cdot,\nu^m)\} $, cf.\ \cite{hwang1980laplace}. Remarkably, no convexity of $V$ is required for this to be true, which makes SGAs good candidates for a tool for non-convex optimisation.
We would like to stress that throughout the paper, in our analysis of algorithm \eqref{eq SGLDm}, we allow for $\beta = 0$ and hence we cover also the Stochastic Gradient Descent (SGD) \cite{BachMoulines2011, Dalalyan2017user}.

Despite the great success of algorithm (\ref{eq SGLDm}) and its various extensions, relatively  little is known about how to optimally choose $s$ and whether sub-sampling (or mini-batching) is a good idea at all \cite{TrueCost}. One of the reasons is that performance analysis appears to be problem-specific as we shall demonstrate on a simple example below. It is clear that subsampling increases the variance of the estimator and induces an additional non-asymptotic bias, see e.g.\ \cite{TrueCost, BrosseDurmusMoulines2018}. Therefore it is not clear that the reduced computational cost of running the algorithm compensates for these adverse effects. On the other hand, in a big data regime it may be computationally infeasible to use all data points at every step of the gradient algorithm and hence subsampling becomes a necessity.

In the present paper we propose a solution to these challenges by constructing a novel Multi-index Antithetic Stochastic Gradient Algorithm (MASGA). 
In settings where the measure $\pi$ in \eqref{eq gibbs} is log-concave, we will rigorously demonstrate that MASGA performs on par with Monte Carlo estimators having access to unbiased samples from the target measure, even though it consists of biased samples. 
Remarkably, our numerical results in Section \ref{sectionNumerics} demonstrate that this optimal performance is achieved even in some non-log-concave settings. To our knowledge, all current state-of-the-art SGAs \cite{CornishDoucet2019, BakerFearnheadFox2019} require the user to a priori determine the mode of the target distribution and hence it is not clear how to implement them in non-log-concave settings. This problem is absent with MASGA, whose implementation is independent of the structure of the target measure. Moreover, the analysis in \cite{CornishDoucet2019} is based on the Bernstein-von Mises phenomenon, which describes the asymptotic behaviour of the target measure as $m \to \infty$, and hence their algorithm is aimed explicitly  at the big data regime, see Section 3 therein. Meanwhile, as we will discuss below, MASGA works optimally irrespectively of the size of the dataset.

\paragraph{Mean-square error analysis.} In the present paper we are studying the problem of computing 
\[ \textstyle
(f,\pi) := \int_{\mathbb R^d} f(x)\pi(dx)\,,
\]
for some $f\in L^2(\pi)$. This framework covers the tasks of approximating minima of possibly non-convex functions or the computation of normalising constants in statistical sampling. 
 To this end, the Markov chain specified by \eqref{eq SGLDm} is used to approximate $(f,\pi)$ with $\mathbb E[f(X_k)]$ for large $k > 0$. More precisely, one simulates $N > 1$ independent copies $(X^i_k)_{k=0}^{\infty}$, for $i \in \{ 1, \ldots , N \}$, of (\ref{eq SGLDm}), to compute the empirical measure $\mu_{N,k} := \frac{1}{N}\sum_{i=1}^{N} \delta_{X^i_k}$. The usual metric for measuring the performance of such algorithms is the (root) mean square error (MSE). Namely, for any $f \in L^2(\pi)$, $k \geq 1$ and $N \geq 1$, we define 
\begin{equation*}
\operatorname{MSE}(\mathcal{A}^{f,k,N}) := \left( \mathbb{E} \left| (f, \pi) -  (f,\mu_{N,k})\right|^2 \right)^{1/2} \,,
\end{equation*}
where $\mathcal{A}^{f,k,N}$ is the algorithm specified by the estimator $(f,\mu_{N,k})$. Then, for a given $\varepsilon > 0$, we look for the optimal number $k$ of steps and the optimal number $N$ of simulated paths, such that for any fixed integrable function $f$ we have $\operatorname{MSE}(\mathcal{A}^{f,k,N}) < \varepsilon$. Note that
\begin{equation}\label{eq MSE decomp}
\operatorname{MSE}(\mathcal{A}^{f,k,N}) \leq \left| (f, \pi)  - (f, \mathcal L(X_k)) \right| + \left(N^{-1}\mathbb V[f(X_k)]\right)^{1/2}\,.
\end{equation}
If $f$ is Lipschitz with a Lipschitz constant $L$, then the Kantorovich duality representation of the $L^1$-Wasserstein distance $W_1$, see \cite[Remark 6.5]{villani2009optimal}, allows us to upper bound the first term of the right hand side by $L W_1(\mathcal L (X_k), \pi)$. Hence it is possible to control the bias by using the vast literature on such Wasserstein bounds (see e.g.\ \cite{ChengJordan2018, DurmusEberle2021, MajkaMijatovicSzpruch2018,BrosseDurmusMoulines2018} and the references therein). Controlling the variance, however, is a more challenging task. Before we proceed, let us consider an example.

\paragraph{Motivational example}
In the context of Bayesian inference, one is interested in sampling from the posterior distribution $\pi$ on $\mathbb{R}^d$ given by
\begin{equation*} \textstyle
\pi(dx) \propto \pi_0(dx) \prod_{i=1}^{m} \pi(\xi_i | x)  \,,
\end{equation*}
where the measure $\pi_0(dx) = \pi_0(x)dx$ is called the prior and $(\xi_i)_{i=1}^{m}$ are i.i.d.\ data points with densities $\pi(\xi_i | x)$ for $x \in \mathbb{R}^d$. 
Note that in this example, for convenience, we assume that the data are conditionally independent given the parameters. See Section \ref{sectionM} for more general settings.
 Taking $\beta = \sqrt{2}$, the potential $V$ in \eqref{eq gibbs} becomes
\begin{equation}\label{introduction:BayesianDrift} \textstyle
 V(x,\nu^m) := - \log \pi_0(x) - m \int \log \pi(y | x) \nu^m(dy) = - \log \pi_0(x) - \sum_{i=1}^{m} \log \pi(\xi_i | x) \,.
\end{equation}
In stochastic gradient algorithms, one replaces the exact gradient $\nabla V(\cdot,\nu^m)$ in \eqref{eq SGLDm} with an approximate gradient, constructed by sampling only $s \ll m$ terms from the sum in (\ref{introduction:BayesianDrift}). This amounts to considering $V(x,\nu^s) = - \log \pi_0(x) - m \int \log \pi(y | x) \nu^s(dy) = - \log \pi_0(x) - \frac{m}{s}\sum_{i=1}^{s} \log \pi(\xi_{\tau_i^k} | x)$, where $\tau_i^k \sim \operatorname{Unif}(\{ 1, \ldots, m \})$ for $i = 1, \ldots , s$ are i.i.d.\ random variables uniformly distributed on $\{ 1, \ldots, m \}$. Note that for any $x \in \mathbb{R}^d$ the noisy gradient $\nabla_x V(x,\nu^s)$ is an unbiased estimator of $\nabla_x V(x,\nu^m)$. If we choose $\beta = \sqrt{2}$ and the time-step $h/(2m)$ in (\ref{eq SGLDm}), we arrive at the Markov chain 
\begin{equation}\label{introduction:BayesianStochasticGradientChain2}
\textstyle
X_{k+1} = X_k + \frac{1}{2} h \left( \frac{1}{m}\nabla \log \pi_0(X_k) + \frac{1}{s} \sum_{i=1}^{s} \nabla \log \pi(\xi_{\tau_i^k} | X_k)\right) + \frac{1}{\sqrt{m}}\sqrt{h} Z_{k+1} \,.
\end{equation}
Note that the term in brackets in \eqref{introduction:BayesianStochasticGradientChain2} is an unbiased estimator of $- \frac{1}{m}\nabla V$.

The main difficulty in quantifying the cost of Monte Carlo algorithms based on \eqref{introduction:BayesianStochasticGradientChain2} stems from the fact that the variance is problem-specific and depends substantially on the interplay between the parameters $m$, $s$, $h$ and $N$. Hence, one may obtain different costs (and thus different answers to the question of profitability of using mini-batching) for different models and different data regimes \cite{TrueCost}.

 In Figures \ref{fig subsampling log} and \ref{fig subsampling exp} we present a numerical experiment for a simple example of a Monte Carlo estimator $(f,\mu_{N,k})$ based on the chain (\ref{introduction:BayesianStochasticGradientChain2}) with the densities $\pi(\xi_i | x)$ specified by a Bayesian logistic regression model, cf.\ Section \ref{sectionNumerics} for details. We take $m = 512$, $h = 10^{-2}$ and we simulate up to $t = 2$, hence we have $k = 200$ iterations. On the left-hand side of both figures we can see the estimated MSE for different numbers of paths $N$ and different numbers of samples $s \leq m$, for two different functions $f$. On the right hand side we can see how the variance changes with $s$. Evidently, subsampling/mini-batching works better for $f(x) = \log|x|$ than for $f(x) = \exp(x)$, since in the former case it allows us to obtain a small MSE by using just two samples, while in the latter the minimal reasonable number of samples seems to be $16$. However, even in this simple example we see that the optimal choice of $s$ and $N$ is far from obvious and very much problem-specific. For an additional discussion on this subject, see Appendix \ref{section:Examples}.

\begin{figure}[h]
	\centering
	\includegraphics[width=0.4\textwidth]{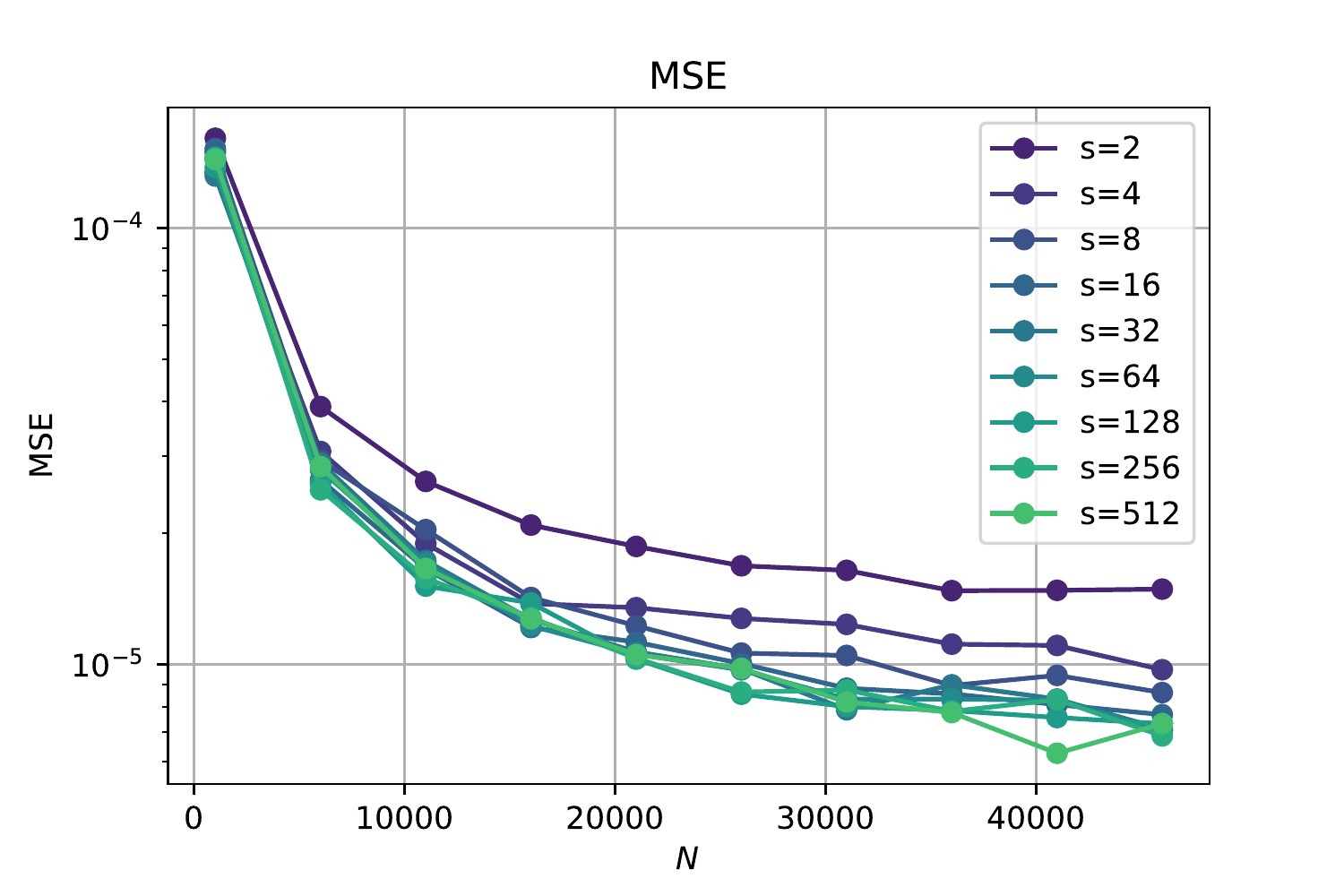}
	\includegraphics[width=0.4\textwidth]{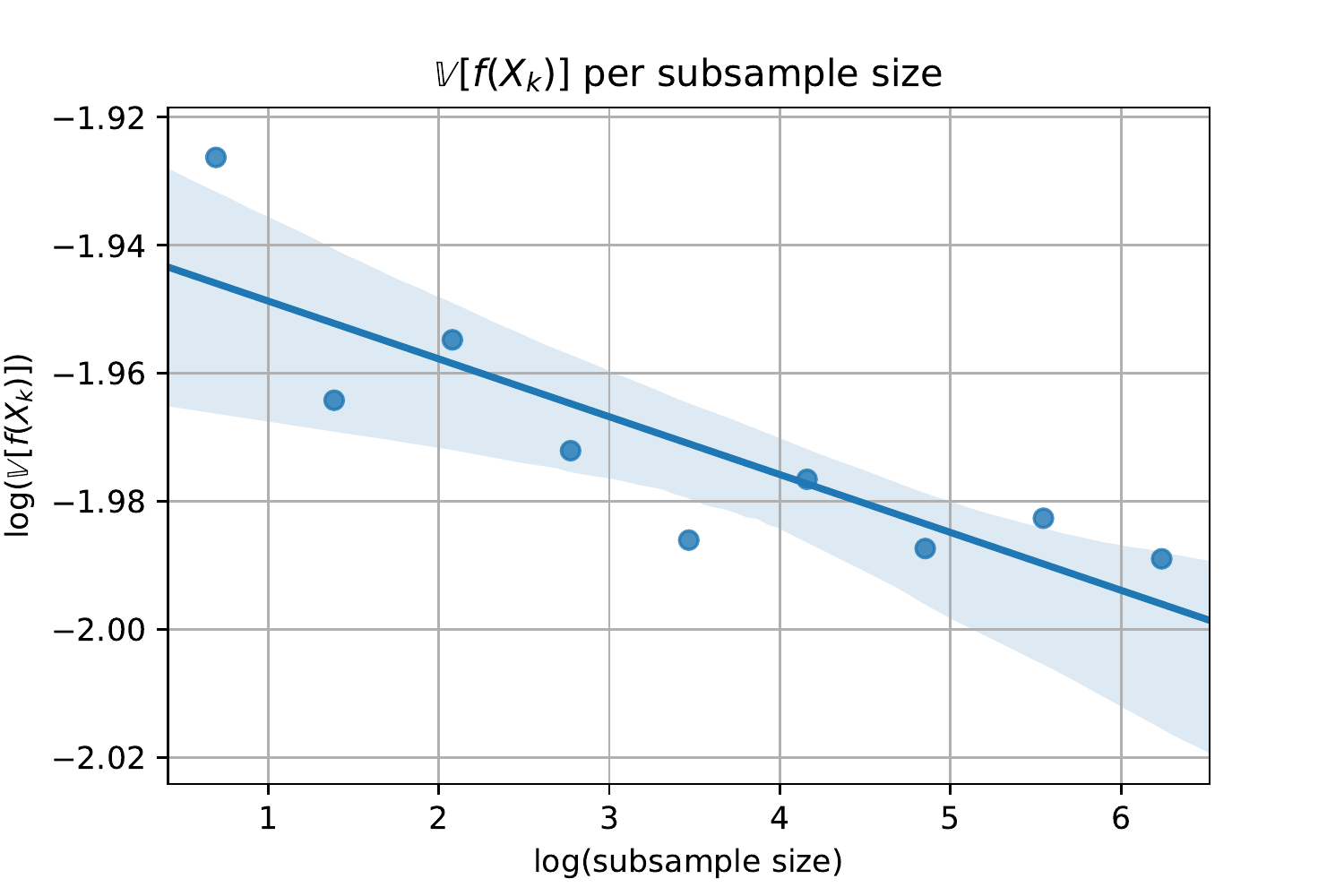}
	\caption{MSE and Variance $\mathbb V[f(X_k)]$ with $f(x) = \log|x|$}
	\label{fig subsampling log}
\end{figure}

\begin{figure}[h]
	\centering
	\includegraphics[width=0.4\textwidth]{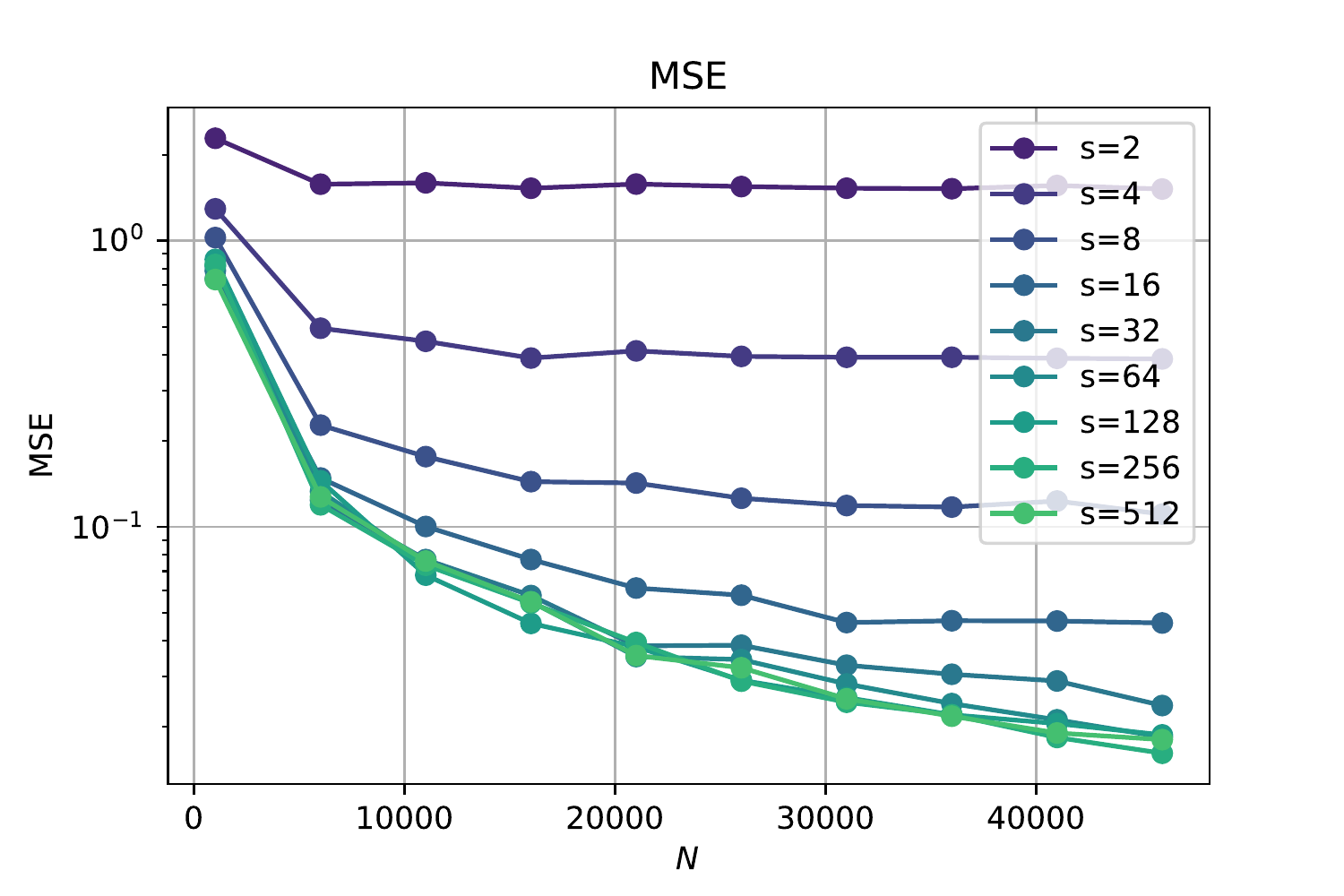}
	\includegraphics[width=0.4\textwidth]{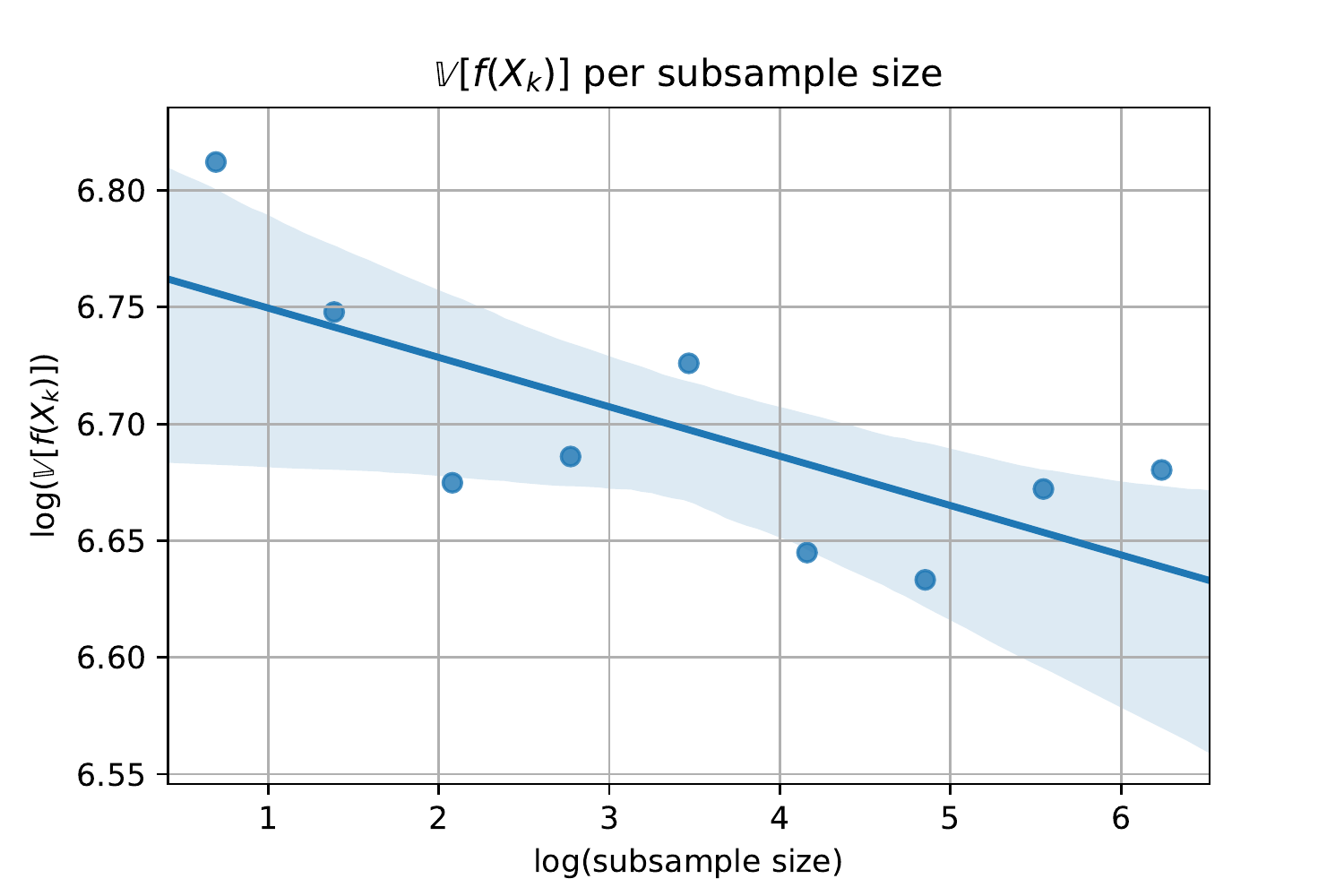}
	\caption{MSE and Variance $\mathbb V[f(X_k)]$ with $f(x) = \exp(x)$}
	\label{fig subsampling exp}
\end{figure}

It has been observed in \cite{TrueCost}, that in some specific regimes there is no benefit of employing the subsampling scheme. The authors of \cite{TrueCost} also observed that a subsampling scheme utilizing control variates can exhibit improved performance, but again only in some specific regimes (see also \cite{BrosseDurmusMoulines2018, BakerFearnheadFox2019} for related ideas). Moreover, in order to implement their scheme one has to know the mode of the sampled distribution in advance and hence it is not 
clear how to adapt it to non-convex settings.
 
The fact that the analysis of the variance of SGLD (and hence of the computational cost of the algorithm $\mathcal{A}^{f,k,N}$) becomes cumbersome even in the simplest possible examples, clearly demonstrates the need for developing a different algorithm, for which the benefits of subsampling could be proven in a rigorous way for a reasonably large class of models. To this end, we turn towards the Multi-level Monte Carlo (MLMC) technique.

\section{Main result}\label{section:mainResult}

In order to approximate the measure $\pi$, we consider a family of (possibly random) measures $(\pi^{\bm \ell})_{\bm\ell \in \mathbb N^r}$, 
where $\mathbb{N}^r$ for $r \geq 1$ is the set of multi-indices $\bm \ell = (\ell_1,\cdots,\ell_r)$, where each $\ell_i \geq 0$ corresponds to a different type of approximation. In this work we focus on the case $r=2$, with $\ell_1$ dictating the number of subsamples at each step of the algorithm and $\ell_2$ the time discretisation error. While for a fixed $\bm \ell$ the samples are biased, we work with methods that are asymptotically unbiased in the sense that
\begin{equation}\label{eq weakerror} \textstyle
\lim_{\bm \ell \rightarrow \infty} \mathbb{E} ( f , \pi^{\bm \ell}) = (f, \pi) \,,
\end{equation}
where $\bm \ell \rightarrow \infty$ means that $\min_{i\in{1,\ldots,r}}\ell_i\rightarrow \infty$. Note that we need to take the expectation in \eqref{eq weakerror} since the measures $\pi^{\bm \ell}$ can be random. We also remark that as the coordinates of $\bm \ell$ increase and the bias decreases, the corresponding computational cost of MLMC increases. One therefore faces an optimisation problem and tries to obtain the minimal computational cost for a prescribed accuracy (or, equivalently, to minimise the bias for a fixed computational budget).   

It turns out, perhaps surprisingly, that a Multi-level Monte Carlo estimator that consists of a hierarchy of biased approximations can achieve computational efficiency of vanilla Monte Carlo built from directly accessible unbiased samples \cite{MG08,Giles2015Acta}. In order to explain this approach, let us define backward difference operators
\[ \textstyle
\Delta_{s}\pi^{\bm \ell} := \pi^{\bm \ell} - \pi^{\bm \ell - \bm e_s}\,, \quad 
\bm \Delta \pi^{\bm \ell}:= \left(\prod_{s=1}^{r}\Delta_{s} \right)\pi^{\bm \ell}\,,
\]
where $\bm e_s$ is the unit vector in the direction $s \in\{1\,\ldots,r\}$ and $\prod_{s=1}^{r}\Delta_{s}$ denotes the concatenation of operators. The core idea of MLMC is to observe that thanks to \eqref{eq weakerror},
\begin{equation}\label{eq telescop} \textstyle
(f,\pi) = \mathbb{E} \sum_{\bm \ell \in \mathbb{N}^r}	 (f, \bm \Delta \pi^{\bm \ell})\,, 
\end{equation}
where we set $\pi^{\bm 0 - \bm e_s} := 0$ for all unit vectors $\bm e_s$, with $\bm 0 = (0, \ldots, 0)$. The original MLMC \cite{Giles2015Acta} has been developed for $r=1$, and the extension to an arbitrary $r$ (named Multi-index Monte Carlo, or MIMC) has been developed in \cite{Haji-Ali2016}. In MIMC we approximate each term $\pi^{\bm \ell}$ on the right-hand side of \eqref{eq telescop} using mutually independent, unbiased Monte Carlo estimators $\pi^{\bm \ell, N_{\bm \ell}}$ with $N_{\bm \ell} \geq 1$ samples each, and we choose a finite set of multi-indices $\mathcal L \subset \mathbb{N}^r$ to define
\begin{equation} \label{eq MLMC}  \textstyle
\cA^{\mathcal L}(f) := \sum_{\bm \ell \in \mathcal L}	 (f, \bm \Delta \pi^{\bm \ell,N_{\bm \ell}})\,.
\end{equation}
Clearly, $\mathbb E(f, \bm \Delta \pi^{\bm \ell,N_{\bm \ell}})  = \mathbb{E}(f, \bm \Delta \pi^{\bm \ell})$ and hence $\cA^{\mathcal L}(f)$ is an asymptotically unbiased estimator of $(f,\pi)$ when $\mathcal{L}$ increases to $\mathbb{N}^r$. Moreover, we note that due to the independence of $\pi^{\bm \ell, N_{\bm \ell}}$ across levels, the variance of the MIMC estimator satisfies $\mathbb V [\cA^{\mathcal L}(f)]=\sum_{\bm \ell \in \mathcal L} \mathbb{V}[(f, \bm \Delta \pi^{\bm \ell,N_{\bm \ell}})]$.

In this work we develop an antithetic extension of the MIMC algorithm. To this end we define pairs $(\pi^{+,\bm \ell},\pi^{-,\bm \ell})$ of copies of $\pi^{\bm \ell}$ in the sense that $\mathbb E[(f,\pi^{\bm \ell})]=\mathbb E[(f,\pi^{+, \bm \ell})]=\mathbb E[(f,\pi^{-, \bm \ell})]$ for all Lipschitz functions $f$, and 
\begin{equation}\label{eq antithetic operator} \textstyle
\Delta^{A}_{s}\pi^{\bm \ell} := \pi^{\bm \ell} - \frac{1}{2}(\pi^{+,\bm \ell - \bm e_s}+\pi^{-,\bm \ell - \bm e_s})\,, \quad 
\bm \Delta^{A} \pi^{\bm \ell}:= \left(\prod_{s=1}^{r}\Delta^{A}_{s} \right)\pi^{\bm \ell}\,.
\end{equation}
The corresponding Antithetic MIMC estimator is given by
\begin{equation} \label{eq AMLMC} \textstyle
\cA^{A,\mathcal L}(f) := \sum_{\bm \ell \in \mathcal L}	 (f, \bm \Delta^{A} \pi^{\bm \ell,N_{\bm \ell}})\,.
\end{equation}
As we will see in the sequel, the reduction of the variance of $\cA^{A,\mathcal L}(f)$ in comparison to $\cA^{\mathcal L}(f)$ can be achieved as a consequence of an appropriately chosen coupling between $\pi^{\bm \ell}$, $\pi^{+,\bm \ell - \bm e_s}$ and $\pi^{-,\bm \ell - \bm e_s}$ for each $\bm \ell \in \mathcal{L}$ and $s \in \{ 1, \ldots , r \}$.
Using the notation introduced in \eqref{eq SGLDm}, we will apply (\ref{eq AMLMC}) to the specific case of $\bm \ell = (\ell_1, \ell_2)$ and $\pi^{\bm \ell}$ given as the law of $X_k^{\ell_1, \ell_2}$ for some fixed $k \geq 1$, where 
\begin{equation}\label{eq AMLMC special}
X_{k+1}^{\ell_1, \ell_2} = X_k^{\ell_1, \ell_2} - h_{\ell_2} \nabla_x V(X_k^{\ell_1, \ell_2}, \nu^{s_{\ell_1}}) + \beta \sqrt{h_{\ell_2}} Z_{k+1} \,.
\end{equation}
In this setting, we call the algorithm $\cA^{A,\mathcal L}(f)$ specified by \eqref{eq AMLMC} the Multi-index Antithetic Stochastic Gradient Algorithm (MASGA). Note that for a fixed $t > 0$ such that $t = kh_{\ell_2}$ for some $k \geq 1$, the chain \eqref{eq AMLMC special} can be interpreted as a discrete-time approximation, both in time with parameter $h$ but also in data with parameter $s$, of the SDE
\begin{equation}\label{eq SDEm timechanged}
dY_t = - \nabla_x V(Y_t, \nu^m) dt + \beta dW_t \,,
\end{equation}
where $(W_t)_{t \geq 0}$ is the standard Brownian motion in $\mathbb{R}^d$. 
Then, MASGA with $\pi^{\bm \ell}$ given as the law of $X_k^{\ell_1, \ell_2}$, for any finite subset $\mathcal{L} \subset \mathbb{N}^2$ provides a biased estimator of $(f,\pi_t)$ (where $\pi_t := \mathcal{L}(Y_t)$ with $Y_t$ given by (\ref{eq SDEm timechanged})), where the bias stems from the use of a finite number of levels. However, since $\pi$ given by (\ref{eq gibbs}) is the limiting stationary distribution of \eqref{eq SDEm timechanged}, $\cA^{A,\mathcal L}(f)$ can be also interpreted as a biased estimator of $(f,\pi)$, with an additional bias coming from the difference between $(f,\pi)$ and $(f,\pi_t)$ due to the simulation up to a finite time.

Note that the construction of MASGA does not require any knowledge of the structure of the target measure, such as the location of its modes \cite{BakerFearnheadFox2019, BrosseDurmusMoulines2018, TrueCost}, or any properties of the potential $V$. 
However, in order to formulate the main result of this paper, we will use the following set of assumptions.

\begin{assumption}\label{as main}
	Let the potential $V : \mathbb{R}^d \times \mathcal{P}(\mathbb{R}^d) \to \mathbb{R}$ be of the form $V(x,\nu^m): = v_0(x) + \int_{\mathbb{R}^k} v(x,y) \nu^m(dy) = v_0(x) + \frac{1}{m}\sum_{i=1}^{m} v(x, \xi_i)$, where $(\xi_i)_{i=1}^m \subset \mathbb{R}^k$ is the data and the functions $v_0 : \mathbb{R}^d \to \mathbb{R}$ and $v : \mathbb{R}^d \times \mathbb{R}^k \to \mathbb{R}$ are such that
	\begin{enumerate}[i)]
		\item For all $\xi \in \mathbb{R}^k$ we have $\nabla v(\cdot,\xi)$, $\nabla v_0(\cdot) \in C^2_b(\mathbb{R}^d;\mathbb{R}^d)$, i.e., the gradients of $v$ and $v_0$ are twice continuously differentiable with all partial derivatives of the first and second order bounded (but the gradients themselves are not necessarily bounded). 
		\item There exists a constant $C > 0$ such that for all $\xi \in \mathbb{R}^k$ and for all $x \in \mathbb{R}^d$ we have
		\begin{equation*}
		|\nabla_x v(x,\xi)|^4 \leq C(1+|x|^4) \qquad \text{ and } \qquad |\nabla v_0(x)|^4 \leq C(1+|x|^4) \,.
		\end{equation*}
		\item There exists a constant $K > 0$ such that for all $\xi \in \mathbb{R}^k$ and for all $x$, $y \in \mathbb{R}^d$ we have
		\begin{equation*}
		\langle  x - y , \nabla_x v(x,\xi) - \nabla_y v(y,\xi) \rangle  \geq K|x - y|^2 \qquad \text{ and } \qquad \langle x - y , \nabla v_0(x) - \nabla v_0(y) \rangle \geq K|x - y|^2 \,.
		\end{equation*}
	\end{enumerate}
\end{assumption}

Note that the first condition above in particular implies that the gradient $\nabla V(\cdot,\nu^m)$ of the potential is globally Lipschitz and the third condition implies that $\pi$ given via (\ref{eq gibbs}) is log-concave. Note also that the Bayesian inference example given in (\ref{introduction:BayesianStochasticGradientChain2}) satisfies Assumptions \ref{as main} if the functions $x \mapsto - \nabla \log \pi(\xi | x)$ satisfy all the respective regularity conditions. We remark that we formulate our main result in this section only for the specific form of $V$ given above just for convenience. Our result holds also for a much more general class of potentials, but the assumptions for the general case are more cumbersome to formulate and hence we postpone their presentation to Section \ref{sectionM}. Moreover, we stress that assuming convexity of $V$ is not necessary for the construction of our algorithm and it is a choice we made solely to simplify the proofs. By combining our approach with the coupling techniques from \cite{MajkaMijatovicSzpruch2018}, it should be possible to extend our results to the non-convex case. This, however, falls beyond the scope of the present paper and is left for future work.
We have the following result.

\begin{theorem}\label{theMainTheorem}
	Let Assumptions \ref{as main} hold.
	Let $\cA^{A,\mathcal L}$ be the MASGA estimator defined in \eqref{eq AMLMC} with $\pi^{\bm \ell}$ given as the law of $X_k^{\ell_1, \ell_2}$, defined via (\ref{eq AMLMC special}) for some fixed $k \geq 1$. 
	Then there exists a set of indices $\mathcal{L}$ and a sequence $(N_{\bm \ell})_{\bm \ell \in \mathcal{L}}$ such that for any $\varepsilon > 0$ and for any Lipschitz function $f$, the estimator $\cA^{A,\mathcal L}(f)$ requires the computational cost of order $\varepsilon^{-2}$ to achieve mean square error $\varepsilon$ for approximating $(f, \pi_t)$ with $t = kh_{\ell_2}$, where $\pi_t := \mathcal{L}(Y_t)$ with $Y_t$ given by (\ref{eq SDEm timechanged}).
\end{theorem}

 The proof of Theorem \ref{theMainTheorem} and an explicit description of the sequence $(N_{\bm \ell})_{\bm \ell \in \mathcal{L}}$, as well as all the details of the relaxed assumptions that can be imposed on $V$, will be given in Section \ref{sectionM}. 
 Note that requiring computational cost of order $\varepsilon^{-2}$ to achieve $\operatorname{MSE} < \varepsilon$ is the best performance that one can expect from Monte Carlo methods, cf.\ \cite{Giles2015Acta}.

The two quantities that feature in the optimization problem formulated in Theorem \ref{theMainTheorem} are the MSE and the computational cost (i.e., we want to optimize the cost given the constraint $\operatorname{MSE} < \varepsilon$ for a fixed $\varepsilon > 0$). Note that the cost is explicitly defined to be
\begin{equation}\label{eq cost def} \textstyle
\operatorname{cost}(\cA^{A,\mathcal L}(f)) := \sum_{\bm \ell \in \mathcal{L}} N_{\ell}\,C_{\bm \ell} \,, \quad C_{\bm \ell} :=  t h_{\ell_2}^{-1}  s_{\ell_1} \,,
\end{equation}
where $C_{\bm \ell}$ is the cost of each path at the level $\bm \ell$, i.e., the product of $k = t h_{\ell_2}^{-1}$ steps and $s_{\ell_1}$ subsamples for a given level $\bm \ell = (\ell_1, \ell_2)$. On the other hand, from our discussion on MSE (\ref{eq MSE decomp}) it is evident that in order to control $\operatorname{MSE}(\cA^{A,\mathcal L}(f))$, it is crucial to find upper bounds on the variance $\mathbb{V}[(f, \bm \Delta^{A} \pi^{\bm \ell})]$ and the bias of the estimator at each level $\bm \ell$. In our proof we will in fact rely on the complexity analysis from \cite{Haji-Ali2016} (see also \cite{Giles2015Acta} and Theorem \ref{th complexity} below for more details), which is concerned exactly with such optimization problems.

For convenience, from now on we will assume that in our MASGA estimator (\ref{eq AMLMC}), both the number of subsamples and the discretisation parameter are rescaled by two when moving between levels. More precisely, for a fixed $s_0 \in \mathbb{N}_+$ and $h_0 > 0$, we assume that $s_{\ell_1} = 2^{\ell_1}s_0$ and $h_{\ell_2} = 2^{-\ell_2}h_0$ for all $\bm \ell = (\ell_1, \ell_2) \in \mathbb{N}^2$. In this setting, the complexity analysis from \cite{Haji-Ali2016} tells us, roughly speaking (with more details in Section \ref{sectionM} below) that in order to obtain $\operatorname{MSE}(\cA^{A,\mathcal L}(f)) < \varepsilon$ with $\operatorname{cost}(\cA^{A,\mathcal L}(f)) < \varepsilon^{-2}$, we need to have for each $\bm \ell \in \mathbb{N}^2$,
\begin{equation}\label{eq crucialCondition} \textstyle
\mathbb{E}[(f, \bm \Delta^{A} \pi^{\bm \ell})^2] \lesssim 2^{- \langle \bm \beta , \bm \ell \rangle} \qquad \text{ and } \qquad C_{\bm \ell} \lesssim 2^{\langle \bm \gamma , \bm \ell \rangle} \,,
\end{equation} 
where $\bm \beta = (\beta_1, \beta_2)$ and $\bm \gamma = (\gamma_1, \gamma_2) \in \mathbb{R}^2$ are such that $\gamma_1 < \beta_1$ and $\gamma_2 < \beta_2$. Note that the bound on $\mathbb{E}[(f, \bm \Delta^{A} \pi^{\bm \ell})^2]$ trivially implies a bound on $\mathbb{V}[(f, \bm \Delta^{A} \pi^{\bm \ell})]$ of the same order (i.e., with the same $\bm \beta$), as well as the bound $\mathbb{E}[(f, \bm \Delta^{A} \pi^{\bm \ell})] \lesssim 2^{- \langle \bm \alpha , \bm \ell \rangle}$ with $\bm \alpha = \bm \beta /2$, which turns out to be crucial in the complexity analysis from \cite{Haji-Ali2016} and is the reason why it suffices to verify (\ref{eq crucialCondition}), cf.\ also Theorem 2 in \cite{Giles2015Acta} and Theorem \ref{th complexity} below. Since, straight from the definition of $C_{\bm \ell}$, it is clear that in our setting we have $\bm \gamma = (1,1)$ (recall that $C_{\bm \ell} \lesssim s_{\ell_1} h_{\ell_2}^{-1} \lesssim 2^{\ell_1 + \ell_2}$), we can infer that in order to prove Theorem \ref{theMainTheorem} all we have to do is to find an upper bound on $\mathbb{E}[(f, \bm \Delta^{A} \pi^{\bm \ell})^2]$ proportional to $2^{- \langle \bm \beta , \bm \ell \rangle}$ with $\bm \beta = (\beta_1, \beta_2)$ such that $\beta_1 > 1$ and $\beta_2 > 1$. In the proof of Theorem \ref{theMainTheorem} we will in fact obtain $\bm \beta = (2,2)$. However, the crucial difficulty in our argument will be to ensure that our upper bound is indeed of the product form, i.e., that we obtain $\mathbb{E}[(f, \bm \Delta^{A} \pi^{\bm \ell})^2] \lesssim s_{\ell_1}^{-2} h_{\ell_2}^2 \lesssim 2^{-2\ell_1 - 2 \ell_2}$.
 
\paragraph{Further extensions}

	The assertion of Theorem \ref{theMainTheorem} states that $\cA^{A,\mathcal L}(f)$, interpreted as an estimator of $(f, \pi_t)$, requires computational cost of order $\varepsilon^{-2}$ to achieve mean square error $\varepsilon$. Since the difference between $(f, \pi_t)$ and $(f, \pi)$ is of order $\mathcal{O}(e^{-\lambda t})$ for some $\lambda > 0$ (cf.\ the discussion in Appendix \ref{section:Examples}), this means that $\cA^{A,\mathcal L}(f)$ interpreted as an estimator of $(f, \pi)$, requires computational cost of order $\varepsilon^{-2}\log(\varepsilon^{-1})$ to achieve mean square error $\varepsilon$. However, $\cA^{A,\mathcal L}(f)$ could be further modified in order to remove the log term, by employing the MLMC in terminal time technique introduced in \cite{MajkaSzpruchVollmerZygalakisGiles2019}, cf.\ Section 2.3 and Remark 3.8 therein. This would involve taking $r = 3$ in (\ref{eq AMLMC}) and modifying the definition as
	\begin{equation}\label{eq AMLMC terminaltime} \textstyle
	\cA^{A,\mathcal L}(f) := \sum_{\bm \ell \in \mathcal L}	 (f, \Delta^{A}_1 \Delta^{A}_2 \Delta_3 \pi^{\bm \ell,N_{\bm \ell}})\,,
	\end{equation}
	i.e., we would take $\pi^{\bm \ell}$ with $\bm \ell = (\ell_1, \ell_2, \ell_3)$ to be the law of $X_{k_{\ell_3}}^{\ell_1, \ell_2}$ given by (\ref{eq AMLMC special}), hence we would introduce a sequence of terminal times $t := k_{\ell_3} h_{\ell_2}$ for the chain (\ref{eq AMLMC special}), changing at each level. However, in the definition (\ref{eq AMLMC terminaltime}) of $\cA^{A,\mathcal L}(f)$ we would use the antithetic difference operators $\Delta^{A}$ only with respect to the subsampling level parameter $\ell_1$ and the discretisation level parameter $\ell_2$, while applying the plain difference operator $\Delta$ to the terminal time level parameter $\ell_3$. The details of how to construct the sequence of terminal times $t := k_{\ell_3} h_{\ell_2}$ can be found in \cite{MajkaSzpruchVollmerZygalakisGiles2019}. We skip this modification in the present paper in the attempt to try to keep the notation as simple as possible.

In the setting where the computational complexity is $\varepsilon^{-2}$, (as it is for MASGA), it is possible to easily modify the biased estimator $\cA^{A,\mathcal L}(f)$ to obtain its unbiased counterpart. Indeed, let $\bm M = (M_1, \ldots, M_r)$ be a random variable on $\mathbb N^r$, independent of  $(\bm \Delta^{A} \pi^{\bm \ell,N_{\bm \ell}})_{\bm \ell \in \mathcal{L}}$. Define $\mathcal L^{\bm M} := \{\ell \in \mathbb{N}^r : \ell_1 \leq M_1, \cdots, \ell_r \leq M_r\} $ and 
\begin{equation*} \textstyle
\cA^{UA,\mathcal L^{\bm M}}(f) := \sum_{\bm \ell \in \mathcal L^{\bm M}}	 \frac{(f, \bm \Delta^{A} \pi^{\bm \ell,N_{\bm \ell}})}{\mathbb P(\bm M \geq \bm \ell)} =  \sum_{\bm \ell \in \mathbb{N}^r}	  1_{\{\bm M \geq \bm \ell\} }\frac{(f, \bm \Delta^{A} \pi^{\bm \ell,N_{\bm \ell}})}{\mathbb P(\bm M \geq \bm \ell)}\,,
\end{equation*}		
where $\bm M \geq \bm \ell$ is understood component-wise. One can see that $\cA^{UA,\mathcal L^{\bm M}}(f)$ is then an unbiased estimator
		of $(f,\pi)$. Indeed,
\[ \textstyle
\mathbb E\big[\cA^{UA,\mathcal L^{\bm M}}(f) \big]
= \mathbb E\left[  \sum_{\bm \ell \in \mathbb{N}^r}	 (f, \bm \Delta^{A} \pi^{\bm \ell,N_{\bm \ell}}) \right] = (f,\pi)\,,
\]		
due to (\ref{eq telescop}). It turns out that in order for the variance of  $\cA^{UA,\mathcal L^{\bm M}}(f)$ to be finite, we need to be in the regime where the computational complexity of the original estimator is $\varepsilon^{-2}$.  We refer the reader to  \cite{Rhee2015,crisan2018unbiased} for more details and recipes for constructing $\bm M$. 
The methods from \cite{Rhee2015} have recently been extended to more general classes of MCMC algorithms in \cite{Jacob2020}, which contains further discussion of the benefits and costs of debiasing.

\paragraph{Literature review} The idea of employing MLMC apparatus to improve efficiency of stochastic gradient algorithms has been studied before. In \cite{MajkaSzpruchVollmerZygalakisGiles2019} we introduced a Multi-level Monte Carlo (MLMC) method for SGLD in the global convexity setting, based on a number of decreasing discretisation levels. We proved that under certain assumptions on the variance of the estimator of the drift, this technique can indeed improve the overall performance of the Monte Carlo method with stochastic gradients. In \cite{MajkaMijatovicSzpruch2018} we extended our approach to cover the non-convex setting, allowing for sampling from probability measures that satisfy a weak log-concavity at infinity condition.  However, the computational complexity of such algorithms is sub-optimal. As we will observe in Section \ref{sectionNumerics} with numerical experiments, the crucial insight of the present paper (and a novelty compared to \cite{MajkaSzpruchVollmerZygalakisGiles2019, MajkaMijatovicSzpruch2018}) is the application of MLMC with respect to the subsampling parameter. Note that, in a different context, the idea to apply MLMC to stochastic approximation algorithms has been studied in \cite{Frikha2016}, see also \cite{DereichMuller2019, Dereich2021}.

At the core of our analysis of Multi-level Monte Carlo estimators lies the problem of constructing the right couplings between Euler schemes (\ref{eq AMLMC special}) on different discretisation levels. For Euler schemes with standard (accurate) drifts this is done via a one-step analysis by coupling the driving noise in a suitable way, cf.\ Sections 2.2 and 2.5 in \cite{MajkaMijatovicSzpruch2018} and Section 2.4 in \cite{EberleMajka2019}. However, in the case of SGLD, one is faced with an additional problem of coupling the drift estimators used on different discretisation levels. In both \cite{MajkaSzpruchVollmerZygalakisGiles2019} and \cite{MajkaMijatovicSzpruch2018} we addressed this issue in the simplest possible way, by coupling the drift estimators independently. In the present paper we show that by employing a non-trivial coupling we can substantially reduce the variance of MLMC and thus obtain the required bound $\mathbb{V}[(f, \bm \Delta^{A} \pi^{\bm \ell})] \lesssim 2^{-2\ell_1 - 2 \ell_2}$ as explained above. We achieve this by utilising the antithetic approach to MLMC as defined in (\ref{eq antithetic operator}). Related ideas for antithetic multi-level estimators have been used e.g. in \cite{GilesSzpruch2012, GilesSzpruch2014, SzpruchTse2019}. However, in the present paper we apply this concept for the first time to Euler schemes with inaccurate drifts. We also remark that, due to our bounds on second moments, we can easily derive confidence intervals for MASGA using Chebyshev's inequality. However, it would also be possible to derive a Central Limit Theorem and corresponding concentration inequalities, in the spirit of \cite{alaya2015central, alaya2020central, jourdain2019non, Kebaier2005}.

The remaining part of this paper is organised as follows. In Section \ref{sectionNumerics} we present numerical experiments confirming our theoretical findings. In Section \ref{sectionM} we present a more general framework for the MASGA estimator, we explain the intuition behind the antithetic approach to MLMC in more detail (see Example \ref{exampleAMLMC}) and we formulate the crucial Lemma \ref{thm:crucialLemma}. We also explain how to prove Theorem \ref{theMainTheorem} based on Lemma \ref{thm:crucialLemma}. In Section \ref{sectionProofs} we prove Lemma \ref{thm:crucialLemma} in a few steps: we first discuss the antithetic estimator with respect to the discretisation parameter, which corresponds to taking $r=1$ and ${\bm \ell} = \ell_2$ in (\ref{eq AMLMC}), then we consider the antithetic estimator with respect to the subsampling parameter, which corresponds to taking $r=1$ and ${\bm \ell} = \ell_1$ in (\ref{eq AMLMC}) and, finally, we explain how these approaches can be combined in a multi-index estimator with $r=2$ and ${\bm \ell} = (\ell_1, \ell_2)$ to prove Lemma \ref{thm:crucialLemma}. Several technical proofs are postponed to Appendices.

\section{Numerical experiments}\label{sectionNumerics}

We showcase the performance of the MASGA estimator in Bayesian inference problems, combining different Bayesian models and priors. The code for all the numerical experiments
can be found at \url{https://github.com/msabvid/MLMC-MIMC-SGD}.

In Subsection \ref{subsect:MASGAvsAMLMC} we compare the MASGA estimator introduced in Section \ref{section:mainResult} with an Antithetic Multi-level Monte Carlo (AMLMC) estimator with respect to the subsampling parameter, corresponding to taking $r=1$ and ${\bm \ell} = \ell_1$ in (\ref{eq AMLMC}). We demonstrate that MASGA indeed achieves the optimal computational complexity. Both these estimators are also compared to a standard Monte Carlo estimator for reference. As we shall see, while the performance of MASGA in our experiments is always better than that of AMLMC, the difference is not substantial. This suggests that, from the practical standpoint, the crucial insight of this paper is the application of the antithetic MLMC approach with respect to the subsampling parameter. Hence in our subsequent experiments in Subsections \ref{subsect:AMLMC} and \ref{subsect:AMLMCvsSGLDCV}, we will focus on the AMLMC estimator, which is easier to implement than MASGA. In Subsection \ref{subsect:AMLMC} we will check its performance in both convex and non-convex settings, whereas in Subsection \ref{subsect:AMLMCvsSGLDCV} we will compare it to the Stochastic Gradient Langevin Dynamics with Control Variates (SGLD-CV) method introduced in \cite{BakerFearnheadFox2019}. More precisely, in Subsection \ref{subsect:AMLMCvsSGLDCV} we present an example of a convex setting in which AMLMC outperforms SGLD-CV. It is worth pointing out, that the latter method can be applied only to convex settings, whereas AMLMC is free of such limitations.

\subsection{MASGA and AMLMC with respect to subsampling}\label{subsect:MASGAvsAMLMC}

Let us begin by introducing the Bayesian logistic regression setting that we will use for our simulations. The data is modelled by 
\begin{equation}\label{eq logistic regression} \textstyle
p(y_i | \iota_i, x) = g(y_i x^T \iota_i)
\end{equation}
with $g(z) = 1/(1+e^{-z})$, $z \in \mathbb{R}$, where $x\in \mathbb R^d$ are the parameters of the model
that need to be sampled from their posterior distribution, $\iota_i$ denotes an observation of the predictive variable in the data, and $y_i$ the binary target variable. Given a dataset of size $m$, by Bayes' rule, the posterior density of $x$ satisfies
\begin{equation}\label{eq posterior bayesian logistic} \textstyle
\pi(x) \propto \pi_0(x) \prod_{i=1}^m g(y_ix^T\iota_i).	
\end{equation}

In our experiments, we will consider two different priors $\pi_0$, namely
\begin{enumerate}[i)]
\item a Gaussian prior $\pi_0 \sim \mathcal N(0, I)$.
\item a mixture of two Gaussians $\pi_0 \sim \frac{1}{2} \mathcal N(0,I) + \frac{1}{2} \mathcal N(1,I)$. 	
\end{enumerate}

We use Algorithm~\ref{alg MIMC} to empirically calculate the cost of
approximating $\mathbb E(f(X))$, for a function $f:\mathbb R^d \rightarrow \mathbb R$, where the law of $X$ is the posterior $\pi$, by 
the MASGA estimator 
$\cA^{A,\mathcal L}_{MASGA}(f) = \cA^{A,\mathcal L}(f)$ defined in \eqref{eq AMLMC}, such that its MSE is under some threshold $\varepsilon$. Recall that in our notation $(f, \bm \Delta^{A} \pi^{\bm \ell})$ denotes the integral of $f$ with respect to the antithetic measure $\Delta^{A} \pi^{\bm \ell}$ given by (\ref{eq antithetic operator}), where $\pi^{\bm \ell}$ is specified by the law of the Markov chain (\ref{eq AMLMC special}) with the potential $V$ determined from (\ref{eq posterior bayesian logistic}) in an analogous way as in the Bayesian inference example (\ref{introduction:BayesianDrift}). 
More explicitly, we have
\begin{equation*} \textstyle
V(x,\nu^m) = -\log \pi_0(x) - \sum_{i=1}^{m} \log g(y_i x^T \iota_i)
\end{equation*}
and
\begin{equation}\label{eq:numericsChain} \textstyle
X_{k+1}^{\ell_1, \ell_2} = X_k^{\ell_1, \ell_2} + \frac{1}{2}h_{\ell_2} \left( \frac{1}{m} \nabla \log \pi_0(X_k^{\ell_1, \ell_2}) + \frac{1}{s_{\ell_1}}   \sum_{i=1}^{s_{\ell_1}} \nabla \log g\left(y_{\tau_i^k}(X_k^{\ell_1, \ell_2})^T \iota_{\tau_i^k}\right) \right) + \frac{1}{\sqrt{m}} \sqrt{h_{\ell_2}} Z_{k+1} \,,
\end{equation}
where $\tau_i^k$ for $k \geq 1$ and $i \in \{ 1, \ldots, s_{\ell_1}  \}$ can correspond to subsampling either with or without replacement (in our simulations we choose the latter).

Below we present the results of our experiments for $f(x) = |x|^2$ (we would like to remark that we obtained similar conclusions for $f(x) = |x|$ and hence we skip the latter example to save space). 

Furthermore, for the Bayesian Logistic regressions we use the covertype dataset~\cite{blackard1999comparative} which has $581\,012$ observations, and 54 columns\footnote{In order to perform a Bayesian logistic regression, the categorical variable specifying the forest type designation is aggregated into a binary variable and is used as the target variable $y_i$ in the model.}. We create a training set containing 20\% of the original observations.

\begin{algorithm}
	\caption{MASGA}	
	\label{alg MIMC}
	\begin{algorithmic}
		\STATE{Initialisation: $n$, $L$, $\varepsilon$}
		\STATE{Calculate $n$ samples $(f, \bm \Delta^{A} \pi^{\bm \ell})^{(j)}$ of $(f, \bm \Delta^{A} \pi^{\bm \ell})$ independently, $j = 1, \ldots , n$ for each multi-index level in $\mathcal L := \{\bm\ell = (\ell_1,\ell_2): \ell_1,\ell_2=0,\ldots, L \}$.}
		\STATE{Calculate the MASGA estimator $\cA^{A,\mathcal L}(f)$ as in \eqref{eq AMLMC} with $N_{\bm \ell} = n$, i.e., $\cA^{A,\mathcal L}(f) = \sum_{\bm \ell \in \mathcal L}	 (f, \bm \Delta^{A} \pi^{\bm \ell,n})$, where for each $\bm \ell \in \mathcal L$ we take $(f, \bm \Delta^{A} \pi^{\bm \ell,n}) := n^{-1} \sum_{j=1}^{n} (f, \bm \Delta^{A} \pi^{\bm \ell})^{(j)}$.}
		\WHILE{$\operatorname{MSE}>\varepsilon$}
		\STATE{
			Calculate the number of paths in each level $\ell$ given by \\
			$N_{\ell_1,\ell_2} = \varepsilon^{-2}\left( \sqrt{\frac{\mathbb{V}[(f, \bm \Delta^{A} \pi^{(\ell_1,\ell_2)})]}{C_{(\ell_1,\ell_2)}}} \sum_{\bm \ell \in \mathcal{L}} \sqrt{\mathbb{V}[(f, \bm \Delta^{A} \pi^{\bm \ell})] C_{\bm \ell}} \right)$ (see Theorem \ref{thm:mainTheorem} for details)
			and calculate extra $N_{\ell_1,\ell_2} - n$ samples in each multi-index level.
		}
		\STATE{
			Update $\cA^{A,\mathcal L}(f)$ and estimate its bias
			(see~\cite{Giles2015Acta, Haji-Ali2016}).
		}
		\IF{bias estimate is less than $\varepsilon/2$}
		\STATE{Set convergence=True, and stop.}
		\ELSE
		\STATE{Set $L:=L+1$, and calculate $n$ samples of $(f, \bm \Delta^{A} \pi^{\bm \ell})$ for $\bm \ell=(\ell_1,L), \, \ell_1\leq L$ and for $\bm \ell=(L,\ell_2), \, \ell_2 \leq L$.}
		\ENDIF
		\ENDWHILE
		\RETURN $\cA^{A,\mathcal L}(f)$ and the number of samples $N_{\ell_1,\ell_2}$ in each level.
	\end{algorithmic}
\end{algorithm}

On the other hand, for AMLMC with respect to subsampling (denoted below by $\cA^{A,\mathcal L}_{MLMC}(f)$) we take $r = 1$ and $\bm \ell = \ell_1$ in (\ref{eq AMLMC}). This corresponds to using a fixed discretisation parameter $h$ and applying the antithetic MLMC estimator only with respect to the subsampling parameter.

Note that in our experiments we apply the estimators $\cA^{A,\mathcal L}_{MASGA}(f)$ and $\cA^{A,\mathcal L}_{MLMC}(f)$ to approximate $(f,\pi_t^h)$, where $\pi_t^h$ is the law given by the chain $X_k$ with $t = kh$, defined by
\begin{equation}\label{eq:numericsChain2} \textstyle
X_{k+1} = X_k + \frac{1}{2}h \left( \frac{1}{m} \nabla \log \pi_0(X_k)  + \frac{1}{m} \sum_{i=1}^{m} \nabla \log g(y_i X_k^T \iota_i) \right) + \frac{1}{\sqrt{m}} \sqrt{h} Z_{k+1} \,,
\end{equation}
with a fixed discretisation parameter $h$ (i.e., we do not take into account the error between $(f,\pi_t^h)$ and $(f,\pi_t)$ when calculating the MSE, where $\pi_t$ is the law of the SDE \eqref{eq SDEm timechanged}). The value of this $h$ is determined by the final level used in MASGA, i.e., $h = h_L$.

The experiment is organised as follows: 
\begin{enumerate}[i)]
\item let $L\geq 1$, and $(\ell_s, \ell_d) = (L,L)$ be the highest multi-indices used in the calculation of $\mathcal A^{A, \mathcal L}_{MASGA}(f)$.
\item We measure the bias of $\mathcal A^{A, \mathcal L}_{MASGA}(f)$ and set $\varepsilon := \sqrt{2} \left( \mathbb{E} \left[ |(f,\pi^{h_L}_t) - \mathcal A^{A, \mathcal L}_{MASGA}(f)|^2 \right] \right)^{1/2}$.
\item We then compare the cost of $\mathcal A^{A, \mathcal L}_{MASGA}(f)$ 
against the cost of $\mathcal A^{A, \mathcal L}_{MLMC}(f)$ with fixed discretisation parameter $h = h_{\ell_d} = h_L$ satisfying $\mathbb E [((f,\pi^{h_L}_t) - \mathcal A_{MLMC}^{A, \mathcal L}(f))^2]<\varepsilon^2$. 	
\end{enumerate}
We repeat the above three steps for $L=1,\ldots,7$, in order to measure the cost for different values of $\varepsilon$.

We perform this comparison on two data regimes: first on the covertype dataset with 100K observations and 54 covariates, and second on a smaller synthetic dataset with 1K observations and 5 covariates. Results are shown in Figure~\ref{fig MASGA vs AMLMC}, where each $\varepsilon$ corresponds to different values of $L$ (see Table~\ref{table:masga setting}).  
\begin{table}[h]
\centering
\begin{tabular}{l|l}
 AMLMC parameter & Value \\
  \hline
 $h_0$ (initial discretisation step size) & $0.005 $ \\ 
  Number of steps in initial discretisation level  &  $100$ \\ 
 (dataset size, dataset dim) &  $(116\, 202, 54), \, \, (1\,000, 5)$ \\
 $s_0$ (initial subsample size) & 4 \\
 $X_0$ & Approximation of the mode of the posterior
\end{tabular}
\caption{MASGA setting for Bayesian Logistic Regression.}\label{table:masga setting}
\end{table}

As expected, the higher the accuracy (the lower the $\varepsilon$) the better the cost of $\mathcal A^{A, \mathcal L}_{MASGA}(f)$ compared to the cost of $\mathcal A^{A, \mathcal L}_{MLMC}(f)$. Depending on the dataset size, it is necessary to reach different levels of accuracy of the MSE to notice an improvement on the cost. This comes from the amount of variance added by the noise added in the chain $X_k^{\ell_1,\ell_2}$ (\ref{eq:numericsChain}) that will decrease as the dataset size $m$ increases. 

\begin{figure}
	\centering
	\includegraphics[scale=0.6]{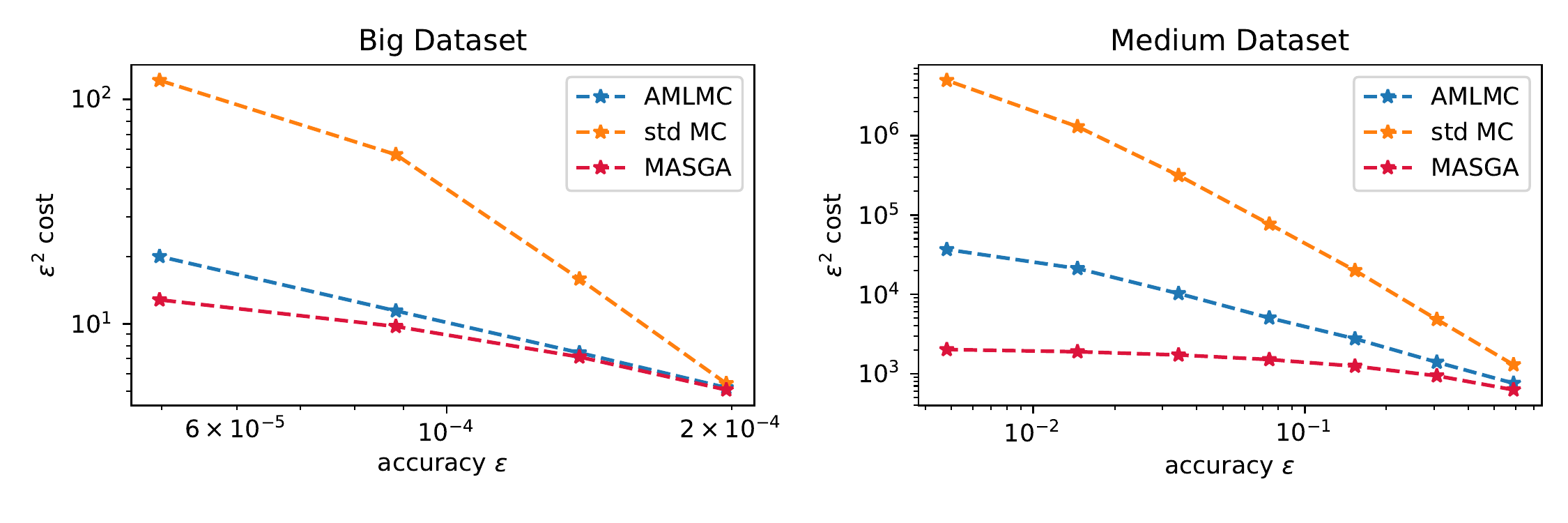}
	\caption{Comparison of MASGA cost against AMLMC cost for the standard Gaussian prior.}
	\label{fig MASGA vs AMLMC}
\end{figure}

\subsection{AMLMC with respect to subsampling in convex and non-convex settings}\label{subsect:AMLMC}

In the subsequent experiments, we study the AMLMC estimator with respect to subsampling, i.e., with a fixed discretisation parameter $h$. We simulate $10\times 2^4$ steps of the chain (\ref{eq:numericsChain}). We take $X_0$ to be an approximation of the mode of the posterior that we pre-compute using Stochastic Gradient Descent to replace the burn-in phase of the Markov chain, cf.\ \cite{BakerFearnheadFox2019}. The number of steps and the step size are chosen so as to be consistent with the finest discretisation level of the MASGA experiment provided in the previous section.

A summary of the AMLMC setting is provided in Table~\ref{table:amlmc setting}.

\begin{figure}
	\centering
	\includegraphics[scale=0.55]{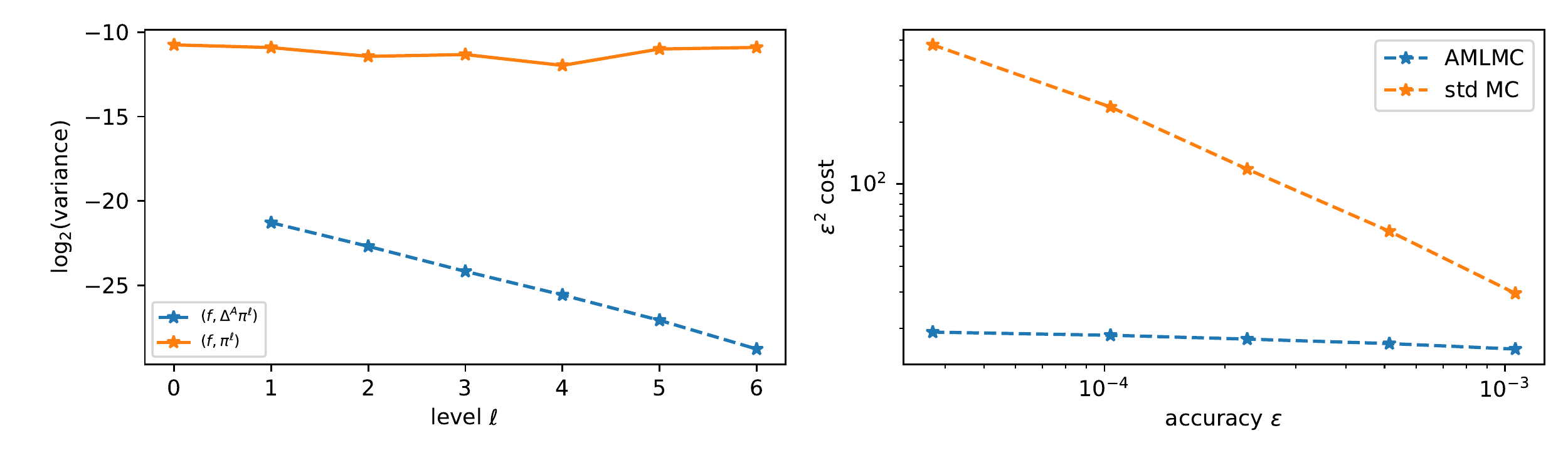}
	\caption{AMLMC estimator $\cA^{A,\mathcal L}(f)$ with respect to subsampling for $f(x)= |x|^2$, where the prior is the standard Gaussian.}
	\label{fig MLMC subsample norm sq}
\end{figure}

\begin{table}[h]
	\centering
	\begin{tabular}{l|l}
		AMLMC parameter & Value \\
		\hline
		$h$ (fixed discretisation step size) & $0.005 \times 2^{-4}$ \\ 
		Number of steps  &  $100 \times 2^4$ \\ 
		(dataset size, dataset dim) &  $(116\, 202, 54)$ \\
		$s_0$ (initial subsample size) & 4 \\
		$X_0$ & Approximation of the mode of the posterior
	\end{tabular}
	\caption{AMLMC setting for Bayesian Logistic Regression.}\label{table:amlmc setting}
\end{table}

\begin{figure}[H]
	\centering
	\includegraphics[scale=0.55]{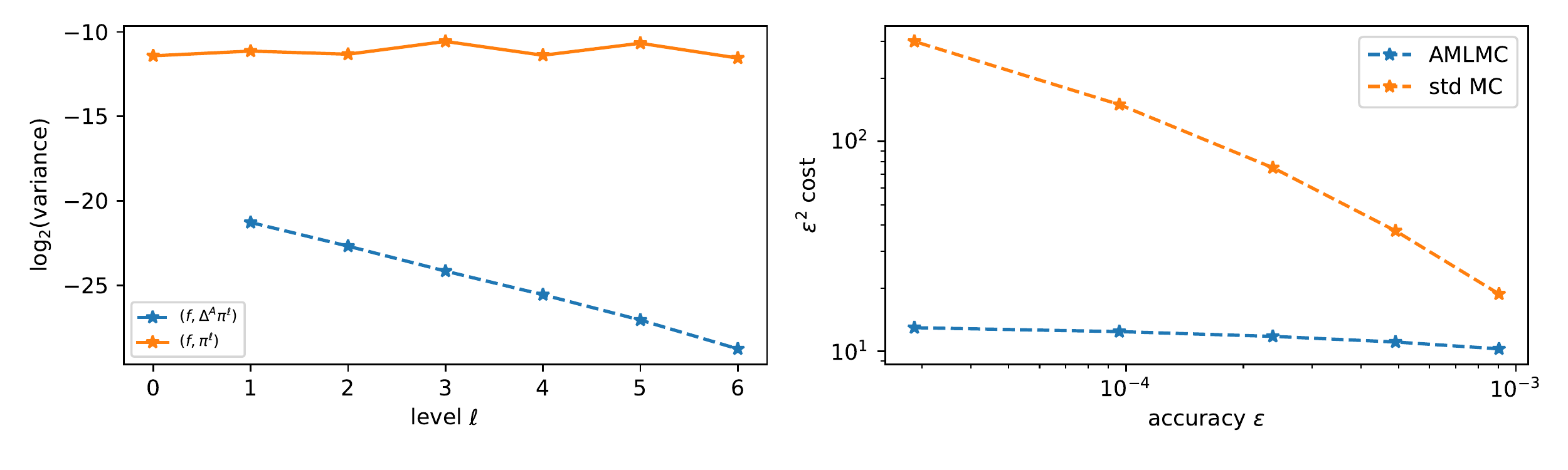}
	\caption{AMLMC estimator $\cA^{A,\mathcal L}(f)$ with respect to subsampling for $f(x)= |x|^2$, where the prior is a mixture of two Gaussians.}
	\label{fig MLMC subsample norm sq Bayesian Logistic with prior mixture of Gaussians}
\end{figure}

Plots in Figure~\ref{fig MLMC subsample norm sq} correspond to the results where $\pi_t^h$ is the approximation of the posterior of a Bayesian Logistic Regression with Gaussian prior. Plots in Figure~\ref{fig MLMC subsample norm sq Bayesian Logistic with prior mixture of Gaussians} use a mixture of Gaussians for the prior. The left plot shows the variance of $(f, \bm \Delta^{A} \pi^{\bm \ell})$ and  $(f,  \pi^{\bm \ell})$ per subsampling level. The right plot displays the computational cost multiplied by $\varepsilon^2$.

These figures indicate that the total cost of approximating $(f,\pi_t^h)$ by $\cA^{A,\mathcal L}(f)$ as described above, is $\mathcal{O}(\varepsilon^{-2})$, even when the prior is not log-concave as is the case of a Mixture of two Gaussians (Figure~\ref{fig MLMC subsample norm sq Bayesian Logistic with prior mixture of Gaussians}).

\subsubsection{Bayesian Mixture of Gaussians}
For our next experiment, we use the setting from Example 5.1 in \cite{welling2011bayesian} to consider a Bayesian mixture of two Gaussians on a 2-dimensional dataset, in order to make the posterior multi-modal. 
Given a dataset of size $m$, by Bayes' rule
\begin{equation} \textstyle
\pi(x) \propto \pi_0(x) \prod_{i=1}^m g(\iota_1, \iota_2 | x),	
\end{equation}
where $x=(x_1,x_2)$, $g$ is the joint density of $(\iota_1, \iota_2)$ where each $\iota_i \sim \frac{1}{2}\mathcal N(x_1, 5) + \frac{1}{2}\mathcal N(x_1+x_2, 5)$. 
For the experiment, we consider a Gaussian prior $\mathcal N(0,I)$ for $\pi_0(x)$, and we create a synthetic dataset with $200$ observations, by sampling from $\iota_i \sim \frac{1}{2}\mathcal N(0, 5) + \frac{1}{2}\mathcal N(1, 5)$.

In this experiment we again take $r = 1$ and $\bm \ell = \ell_1$ in (\ref{eq AMLMC}), which corresponds to using a fixed discretisation parameter $h$ and applying the antithetic MLMC estimator only with respect to the subsampling parameter. We then use the same setting as before: we apply our estimator $\cA^{A,\mathcal L}_{MLMC}(f)$ to approximate $(f,\pi_t^h)$, where $\pi_t^h$ is the law given by the chain $X_k$ defined in (\ref{eq:numericsChain2}), with a fixed discretisation parameter $h = 1$ (i.e., we do not take into account the error between $(f,\pi_t^h)$ and $(f,\pi_t)$ when calculating the MSE). We simulate $2\times 10^5$ steps of the chain (\ref{eq AMLMC special}), starting from $X_0=0$ (see Table~\ref{table:amlmc setting mixture}). 

In this example there is the additional difficulty that the posterior has two modes. It is therefore necessary to ensure that the number of steps is high enough so that the chain has explored all the space (Figure~\ref{fig MCMC mixture}).  

\begin{table}[h]
\centering
\begin{tabular}{l|l}
 AMLMC parameter & Value \\
  \hline
 $h$ (fixed discretisation step size) & $1$ \\ 
  Number of steps  &  $200\, 000$ \\ 
 (dataset size, dataset dim) &  $(200, 2)$ \\
 $s_0$ (initial subsample size) & 2 \\
 $X_0$ & $(0,0)$
\end{tabular}
\caption{AMLMC setting for Bayesian Mixture of Gaussians.}\label{table:amlmc setting mixture}
\end{table}

Results for this experiment are shown in Figure~\ref{fig MLMC subsample norm sq Bayesian Mixture with prior Gaussian}, indicating that the total cost of approximating $(f,\pi_t)$ by the MASGA estimator $\cA^{A,\mathcal L}(f)$ is $\mathcal{O}(\varepsilon^{-2})$. We obtain the following rates of decay of the variance and the absolute mean of $\bm \Delta^{A} \pi^{\bm \ell})$:

\[ \textstyle i)\,\, \mathbb E[(f, \bm \Delta^A \pi^{\bm \ell})]| \lesssim 2^{- 1.01\ell}, 
\quad ii) \,\,\mathbb{V}[(f, \bm \Delta^A \pi^{\bm \ell})] \lesssim 2^{- 1.82\, \ell} \,.
\]

\begin{figure}[H]
	\centering
	\includegraphics[scale=0.55]{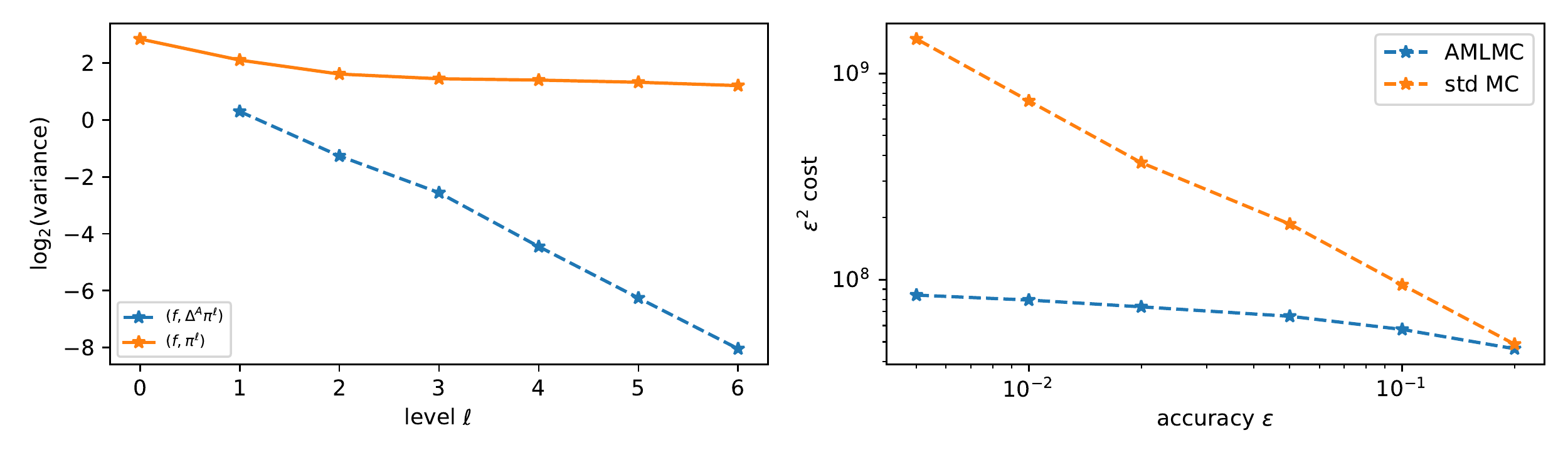}
	\caption{AMLMC estimator $\cA^{A,\mathcal L}(f)$  with respect to subsampling for $f(x)= |x|^2$, for a multi-modal posterior, where the prior $\pi_0$ is the standard Gaussian.}
	\label{fig MLMC subsample norm sq Bayesian Mixture with prior Gaussian}
\end{figure}

\begin{figure}[H]
	\centering
	\includegraphics[scale=0.55]{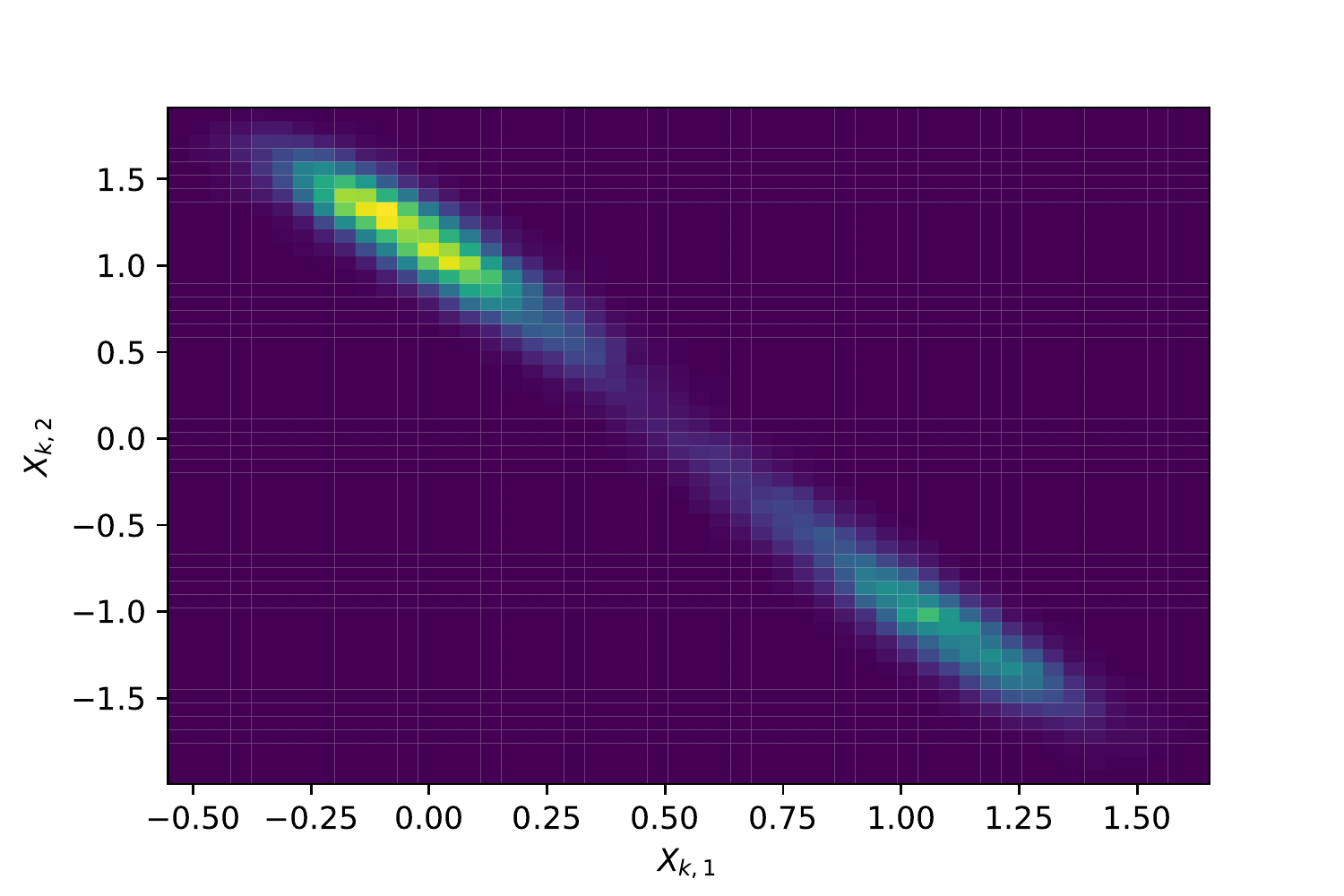}
	\caption{Histogram of $X_k = (X_{k,1}, X_{k,2})$ sampled from~\eqref{eq AMLMC special}.}
	\label{fig MCMC mixture}
\end{figure}

\subsection{AMLMC and SGLD with control variates}\label{subsect:AMLMCvsSGLDCV}

In this subsection we compare the AMLMC estimator with respect to subsampling against the Stochastic Gradient Langevin Dynamics method with control variates (SGLD-CV) from~\cite{BakerFearnheadFox2019, TrueCost}. We work with the standard Gaussian prior $\pi_0$. For SGLD-CV, for a fixed time step size $h$, and for a fixed subsample size $s_{\ell_1}$, instead of the process \eqref{eq AMLMC special} we use 
\begin{equation}\label{eq AMLMC special with cv} \textstyle
X_{k+1}^{\ell_1} = X_k^{\ell_1} - h \left(\nabla_x V(\hat x, \nu^{m}) + \left(\nabla_x V(X_k^{\ell_1}, \nu^{s_{\ell_1}})-\nabla_x V(\hat x, \nu^{s_{\ell_1}})\right)\right) + \beta \sqrt{h_{\ell_2}} Z_{k+1} \,,
\end{equation}
where $\beta = 1/\sqrt{m}$ and $\hat x$ is a fixed value denoting an estimate of the mode of the posterior $\pi(x)$.
We undertake the following steps:
\begin{enumerate}
	\item We estimate the mode of the posterior $\hat x$ by using stochastic gradient descent.
	\item For each considered accuracy $\varepsilon$, we run AMLMC with respect to subsampling (as described in Subsection \ref{subsect:AMLMC}) to get the maximum subsample size $s_{\ell_1}$ necessary to achieve an estimator $\cA^{A,\mathcal L}_{MLMC}(f)$ such that $\operatorname{MSE}(\cA^{A,\mathcal L}_{MLMC}(f)) \lesssim \varepsilon$.
	\item We use each pair $(\varepsilon, s_{\ell_1})$ from the previous step to calculate the cost of SGLD-CV.    	
\end{enumerate}
The AMLMC setting values are listed in Table~\ref{table:amlmc setting comparison} and results are shown in Figure~\ref{fig MLMC subsample norm sq comparison with MCMC cv}.

\begin{table}[h]
	\centering
	\begin{tabular}{l|l}
		AMLMC parameter & Value \\
		\hline
		$h$ (fixed discretisation step size) & $0.5$ \\ 
		Number of steps  &  $100$ \\ 
		(dataset size, dataset dim) &  $(116\, 202, 54)$ \\
		$s_0$ (initial subsample size) & 4 \\
		$X_0$ & Approximation of the mode of the posterior
	\end{tabular}
	\caption{AMLMC setting for Bayesian Logistic Regression.}\label{table:amlmc setting comparison}
\end{table}

\begin{figure}
	\centering
	\includegraphics[scale=0.55]{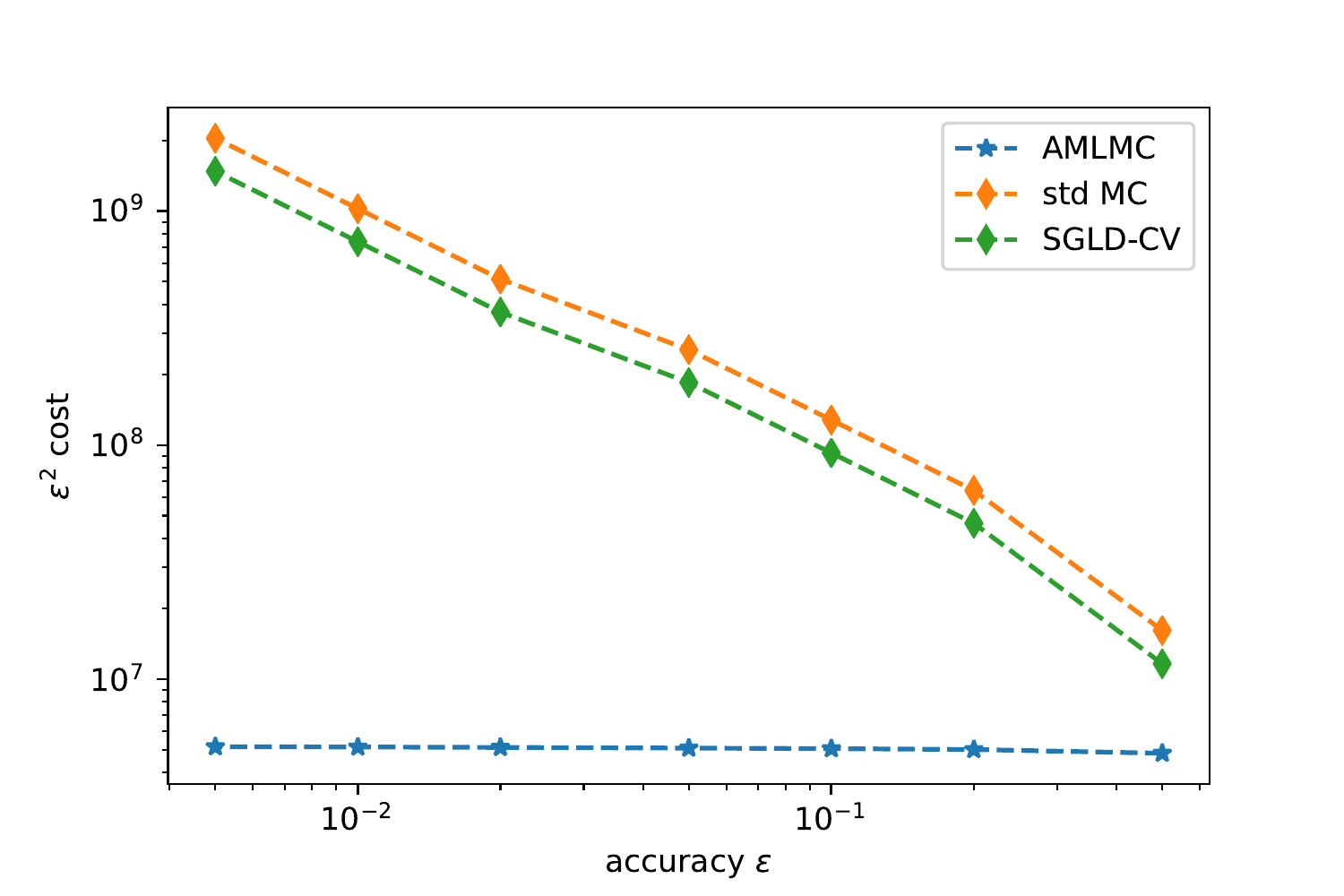}
	\caption{Comparison of AMLMC with respect to subsampling for $f(x)= |x|^2$ vs SGLD with control variate.}
	\label{fig MLMC subsample norm sq comparison with MCMC cv}
\end{figure}

The SGLD-CV method has been shown in \cite{BakerFearnheadFox2019, TrueCost} to reduce the variance (and hence improve the performance) compared to the standard SGLD and the standard Monte Carlo methods. However, in some numerical examples, as demonstrated in this subsection, this gain can be relatively small compared to the gain from using AMLMC.

\section{General setting for MASGA}\label{sectionM}

We will work in a more general setting than the one presented in Section \ref{section:mainResult}. Namely, we consider an SDE
\begin{equation}\label{SDE}
dX_t = a(X_t)dt + \beta dW_t \,,
\end{equation}
where $a: \mathbb{R}^d \to \mathbb{R}^d$, $\beta \in \mathbb{R}_{+}$ and $(W_t)_{t \geq 0}$ is a $d$-dimensional Brownian motion. Furthermore, let $(Z_k)_{k=1}^{\infty}$ be i.i.d.\ random variables in $\mathbb{R}^d$ with $Z_k \sim \mathcal{N}(0,I)$ for $k \geq 1$. For a fixed discretisation parameter $h > 0$ we consider a discretisation of (\ref{SDE}) given by
$X_{k+1} = X_k + ha(X_k) + \beta \sqrt{h} Z_{k+1}$,
as well as its inaccurate drift counterpart
\begin{equation}\label{eq ineuler}
X_{k+1} = X_k + h b(X_k, U_k) + \beta \sqrt{h} Z_{k+1} \,.
\end{equation}
Here $b:\bR^d \times \bR^n \rightarrow \bR^d$ is an unbiased estimator of $a$ in the sense that $(U_k)_{k=0}^{\infty}$ are mutually independent $\mathbb{R}^n$-valued random variables independent of $(Z_k)_{k=1}^{\infty}$ such that for any $k \geq$ 0 we have
\begin{equation}\label{estimator}
\bE[ b(x,U_k)]=a(x) \text{ for all } x \in \mathbb{R}^d \,.
\end{equation}
Note that for each $k$, the random variable $X_k$ is independent of $U_k$ and that $\bE[b(X_k,U_k)|X_k] = a(X_k)$. Moreover, note that the framework where the drift estimator $b(x,U)$ depends on a random variable $U$ is obviously a generalisation of (\ref{eq SGLDm}), since as a special case of (\ref{eq ineuler}) we can take $b(x,U) = - \nabla_x V (x, \mathcal{L}(U))$, where $\mathcal{L}(U)$ denotes the law of $U$. We use the name Stochastic Gradient Langevin Dynamics (SGLD) \cite{TrueCost, MajkaMijatovicSzpruch2018, Dalalyan2017user} to describe (\ref{eq ineuler}) even in the general abstract setting where $b$ and $a$ are not necessarily of gradient form.

This setting, besides having the obvious advantage of being more general than the one presented in Assumptions \ref{as main}, allows us also to reduce the notational complexity by replacing sums of gradients with general abstract functions $a$ and $b$. As a motivation for considering such a general framework, let us discuss an example related to generative models.

\begin{example}
	Let $\nu$ denote an unknown data measure, supported on $\mathbb R^D$, and let $\nu^m$ be its empirical approximation. While $D$ is typically very large, in many applications $\nu$  can be well approximated by a probability distribution supported on a lower dimensional space, say $\mathbb R^d$, with $d \ll D$. The aim of generative models \cite{goodfellow2014generative} is to map samples from some basic distribution $\mu$ supported on $\mathbb R^d$, into samples from $\nu$. More precisely, one considers a parametrised map $G:\mathbb R^d \times \Theta \rightarrow \mathbb R^D$, with a parameter space $\Theta\subseteq \mathbb R^p$, that transports $\mu$ into  $G(\theta)_{\#}\mu := \mu(G(\theta)^{-1}(B))$, $B\in\mathbb R^D$. One then seeks $\theta$ such that $G(\theta)_{\#}\mu$ is a good approximation of $\nu$ with respect to a user-specified metric. In this example we consider  $f:\mathbb R^d \rightarrow \mathbb R$ and $\Phi:\mathbb R \rightarrow \mathbb R$ and define
	\[
	\text{dist}(G(\theta)_{\#}\mu,\nu^m):=\Phi\left( \int f(x)(G(\theta)_{\#}\mu)(dx)  - \int f(x)\nu^m(dx) \right)\,.
	\]
	A popular choice for the generator $G$ is a neural network \cite{goodfellow2014generative}.  In the case of a one-hidden-layer network with $\theta=(\alpha,\beta) \in \mathbb R^p \times \mathbb R^p$ with the activation function $\psi :\mathbb R^d \rightarrow \mathbb R^d$, one takes
	\[
	G(x,\theta):= \frac1p\sum_{i=1}^p \beta_{i} \psi(\a_{i} x)\,.
	\]
	With this choice of $G$, the authors of \cite{hu2019mean} derived a gradient flow equation that minimises suitably regularised $\text{dist}(G(\theta)_{\#}\mu,\nu^m)$. The gradient flow identified in \cite{hu2019mean}, when discretised, is given by
	\begin{equation}\label{eq generative chain}
	\theta_{k+1} = \theta_k - h \left( b(\theta_k,\nu^m) - \frac{\sigma^2}{2}\nabla U(\theta_{k}) \right) + \sigma \sqrt{h} Z_{k+1}\,,
	\end{equation}
	with $U:\mathbb R^{d} \rightarrow \mathbb R$ being a regulariser, $\sigma > 0$ a regularisation parameter, $(Z_k)_{k=1}^{\infty}$ a sequence of i.i.d.\ random variables with the standard normal distribution, and 
	\[
	b(\theta,\nu^m) = \nabla_x \Phi \left(  \int (f \circ G)(x,\theta)  \mu(dx) - \int f(x)\nu^m(dx)\right)
	\int \nabla_{\theta} (f\circ G)(x,\theta) \mu(dx)\,.  
	\]
	We refer the reader to \cite{hu2019mean} for more details and to \cite{jabir2019mean} for an extension to deep neural networks. One can see that $b$ may depend on the data in a non-linear way and hence the general setting of (\ref{eq ineuler}) becomes necessary for the analysis of the stochastic gradient counterpart of (\ref{eq generative chain}). An application of MASGA to the study of generative models will be further developed in a future work. 
\end{example}

In order to analyse the MASGA estimator (\ref{eq AMLMC}), we need to interpret the Markov chain
\begin{equation}\label{sect3chain}
X^{s,h}_{k+1} = X^{s,h}_k + h b(X^{s,h}_k, U^s_k) + \beta \sqrt{h} Z_{k+1}
\end{equation}
as characterized by two parameters: the discretisation parameter $h > 0$ and the drift estimation parameter $s \in \mathbb{N}$ that corresponds to the quality of approximation of $a(x)$ by $b(x,U^s_k)$, for some mutually independent random variables $(U^s_k)_{k=0}^{\infty}$. We will now carefully explain how to implement the antithetic MLMC framework from Section \ref{section:mainResult} in this setting.

To this end, suppose that we have a decreasing sequence $(h_\ell)_{\ell=0}^{L} \subset \mathbb{R}_{+}$ of discretisation parameters and an increasing sequence $(s_\ell)_{\ell=0}^{L} \subset \mathbb{N}_{+}$ of drift estimation parameters for some $L \geq 1$. For any $\ell_1$, $\ell_2 \in \{ 1, \ldots, L \}$ and any function $f : \mathbb{R}^d \to \mathbb{R}$, we define
\begin{equation}\label{introduction:DeltaPhi}
\Delta \Phi^{s_{\ell_1}, h_{\ell_2}}_{f,k} := \left( f(X^{s_{\ell_1},h_{\ell_2}}_k) - f(X^{s_{\ell_1}, h_{\ell_2-1}}_k) \right) - \left( f(X^{s_{\ell_1-1},h_{\ell_2}}_k) - f(X^{s_{\ell_1-1}, h_{\ell_2-1}}_k) \right) \,,
\end{equation}
and we also put $\Delta \Phi^{s_{0}, h_{0}}_{f,k} := f(X^{s_0,h_0}_k)$, $\Delta \Phi^{s_{0}, h_{1}}_{f,k} := f(X^{s_0,h_1}_k) - f(X^{s_0,h_0}_k)$ and $\Delta \Phi^{s_{1}, h_{0}}_{f,k} := f(X^{s_1,h_0}_k) - f(X^{s_0,h_0}_k)$. Then we can define a Multi-index Monte Carlo estimator
\begin{equation}\label{introduction:MultiIndexEstimator}
\mathcal{A} := \sum_{\ell_1 = 0}^{L} \sum_{\ell_2 = 0}^{L} \frac{1}{N_{\ell_1,\ell_2}} \sum_{j=1}^{N_{\ell_1,\ell_2}} \Delta \Phi^{s_{\ell_1}, h_{\ell_2}, (j)}_{f,k} \,,
\end{equation}
where $\Delta \Phi^{s_{\ell_1}, h_{\ell_2}, (j)}_{f,k}$ for $j = 1, \ldots, N_{\ell_1,\ell_2}$ are independent copies of $\Delta \Phi^{s_{\ell_1}, h_{\ell_2}}_{f,k}$. Here $N_{\ell_1,\ell_2}$ is the number of samples at the (doubly-indexed) level $\bm \ell = (\ell_1, \ell_2)$. Note that (\ref{introduction:MultiIndexEstimator}) corresponds to the regular (non-antithetic) MLMC estimator defined in (\ref{eq MLMC}) with $r = 2$ and with $L$ levels for both parameters.

We will now explain how to obtain the MASGA estimator (\ref{eq AMLMC}) by modifying (\ref{introduction:MultiIndexEstimator}) by replacing the difference operator $\Delta \Phi^{s_{\ell_1}, h_{\ell_2}}_{f,k}$ with its antithetic counterpart. To this end, we will need to take a closer look at the relation between the chains (\ref{sect3chain}) on different levels. From now on, we focus on sequences of parameters $h_\ell := 2^{-\ell}h_0$ and $s_\ell := 2^\ell s_0$ for $\ell \in \{ 1, \ldots, L \}$ and some fixed $h_0 > 0$ and $s_0 \in \mathbb{N}$. Then, we observe that for a fixed $s_{\ell_1}$, the chain $X^{s_{\ell_1}, h_{\ell_2}}$ has twice as many steps as the chain $X^{s_{\ell_1}, h_{\ell_2-1}}$, i.e., for any $k \geq 0$ we have
\begin{equation}\label{sect3fineandcoarse}
\begin{split}
X^{s_{\ell_1},h_{\ell_2}}_{k+2} &= X^{s_{\ell_1},h_{\ell_2}}_{k+1} + h_{\ell_2} b(X^{s_{\ell_1},h_{\ell_2}}_{k+1}, U^{s_{\ell_1},h_{\ell_2}}_{k+1}) + \beta \sqrt{h_{\ell_2}}Z_{k+2} \\
X^{s_{\ell_1},h_{\ell_2}}_{k+1} &= X^{s_{\ell_1},h_{\ell_2}}_{k} + h_{\ell_2} b(X^{s_{\ell_1},h_{\ell_2}}_{k+1}, U^{s_{\ell_1},h_{\ell_2}}_{k}) + \beta \sqrt{h_{\ell_2}}Z_{k+1} \\
X^{s_{\ell_1},h_{\ell_2-1}}_{k+2} &= X^{s_{\ell_1},h_{\ell_2-1}}_{k} + 2h_{\ell_2} b(X^{s_{\ell_1},h_{\ell_2-1}}_{k}, U^{s_{\ell_1},h_{\ell_2-1}}_{k}) + \beta \sqrt{2h_{\ell_2}}(Z_{k+2} + Z_{k+1}) \,.
\end{split}
\end{equation}
Throughout the paper, we will refer to $(X^{s_{\ell_1},h_{\ell_2}}_k)_{k \in \mathbb{N}}$ as the fine chain, and to $(X^{s_{\ell_1},h_{\ell_2-1}}_k)_{k\in 2 \mathbb{N}}$ as the coarse chain. 
Note that for the chain $(X^{s_{\ell_1},h_{\ell_2-1}}_k)_{k\in 2 \mathbb{N}}$ we could in principle use a sequence of standard Gaussian random variables $(\hat{Z}_k)_{k \in 2 \mathbb{N}}$ completely unrelated to the one that we use for $(X^{s_{\ell_1},h_{\ell_2}}_k)_{k \in \mathbb{N}}$ (which is $(Z_k)_{k \in \mathbb{N}}$). However, we choose $\hat{Z}_{k+2} := (Z_{k+1} + Z_{k+2})/\sqrt{2}$ for all $k \geq 0$, which corresponds to using the synchronous coupling between levels (which turns out to be a good choice in the global convexity setting as in Assumptions \ref{as main}, cf.\ \cite{MajkaSzpruchVollmerZygalakisGiles2019}). Moreover, note that since the chain $X^{s_{\ell_1}, h_{\ell_2}}$ moves twice as frequently as $X^{s_{\ell_1}, h_{\ell_2-1}}$, it needs twice as many random variables $U^{s_{\ell_1}, h_{\ell_2}}$ as $X^{s_{\ell_1}, h_{\ell_2-1}}$ needs $U^{s_{\ell_1}, h_{\ell_2-1}}$. This can be interpreted as having to choose how to estimate the drift $a$ twice as frequently (at each step of the chain).

The idea of the antithetic estimator (with respect to the discretisation parameter) involves replacing $f(X^{s_{\ell_1}, h_{\ell_2-1}})$ in (\ref{introduction:DeltaPhi}) with a mean of its two independent copies, i.e., with the quantity given by $\frac{1}{2}\left( f(X^{s_{\ell_1}, h_{\ell_2-1}-}) + f(X^{s_{\ell_1}, h_{\ell_2-1}+}) \right)$, where the first copy $X^{s_{\ell_1}, h_{\ell_2-1}-}$ uses $U^{s_{\ell_1}, h_{\ell_2}}_k$ and the other copy $X^{s_{\ell_1}, h_{\ell_2-1}+}$ uses $U^{s_{\ell_1}, h_{\ell_2}}_{k+1}$ to estimate the drift, instead of drawing their own independent copies of $U^{s_{\ell_1},h_{\ell_2-1}}_{k}$, i.e.,
\begin{equation*}
\begin{split}
X^{s_{\ell_1},h_{\ell_2-1}-}_{k+2} &= X^{s_{\ell_1},h_{\ell_2-1}-}_{k} + 2h_{\ell_2} b(X^{s_{\ell_1},h_{\ell_2-1}-}_{k}, U^{s_{\ell_1},h_{\ell_2}}_{k}) + \beta \sqrt{2h_{\ell_2}}(Z_{k+2} + Z_{k+1}) \\
X^{s_{\ell_1},h_{\ell_2-1}+}_{k+2} &= X^{s_{\ell_1},h_{\ell_2-1}+}_{k} + 2h_{\ell_2} b(X^{s_{\ell_1},h_{\ell_2-1}+}_{k}, U^{s_{\ell_1},h_{\ell_2}}_{k+1}) + \beta \sqrt{2h_{\ell_2}}(Z_{k+2} + Z_{k+1}) \,.
\end{split}
\end{equation*}
Hence the term $\left( f(X^{s_{\ell_1},h_{\ell_2}}_k) - f(X^{s_{\ell_1}, h_{\ell_2-1}}_k) \right)$ appearing in (\ref{introduction:DeltaPhi}) would be replaced with the antithetic term $\left( f(X^{s_{\ell_1},h_{\ell_2}}_k) - \frac{1}{2}\left( f(X^{s_{\ell_1}, h_{\ell_2-1}-}) + f(X^{s_{\ell_1}, h_{\ell_2-1}+})  \right) \right)$ and the same can be done for any fixed $s_{\ell_1}$. 
Let us explain the intuition behind this approach on a simple example with $f(x) = x$ and a state-independent drift $a$. 

\begin{example}\label{exampleAMLMC}
	We fix $s_{\ell_1}$ and suppress the dependence on $\ell_1$ in the notation, in order to focus only on MLMC via discretisation parameter.
	Let $\xi = (\xi_1, \ldots, \xi_m)$ be a collection of $m$ data points and let us consider a state-independent drift $a = \frac{1}{m}\sum_{i=1}^{m} \xi_i$ and its unbiased estimator 
	$b(U^s) := \frac{1}{s} \sum_{i=1}^{s} \xi_{U^s_i}$ 
	for $s \leq m$, where $U^s = (U^s_1, \ldots U^s_s)$ is such that $U^s_j \sim \operatorname{Unif}\{ 1, \ldots , m \}$ for all $j \in \{1, \ldots, s \}$ (i.e., we sample with replacement $s$ data points from the data set $\xi$ of size $m$). Consider the standard MLMC estimator with the fine $X^f$ and the coarse $X^c$ schemes defined as
	\begin{equation*}
	\begin{split}
	X^f_{k+1} &= X^f_{k} + h b(U^{s,f}_k) + \beta \sqrt{h} Z_{k+1} \,, \qquad X^f_{k+2} = X^f_{k+1} + h b(U^{s,f}_{k+1}) + \beta \sqrt{h} Z_{k+2} \,, \\
	X^c_{k+2} &= X^c_{k} + 2h b(U^{s,c}_k) + \beta \sqrt{2h} \hat{Z}_{k+2} \,,
	\end{split}
	\end{equation*}
	where $\hat{Z}_{k+2} := (Z_{k+1} + Z_{k+2})/\sqrt{2}$ , $\beta = 1/\sqrt{m}$, with $U^{s,f}_{k}$, $U^{s,f}_{k+1}$ and $U^{s,c}_{k}$ being independent copies of $U^s$ for any $k \geq 0$. Note that $(X^f_k)_{k \in \mathbb{N}}$ corresponds to $(X^{s_{\ell_1},h_{\ell_2}}_k)_{k \in \mathbb{N}}$ in (\ref{sect3fineandcoarse}) whereas $(X^c_k)_{k \in 2\mathbb{N}}$ corresponds to $(X^{s_{\ell_1},h_{\ell_2-1}}_k)_{k \in 2\mathbb{N}}$ for some fixed $\ell_1$, $\ell_2$. Recall from the discussion in Section \ref{section:mainResult} that our goal is to find a sharp upper bound on the variance (or the second moment) of $X^f_k - X^c_k$ for any $k \geq 1$ (which corresponds to bounding the variance of the standard, non-antithetic MLMC estimator (\ref{introduction:MultiIndexEstimator}) for a Lipschitz function $f$, cf.\ the difference (\ref{introduction:DeltaPhi}) taken only with respect to the time-discretisation parameter $h$, with fixed $s$). We have
	\begin{equation*}
	\begin{split}
	&\mathbb{E}\left| X^f_{k+2} - X^c_{k+2} \right|^2 = \mathbb{E} \left| X^f_{k} - X^c_{k} \right|^2 + \mathbb{E} \left\langle X^f_{k} - X^c_{k}, h b(U^{s,f}_k) + h b(U^{s,f}_{k+1}) - 2h b(U^{s,c}_k) \right\rangle \\
	&+ \mathbb{E} \left|h b(U^{s,f}_k) + h b(U^{s,f}_{k+1}) - 2h b(U^{s,c}_k) \right|^2 = \mathbb{E} \left| X^f_{k} - X^c_{k} \right|^2 + h^2 \mathbb{E} \left|   b(U^{s,f}_k) +  b(U^{s,f}_{k+1}) - 2 b(U^{s,c}_k) \right|^2 \,,
	\end{split}
	\end{equation*}
	where in the second step we used conditioning and the fact that $b$ is an unbiased estimator of $a$.
	Hence we can show that, if we choose $X^f_0 = X^c_0$, then for all $k \geq 1$ we have
	$\mathbb{E}\left| X^f_{k} - X^c_{k} \right|^2 \leq C h$
	for some $C > 0$  and we get a variance contribution of order $h$. On the other hand, if we want to apply the antithetic approach as in (\ref{eq AMLMC}), we can define
	\begin{equation*}
	X^{c-}_{k+2} = X^{c-}_{k} + 2h b(U^{s,f}_k) + \beta \sqrt{2h} \hat{Z}_{k+2} \,, \qquad
	X^{c+}_{k+2} = X^{c+}_{k} + 2h b(U^{s,f}_{k+1}) + \beta \sqrt{2h} \hat{Z}_{k+2} 
	\end{equation*}
	with $\beta = 1/\sqrt{m}$, and, putting $\bar{X}^c_{k} := \frac{1}{2} \left( X^{c-}_{k} + X^{c+}_{k} \right)$, we obtain
	\begin{equation*}
	\bar{X}^c_{k+2} = \frac{1}{2} \left( X^{c-}_{k} + X^{c+}_{k} \right) + h \left( b(U^{s,f}_k) + b(U^{s,f}_{k+1}) \right) + \beta \sqrt{2h} \hat{Z}_{k+2} \,.
	\end{equation*}
	Then we have
	$\mathbb{E}\left| X^f_{k+2} - \bar{X}^c_{k+2} \right|^2  = \mathbb{E}\left| X^f_{k} - \bar{X}^c_{k} \right|^2$
	and, choosing $X^f_0 = X^{c-}_0 = X^{c+}_0$, the variance contribution vanishes altogether.
\end{example}

In the general case, $b(x,U^s)$ is a nonlinear function of the data and also depends on the state $x$. Therefore one should not expect that the variance of the drift estimator can be completely mitigated. Nonetheless, careful analysis will allow us to conclude that the application of the antithetic difference operators on all levels in our MASGA estimator allows us to obtain a desired upper bound on the variance as described in Section \ref{section:mainResult}.

Having explained the motivation behind the antithetic approach to MLMC, let us now focus on the antithetic estimator with respect to the drift estimation parameter $s$. To this end,
let us now fix $h_{\ell_2}$ and observe that for the chain $X^{s_{\ell_1}, h_{\ell_2}}$, the value of the drift estimation parameter $s_{\ell_1} = 2s_{\ell_1-1}$ is twice the value for the chain $X^{s_{\ell_1-1}, h_{\ell_2}}$. In the context of the subsampling drift as in Assumptions \ref{as main}, this corresponds to the drift estimator in $X^{s_{\ell_1}, h_{\ell_2}}$ using twice as many samples as the drift estimator in $X^{s_{\ell_1-1}, h_{\ell_2}}$. Hence, instead of using independent samples for $X^{s_{\ell_1-1}, h_{\ell_2}}$, we can consider two independent copies of $X^{s_{\ell_1-1}, h_{\ell_2}}$, the first of which uses the first half $U^{s_{\ell_1},h_{\ell_2},1}_{k}$ of samples of $X^{s_{\ell_1}, h_{\ell_2}}$ and the other uses the second half $U^{s_{\ell_1},h_{\ell_2},2}_{k}$, namely,
\begin{equation}\label{sect3AMLMCviasubsampling}
\begin{split}
X^{s_{\ell_1-1}-, h_{\ell_2}}_{k+1} &:= X^{s_{\ell_1-1}-, h_{\ell_2}}_k + h_{\ell_2} b(X^{s_{\ell_1-1}-, h_{\ell_2}}_k, U^{s_{\ell_1},h_{\ell_2},1}_{k}) + \beta \sqrt{h_{\ell_2}} Z_{k+1} \\
X^{s_{\ell_1-1}+, h_{\ell_2}}_{k+1} &:= X^{s_{\ell_1-1}+, h_{\ell_2}}_k + h_{\ell_2} b(X^{s_{\ell_1-1}+, h_{\ell_2}}_k, U^{s_{\ell_1},h_{\ell_2},2}_{k}) + \beta \sqrt{h_{\ell_2}} Z_{k+1} \,.
\end{split}
\end{equation}
Hence, using the antithetic approach, we could replace $\left( f(X^{s_{\ell_1}, h_{\ell_2}}_k) - f(X^{s_{\ell_1-1}, h_{\ell_2}}_k) \right)$ in (\ref{introduction:DeltaPhi}) with the difference $\left( f(X^{s_{\ell_1}, h_{\ell_2}}_k) - \frac{1}{2}\left( f(X^{s_{\ell_1-1}-, h_{\ell_2}}_k) + f(X^{s_{\ell_1-1}+, h_{\ell_2}}_k) \right) \right)$. 

Combining the ideas of antithetic estimators both with respect to the parameter $s$ and $h$, we arrive at a nested antithetic difference $\Delta \operatorname{Ant} \Phi^{s_{\ell_1}, h_{\ell_2}}_{f,k}$, defined for any $\ell_1$, $\ell_2 \in \{ 1, \ldots L \}$ and any $k \geq 1$ as
\begin{equation}\label{sect3nineterms}
\begin{split}
\Delta \operatorname{Ant} \Phi^{s_{\ell_1}, h_{\ell_2}}_{f,k} &:= \left[ f(X^{s_{\ell_1}, h_{\ell_2}}_k) - \frac{1}{2} \left( f(X^{s_{\ell_1}, h_{\ell_2-1}-}_k) + f(X^{s_{\ell_1}, h_{\ell_2-1}+}_k) \right) \right] \\
&- \frac{1}{2}\Bigg[ \left( f(X^{s_{\ell_1-1}-,h_{\ell_2}}_k) - \frac{1}{2}\left( f(X^{s_{\ell_1-1}-,h_{\ell_2-1}-}_k) + f(X^{s_{\ell_1-1}-,h_{\ell_2-1}+}_k) \right) \right) \\
&+ \left( f(X^{s_{\ell_1-1}+,h_{\ell_2}}_k) - \frac{1}{2}\left( f(X^{s_{\ell_1-1}+,h_{\ell_2-1}-}_k) + f(X^{s_{\ell_1-1}+,h_{\ell_2-1}+}_k) \right) \right)\Bigg] \,,
\end{split}
\end{equation}
with the same convention as in (\ref{introduction:DeltaPhi}) for the case of $\ell_1 = 0$ or $\ell_2 = 0$. We can now plug this difference into the definition of a Multi-index Monte Carlo estimator (\ref{introduction:MultiIndexEstimator}) to obtain
\begin{equation}\label{sect3MASGA}
\operatorname{Ant} \mathcal{A} := \sum_{\ell_1=0}^{L}  \sum_{\ell_2=0}^{L} \frac{1}{N_{\ell_1,\ell_2}} \sum_{j=1}^{N_{\ell_1,\ell_2}} \Delta \operatorname{Ant} \Phi^{s_{\ell_1}, h_{\ell_2}, (j)}_{f,k} \,.
\end{equation}
Note that this  corresponds to the Antithetic MIMC estimator introduced in (\ref{eq AMLMC}), based on the chain (\ref{eq AMLMC special}), with $\ell_1$ corresponding to the number of samples $s$ and $\ell_2$ to the discretisation parameter $h$, but with a more general drift $a$ and its estimator $b$.

In order to formulate our result for the general setting presented in this section, we need to specify the following assumptions.

\begin{assumption}[Lipschitz condition and global contractivity of the drift] \label{as diss}
	The drift function $a : \mathbb{R}^d \to \mathbb{R}^d$ satisfies the following conditions:
	\begin{enumerate}[i)]
		\item Lipschitz condition: there is a constant $L > 0$ such that
		\begin{equation}\label{driftLipschitz}
		|a(x) - a(y)| \leq L|x - y| \qquad \text{ for all } x, y \in \mathbb{R}^d \,.
		\end{equation}
		\item Global contractivity condition: there exists a constant $K > 0$ such that 
		\begin{equation}\label{driftDissipativity}
		\langle  x - y , a(x) - a(y) \rangle  \leq -K|x - y|^2 \qquad \text{ for all } x , y \in \mathbb{R}^d \,.
		\end{equation}
		\item Smoothness: $a \in C^2_b(\mathbb{R}^d; \mathbb{R}^d)$ (where $C^2_b(\mathbb{R}^d; \mathbb{R}^d)$ is defined as in Assumptions \ref{as main}). In particular, there exist constants $C_{a^{(1)}}$, $C_{a^{(2)}} > 0$ such that
		\begin{equation}\label{driftSmoothness} 
		|D^{\alpha}a(x)| \leq C_{a^{(|\alpha|)}}
		\end{equation}
		for all $x \in \mathbb{R}^d$ and for all multiindices $\alpha$ with $|\alpha| = 1$, $2$.
	\end{enumerate}
\end{assumption}

We remark that condition (\ref{driftSmoothness}) could be easily removed by approximating a non-smooth drift $a$ with suitably chosen smooth functions. This, however, would create additional technicalities in the proof and hence we decided to work with (\ref{driftSmoothness}).

We now impose the following assumptions on the estimator $b$ of the drift $a$.

\begin{assumptionsingle}[Lipschitz condition of the estimator] \label{as Lipschitz estimator}
	There is a constant $\bar{L} > 0$ such that for all $x$, $y \in \mathbb{R}^d$ and all random variables $U^s$ such that $\mathbb{E}[b(x,U^s)] = a(x)$ for all $x \in \mathbb{R}^d$, we have
	\begin{equation}\label{ina driftLipschitz}
	\mathbb{E}|b(x,U^s) - b(y,U^s)| \leq \bar{L}|x - y| \,.
	\end{equation}
\end{assumptionsingle}

\begin{assumptionsingle}[Variance of the estimator] \label{as ina}
	There exists a constant $\sigma > 0$ of order $\mathcal{O}(s^{-1})$ such that for any $x \in \mathbb{R}^d$ and any random variable $U^s$ such that $\mathbb{E}[b(x,U^s)] = a(x)$ for all $x \in \mathbb{R}^d$, we have
	\begin{equation}\label{ina estimatorVariance}
	\mathbb{E} \left| b(x, U^s) - a(x) \right|^2 \leq \sigma^2 (1+|x|^2) \,.
	\end{equation}
\end{assumptionsingle}

\begin{assumptionsingle}[Fourth centered moment of the estimator]\label{as fourth moment}
	There exists a constant $\sigma^{(4)} \geq 0$ of order $\mathcal{O}(s^{-2})$ such that for any $x \in \mathbb{R}^d$ and for any random variable $U^s$ such that $\mathbb{E}[b(x,U^s)] = a(x)$ for all $x \in \mathbb{R}^d$, we have
	\begin{equation}\label{fourthCenteredMomentOfEstimator}
	\mathbb{E} \left| b(x, U^s) - a(x) \right|^4 \leq \sigma^{(4)} (1+|x|^4) \,.
	\end{equation}
\end{assumptionsingle}

Note that obviously Assumption \ref{as fourth moment} implies Assumption \ref{as ina}. However, we formulate these conditions separately in order to keep track of the constants in the proofs. Moreover, with the same constant $\sigma^{(4)}$ as in (\ref{fourthCenteredMomentOfEstimator}), we impose

\begin{assumptionsingle}\label{as smoothnessEstimator}
 The estimator $b(x,U)$ as a function of $x$ is twice continuously differentiable for any $U$ and, for any $x \in \mathbb{R}^d$, we have
		$\mathbb{E} \left| \nabla b(x,U^s) - \nabla a(x) \right|^4 \leq \sigma^{(4)}(1+|x|^4)$.
\end{assumptionsingle}

Note that $\nabla a(x)$, $\nabla b(x,U^s) \in \mathbb{R}^{d \times d}$ and we use the matrix norm $|\nabla a(x)|^2 := \sum_{i,j=1}^{d} |\partial_i a_j(x)|^2$, where $a(x) = (a_1(x), \ldots, a_d(x))$.

\begin{assumption}\label{as additionalAssumption}
	Partial derivatives of the estimator of the drift are estimators of the corresponding partial derivatives of the drift. More precisely, for any multi-index $\alpha$ with $|\alpha| \leq 2$ and for any random variable $U^s$ such that $\mathbb{E}[b(x,U^s)] = a(x)$ for all $x \in \mathbb{R}^d$, we have $\mathbb{E}[D^{\alpha} b(x,U^s)] = D^{\alpha} a(x)$ for any $x \in \mathbb{R}^d$.
\end{assumption}

\begin{assumptionsingle}[Growth of the drift]\label{as linear fourth}
	There exists a constant $L_0^{(4)} > 0$ such that for all $x \in \mathbb{R}^d$ we have
	\begin{equation}\label{linearGrowthFourth}
	|a(x)|^4 \leq L_0^{(4)}\left(1 + |x|^4\right) \,.
	\end{equation}
\end{assumptionsingle}

Finally, we have the following condition that specifies the behaviour of the drift estimator $b$ with respect to the random variables $U^{s_{\ell_1},h_{\ell_2},1}_{k}$ and $U^{s_{\ell_1},h_{\ell_2},2}_{k}$ introduced in (\ref{sect3AMLMCviasubsampling}).

\begin{assumptionsingle}\label{as splittingEstimator}
	For any $x \in \mathbb{R}^d$ we have
	$b(x, U^{s_{\ell_1},h_{\ell_2}}_{k}) = \frac{1}{2}b(x, U^{s_{\ell_1},h_{\ell_2},1}_{k}) + \frac{1}{2}b(x, U^{s_{\ell_1},h_{\ell_2},2}_{k})$.
\end{assumptionsingle}

Even though this set of conditions is long, the assumptions are in fact rather mild and it is an easy exercise to verify that when
\begin{equation*}
a(x) := \frac{1}{m}\sum_{i=1}^{m} \nabla_x v(x,\xi_i) \qquad \text{ and } \qquad b(x,U^s) := \frac{1}{s} \sum_{i=1}^{s} \nabla_x v(x, \xi_{U^s_i}) \,,
\end{equation*}
where $U^s_i \sim \operatorname{Unif}(\{ 1, \ldots , m \})$ for $i \in \{ 1, \ldots , s \}$ are i.i.d.\ random variables, uniformly distributed on $\{ 1, \ldots, m \}$, whereas $v : \mathbb{R}^d \times \mathbb{R}^k \to \mathbb{R}$ is the function satisfying Assumptions \ref{as main}, then $a$ and $b$ satisfy all Assumptions \ref{as diss}, \ref{as Lipschitz estimator}, \ref{as ina}, \ref{as fourth moment}, \ref{as smoothnessEstimator}, \ref{as additionalAssumption}, \ref{as linear fourth} and \ref{as splittingEstimator}. The only conditions that actually require some effort to be checked are Assumptions \ref{as ina} and \ref{as fourth moment}, however, they can be verified by extending the argument from Example 2.15 in \cite{MajkaMijatovicSzpruch2018}, where a similar setting was considered. As it turns out, these conditions hold also for the case of subsampling without replacement. All the details are provided in Appendix \ref{sectionSubsampling}.

We have the following result.
\begin{theorem}\label{thm:mainTheorem}
	Under Assumptions \ref{as diss}, \ref{as Lipschitz estimator}, \ref{as ina}, \ref{as fourth moment}, \ref{as smoothnessEstimator}, \ref{as additionalAssumption}, \ref{as linear fourth} and \ref{as splittingEstimator} on the drift $a$ and its estimator $b$, the MASGA estimator (\ref{sect3MASGA})
	achieves the mean square error smaller than $\varepsilon > 0$ at the computational cost $\varepsilon^{-2}$. Here, at each level $(\ell_1,\ell_2) \in \{ 0, \ldots , L \}^2$, the number of paths $N_{\ell_1,\ell_2}$ is given by $\varepsilon^{-2}\left( \sqrt{\frac{\mathbb{V}[(f, \bm \Delta^{A} \pi^{(\ell_1,\ell_2)})]}{C_{(\ell_1,\ell_2)}}} \sum_{\bm \ell \in [L]^2} \sqrt{\mathbb{V}[(f, \bm \Delta^{A} \pi^{\bm \ell})] C_{\bm \ell}} \right) $, where $[L]^2 := \{0, \ldots , L \}^2$, $C_{\bm \ell}$ is defined via (\ref{eq cost def}) and $(f, \bm \Delta^{A} \pi^{(\ell_1,\ell_2)}) := \Delta \operatorname{Ant} \Phi^{s_{\ell_1},h_{\ell_2}}_{f,k}$ given by (\ref{sect3nineterms}). 
\end{theorem}

As we explained in Section \ref{section:mainResult}, the proof of Theorem \ref{theMainTheorem} (and its generalisation Theorem \ref{thm:mainTheorem}) relies on the MIMC complexity analysis from \cite{Haji-Ali2016} (see also \cite{Giles2015Acta}). 
\begin{theorem} \label{th complexity}
	Fix $\varepsilon \in (0, e^{-1})$. Let $(\bm \alpha, \bm \beta, \bm \gamma)\in \mathbb R^{3r}$ be a triplet of vectors in $\mathbb{R}^r$ such that for all $k \in \{ 1, \ldots , r \}$ we have $\alpha_k \geq \frac{1}{2} \beta_{k}$. Assume that for each $\bm \ell \in \mathbb{N}^r$
	\[
	i)\,\,| \mathbb E[(f, \bm \Delta^A \pi^{\bm \ell})]| \lesssim 2^{- \langle  \bm \alpha , \bm \ell \rangle}, 
	\quad ii) \,\,\mathbb{V}[(f, \bm \Delta^A \pi^{\bm \ell})] \lesssim 2^{- \langle \bm \beta,  \bm \ell \rangle}	
	\quad iii)\,\,  C_{\bm \ell} \lesssim 2^{\langle \bm \gamma , \bm \ell \rangle} \,.
	\]
	If $\max_{k\in[1,\ldots,r]}\frac{(\gamma_k-\beta_k)}{\alpha_k} < 0$, then there exists a set $\mathcal L \subset \mathbb{N}^r$ and a sequence $(N_{\bm \ell})_{\bm \ell\in \mathcal L}$ such that the MLMC estimator $\mathcal{A}^{A, \mathcal{L}}(f)$ defined in \eqref{eq AMLMC} satisfies
	\[
	\mathbb E[ ((f,\pi) -  \mathcal{A}^{A, \mathcal{L}}(f))^2  ] < \varepsilon^2\,,
	\]
	with the computational cost $\varepsilon^{-2}$.
\end{theorem}

The key challenge in constructing and analysing MIMC estimators is to establish conditions i)-iii) in Theorem \ref{th complexity} i.e., to show that the leading error bounds for the bias, variance and cost can be expressed in the product form. In fact, there are very few results in the literature that present the analysis giving i)-iii), with the exception of \cite[Section 9]{Giles2015Acta} and \cite{giles2019multilevel}. The bulk of the analysis in this paper is devoted to the analysis of ii). We remark that the optimal choice of $\mathcal L=[L]^2$ is dictated by the relationship between $(\bm \alpha, \bm \beta, \bm \gamma)$, see \cite{Haji-Ali2016}.

The following lemma will be crucial for the proof of Theorem \ref{thm:mainTheorem}.

\begin{lemma}\label{thm:crucialLemma}
	Let Assumptions \ref{as diss}, \ref{as Lipschitz estimator}, \ref{as ina}, \ref{as fourth moment}, \ref{as smoothnessEstimator}, \ref{as additionalAssumption}, \ref{as linear fourth} and \ref{as splittingEstimator} hold. Then there is a constant $h_0 > 0$ and a constant $C > 0$ (independent of $s$ and $h$) such that for any Lipschitz function $f: \mathbb{R}^d \to \mathbb{R}$, for any $s \geq 1$, $h \in (0, h_0)$ and for any $k \geq 1$,
	\begin{equation*}
	\mathbb{E}|\Delta \operatorname{Ant} \Phi^{s, h}_{f,k}|^2 \leq C h^2/s^2 \,.
	\end{equation*}
\end{lemma}

As we already indicated in Section \ref{section:mainResult}, once we have an upper bound on the second moment (and thus on the variance) of $\Delta \operatorname{Ant} \Phi^{s, h}_{f,k}$ such as in Lemma \ref{thm:crucialLemma}, the proof of Theorem \ref{thm:mainTheorem} becomes rather straightforward.

\begin{proof}[Proof of Theorem \ref{thm:mainTheorem}]
	Note that we have $\mathbb{V}[(f, \bm \Delta^A \pi^{\bm \ell})] \leq \mathbb{E} [(f, \bm \Delta^A \pi^{\bm \ell})^2] \lesssim h_{\ell_2}^2 / s_{\ell_1}^2$ due to Lemma \ref{thm:crucialLemma}. Moreover, the number of time-steps and the number of subsamples at each level of MIMC is doubled (i.e., we have $s_{\ell_1} = 2^{\ell_1}s_0$ and $h_{\ell_2} = 2^{-\ell_2}h_0$ for all $\ell_1$, $\ell_2 \in \{ 0, \ldots, L \}$) and hence $\bm \gamma = (1,1)$ and $\bm \beta=(2,2)$ in the assumption of Theorem \ref{th complexity} (recall that $C_{\bm \ell} \lesssim s_{\ell_1} h_{\ell_2}^{-1}$). Finally, we have
	\begin{equation*}
	|\mathbb{E} (f, \bm \Delta^A \pi^{\bm \ell})| \leq \left( \mathbb{E} |(f, \bm \Delta^A \pi^{\bm \ell})|^2 \right)^{1/2} \lesssim h_{\ell_2} / s_{\ell_1} \,,
	\end{equation*}
	 which implies that $\bm \alpha = (1,1)$. Hence the assumptions of Theorem \ref{th complexity} are satisfied and the overall complexity of MIMC is indeed $\varepsilon^{-2}$.
\end{proof}

\begin{proof}[Proof of Theorem \ref{theMainTheorem}]
	Under Assumptions \ref{as main}, the function $a(x) := - \nabla v_0(x) - \frac{1}{m}\sum_{i=1}^{m} \nabla_x v(x,\xi_i)$ and its estimator $b(x,U^s) := -\nabla v_0(x) - \frac{1}{s}\sum_{i=1}^{s} \nabla_x v(x,\xi_{U^s_i})$, where $U^s_i$ are mutually independent random variables, uniformly distributed on $\{ 1, \ldots m \}$, satisfy all Assumptions \ref{as diss}, \ref{as Lipschitz estimator}, \ref{as ina}, \ref{as smoothnessEstimator}, \ref{as additionalAssumption}, \ref{as linear fourth}, \ref{as fourth moment} and \ref{as splittingEstimator}. Hence we can just apply Theorem \ref{thm:mainTheorem} to conclude.
\end{proof}

\section{Analysis of AMLMC}\label{sectionProofs}

The estimator we introduced in (\ref{eq AMLMC}), see also (\ref{sect3MASGA}), can be interpreted as built from two building blocks: 
the antithetic MLMC estimator with respect to the discretisation parameter, which corresponds to taking $r=1$ and ${\bm \ell} = \ell_2$ in (\ref{eq AMLMC}), and the antithetic MLMC estimator with respect to subsampling, which corresponds to taking $r=1$ and ${\bm \ell} = \ell_1$ in (\ref{eq AMLMC}). Let us begin our analysis by focusing on the former.

\subsection{AMLMC via discretisation}\label{sectionAntitheticViaDiscretisation}

We will analyse one step of the MLMC algorithm, for some fixed level $\bm \ell$. To this end, let us first introduce the following fine $(X^f_k)_{k \in \mathbb{N}}$ and coarse $(X^c_k)_{k \in 2\mathbb{N}}$ chains
\begin{equation}\label{standardMLMCchains}
\begin{split}
X^f_{k+1} &= X^f_{k} + h b(X^f_{k}, U^f_{k}) + \beta \sqrt{h}Z_{k+1} \,, \, \,
X^f_{k+2} = X^f_{k+1} + h b(X^f_{k+1}, U^f_{k+1}) + \beta \sqrt{h} Z_{k+2} \\
X^c_{k+2} &= X^c_{k} + 2h b(X^c_{k}, U^c_{k}) + \beta \sqrt{2h} \hat{Z}_{k+2} \,,
\end{split}
\end{equation}
where $h > 0$ is fixed, $(U^f_{k})_{k=0}^{\infty}$ and $(U^c_{k})_{k=0}^{\infty}$ are mutually independent random variables such that for $U \in \{ U^f_{k}, U^c_{k} \}$ we have $\mathbb{E}[b(x,U)] = a(x)$ for all $x \in \mathbb{R}^d$ and all $k \geq 0$ and $(Z_k)_{k=1}^{\infty}$ are i.i.d.\ random variables with $Z_k \sim \mathcal{N}(0,I)$. We also have $\hat{Z}_{k+2} := \frac{1}{\sqrt{2}}\left( Z_{k+1} + Z_{k+2} \right)$. 

In order to analyse the antithetic estimator, we also need to introduce two auxiliary chains
\begin{equation}\label{antitheticChains}
\begin{split}
X^{c+}_{k+2} &= X^{c+}_{k} + 2h b(X^{c+}_{k}, U^f_{k}) + \beta \sqrt{2h} \hat{Z}_{k+2} \\
X^{c-}_{k+2} &= X^{c-}_{k} + 2h b(X^{c-}_{k}, U^f_{k+1}) + \beta \sqrt{2h} \hat{Z}_{k+2} \,.
\end{split}
\end{equation}
Furthermore, we denote
$\bar{X}^c_{k} = \frac{1}{2} \left( X^{c+}_{k} + X^{c-}_{k} \right)$.

Before we proceed, let us list a few simple consequences of Assumptions \ref{as diss}. We have the following bounds:
\begin{equation}\label{driftLinearGrowth}
|a(x)| \leq L_0(1 + |x|) \text{ for all } x \in \mathbb{R}^d \,, \text{ where } L_0 := \max (L, |a(0)|) \,.
\end{equation}
Let $M_2 := 4 L |a(0)|^2 / K^2 + 2 |a(0)|^2 / K$ and $M_1 : = K / 2$. Then we have
\begin{equation}\label{driftLyapunov}
\langle x, a(x) \rangle \leq M_2 - M_1|x|^2 \text{ for all } x \in \mathbb{R}^d \,.
\end{equation}
Finally, for all random variables $U$ satisfying (\ref{estimator}), we have
\begin{equation}\label{estimatorGrowth}
\mathbb{E} |b(x,U)|^2 \leq \bar{L}_0(1 + |x|^2) \text{ for all } x \in \mathbb{R}^d \,,
\end{equation}
where $\bar{L}_0 := \sigma^2 h^{\alpha} + 2 \max (L^2 , |a(0)|^2)$. Note that (\ref{driftLinearGrowth}) is an immediate consequence of (\ref{driftLipschitz}), (\ref{driftLyapunov}) follows easily from (\ref{driftLipschitz}) and (\ref{driftDissipativity}) (cf. the proof of Lemma 2.11 in \cite{MajkaMijatovicSzpruch2018}), whereas (\ref{estimatorGrowth}) is implied by (\ref{driftLipschitz}) and (\ref{ina estimatorVariance}), cf. (2.40) in \cite{MajkaMijatovicSzpruch2018}. Throughout our proofs, we will also use uniform bounds on the second and the fourth moments of Euler schemes with time steps $h$ and $2h$, i.e., we have 
\begin{equation*}
\mathbb{E}|X^f_{k}|^2 \leq C_{IEul} \,, \quad \mathbb{E}|X^{c-}_{k}|^2 \leq C^{(2h)}_{IEul} \,, \quad \mathbb{E}|X^f_{k}|^4 \leq C^{(4)}_{IEul} \quad \text{ and } \quad \mathbb{E}|X^{c-}_{k}|^4 \leq C^{(4),(2h)}_{IEul} \,,
\end{equation*}
where the exact formulas for the constants $C_{IEul}$, $C^{(2h)}_{IEul}$, $C^{(4)}_{IEul}$ and $C^{(4),(2h)}_{IEul}$ can be deduced from Lemma \ref{lem:fourthmoment} in the Appendix (for $C_{IEul}$, $C^{(2h)}_{IEul}$ see also Lemma 2.17 in \cite{MajkaMijatovicSzpruch2018}).

We now fix $g \in C_b^2(\mathbb{R}^d; \mathbb{R})$. Denote by $C_{g^{(1)}}$, $C_{g^{(2)}}$ positive constants such that $|D^{\alpha}g(x)| \leq C_{g^{(|\alpha|)}}$ for all $x \in \mathbb{R}^d$ and all multiindices $\alpha$ with $|\alpha| = 1$, $2$.

We begin by presenting the crucial idea of our proof. We will use the Taylor formula to write
\begin{equation*}
\begin{split}
g &(\bar{X}^c_{k}) - g(X^{c-}_{k}) = - \Big[ \sum_{|\alpha| = 1} D^{\alpha} g(\bar{X}^c_{k}) \left( \bar{X}^c_{k} - X^{c-}_{k} \right)^{\alpha} \\
&+ \sum_{|\alpha| = 2} \int_0^1 (1-t) D^{\alpha} g\left(\bar{X}^c_{k} + t\left(\bar{X}^c_{k} - X^{c-}_{k}\right)\right) dt \left( \bar{X}^c_{k} - X^{c-}_{k} \right)^{\alpha} \Big] \,.
\end{split}
\end{equation*}
We also express $g(\bar{X}^c_{k}) - g(X^{c+}_{k})$ in an analogous way. Note that 
\begin{equation}\label{antisymmetry}
\bar{X}^c_{k} - X^{c-}_{k} = \frac{1}{2} X^{c+}_{k} - \frac{1}{2} X^{c-}_{k} = - (\bar{X}^c_{k} - X^{c+}_{k}) 
\end{equation}
and hence we have
\begin{equation*}
\begin{split}
&g(X^f_{k}) - \frac{1}{2}\left( g(X^{c-}_{k}) + g(X^{c+}_{k}) \right) = g(X^f_{k}) - g (\bar{X}^c_{k}) + \frac{1}{2} \left( g (\bar{X}^c_{k}) - g(X^{c-}_{k}) + g (\bar{X}^c_{k}) - g(X^{c+}_{k}) \right) \\
&= g(X^f_{k}) - g (\bar{X}^c_{k}) - \frac{1}{2} \sum_{|\alpha| = 2} \int_0^1 (1-t) D^{\alpha} g\left(\bar{X}^c_{k} + t\left(\bar{X}^c_{k} - X^{c-}_{k}\right)\right) dt \left( \bar{X}^c_{k} - X^{c-}_{k} \right)^{\alpha} \\
&- \frac{1}{2} \sum_{|\alpha| = 2} \int_0^1 (1-t) D^{\alpha} g\left(\bar{X}^c_{k} + t\left(\bar{X}^c_{k} - X^{c+}_{k}\right)\right) dt \left( \bar{X}^c_{k} - X^{c+}_{k} \right)^{\alpha} \,,
\end{split}
\end{equation*} 
i.e., the first order terms in the Taylor expansions cancel out. Thus, using the inequalities $\sum_{|\alpha| = 2}  D^{\alpha} g(x) \left( \bar{X}^c_{k} - X^{c+}_{k} \right)^{\alpha} \leq \| \nabla^2 g \|_{\operatorname{op}} |\bar{X}^c_{k} - X^{c+}_{k}|^2$ 
and $\| \nabla^2 g \|_{\operatorname{op}} \leq C_{g^{(2)}}$ for all $x \in \mathbb{R}^d$, where $\| \nabla^2 g \|_{\operatorname{op}}$ is the operator norm of the Hessian matrix, we see that
\begin{equation}\label{e:aux1}
\begin{split}
&\mathbb{E}\left| g(X^f_{k}) - \frac{1}{2}\left( g(X^{c-}_{k}) + g(X^{c+}_{k}) \right)\right|^2 \leq 2 \mathbb{E}\left| g(X^f_{k}) - g (\bar{X}^c_{k}) \right|^2 \\ &+ \frac{1}{2} C_{g^{(2)}} \left( \mathbb{E}\left| \bar{X}^c_{k} - X^{c-}_{k} \right|^4 + \mathbb{E}\left| \bar{X}^c_{k} - X^{c+}_{k} \right|^4 \right) = 2 \mathbb{E}\left| g(X^f_{k}) - g (\bar{X}^c_{k}) \right|^2 + \frac{1}{16} C_{g^{(2)}} \mathbb{E}\left| X^{c+}_{k} - X^{c-}_{k} \right|^4\,.
\end{split}
\end{equation}
Note that we purposefully introduced the term $\mathbb{E}\left| X^{c+}_{k} - X^{c-}_{k} \right|^4$, since it will provide us with an improved rate in $h$. Indeed, we have the following result.

\begin{lemma}\label{lem:term2}
	Let Assumptions \ref{as diss}, \ref{as ina}, \ref{as fourth moment} and \ref{as linear fourth} hold.
	If $X^{c+}_0 = X^{c-}_0$, then for all $k \geq 1$ and for all $h \in (0,h_0)$ we have
	\begin{equation}\label{e:term2}
	\mathbb{E} |X^{c+}_{k} - X^{c-}_{k}|^4 \leq \frac{\bar{C}_1}{\bar{c}_1} h^2 \,,
	\end{equation}
	where
	$\bar{C}_1 := 72 \frac{1}{\varepsilon} h_0^{2\alpha} \left( 1 + 2 C_{IEul}^{(2h)} + C_{IEul}^{(4),(2h)} \right) \sigma^4 + 32 \left( 432\sqrt{2} + 648 + 27 h_0 \right) \sigma^{(4)}(1 + C_{IEul}^{(4),(2h)})$
	and $\bar{c}_1$, $h_0$, $\varepsilon > 0$ are chosen such that
	\begin{equation*}
	- 8K + 72 \varepsilon + 72 h_0 L^2 + 432 h_0^3 L^4 +32 h_0^2 \left( \left( (1 + 432L^4)216L^4  \right)^{1/2} + \frac{1 + 1296L^4}{4} \right)  \leq - \bar{c}_1 \,.
	\end{equation*}
\end{lemma}

Note that the constant $\bar{C}_1$ above is of order $\mathcal{O}(s^{-2})$ due to Assumptions \ref{as ina} and \ref{as fourth moment}. Hence the bound in (\ref{e:term2}) is in fact of order $\mathcal{O}(h^2s^{-2})$, which is exactly what is needed in Lemma \ref{thm:crucialLemma}.

We remark that, in principle, it would be now possible to bound also the first term on the right hand side of (\ref{e:aux1}) and hence to obtain a bound on the variance of the antithetic MLMC estimator with respect to discretisation, corresponding to taking $r=1$ and ${\bm \ell} = \ell_2$ in (\ref{eq AMLMC}). However, such an estimator does not perform on par with the MASGA estimator (even though it is better than the standard MLMC) and hence we skip its analysis. In this subsection, we present only the derivation of the inequality (\ref{e:aux1}) and we formulate the lemma about the bound on the term $\mathbb{E} |X^{c+}_{k} - X^{c-}_{k}|^4$, since they will be needed in our analysis of MASGA.

We remark that the proof of Lemma \ref{lem:term2} is essentially identical to the proof of Lemma \ref{thm:AntitheticSubsamplingAux} below (with different constants) and is therefore skipped. The latter proof can be found in the Appendix.

\subsection{AMLMC via subsampling}\label{sectionAntitheticViaSubsampling}

As the second building block of our MASGA estimator, we discuss a multi-level algorithm for subsampling that involves taking different drift estimators (different numbers of samples) at different levels, but with constant discretisation parameter across the levels.

We fix $s_0 \in \mathbb{N}_+$ and for $\ell \in \{ 0, \ldots, L \}$ we define $s^{\ell} := 2^{\ell} s_0$ and we consider
chains
\begin{equation}\label{subsamplingFineCoarseChains}
X^{f}_{k+1} = X^{f}_{k} + h b(X^{f}_{k},U^{f}_{k}) + \beta \sqrt{h} Z_{k+1} \,, \, \, X^{c}_{k+1} = X^{c}_{k} + h b(X^{c}_{k},U^{c}_{k}) + \beta \sqrt{h} Z_{k+1} \,,
\end{equation}
where $U^{f}_{k}$ is from a higher level (in parameter $s$) than $U^c_k$, i.e., we have $b(x,U^{f}_{k}) = \frac{1}{2}b(x,U^{f,1}_{k}) + \frac{1}{2}b(x,U^{f,2}_{k})$ and $U^{f,1}_{k}$, $U^{f,1}_{k}$ are both from the same level (in parameter $s$) as $U^{c}_{k}$, cf.\ Assumption \ref{as splittingEstimator}. In the special case of subsampling, we have
\begin{equation*}
b(x,U^f_k) := \frac{1}{2s}\sum_{i=1}^{2s} \hat{b}(x,\theta_{(U^f_k)_i}) \qquad \text{ and } \qquad b(x,U^c_k) := \frac{1}{s}\sum_{i=1}^{s} \hat{b}(x,\theta_{(U^c_k)_i}) \,,
\end{equation*}
whereas $b(x,U^{f,1}_k) := \frac{1}{s}\sum_{i=1}^{s} \hat{b}(x,\theta_{(U^f_k)_i})$ and $b(x,U^{f,2}_k) := \frac{1}{s}\sum_{i=s}^{2s} \hat{b}(x,\theta_{(U^f_k)_i})$, for some kernels $\hat{b}: \mathbb{R}^d \times \mathbb{R}^k \to \mathbb{R}^d$, where $(U^f_k)_i$, $(U^c_k)_i \sim \operatorname{Unif}(\{ 1, \ldots , m \})$. In order to introduce the antithetic counterpart of (\ref{subsamplingFineCoarseChains}), we will replace the random variable $U^c_k$ taken on the coarse level with the two components of $U^f_k$. Namely, let us denote
\begin{equation}\label{AMLMCforSubsamplingDefinitions}
X^{c-}_{k+1} = X^{c-}_{k} + h b(X^{c-}_k, U^{f,1}_k) + \beta \sqrt{h} Z_{k+1} \,, \, \, X^{c+}_{k+1} = X^{c+}_{k} + h b(X^{c+}_k, U^{f,2}_k) + \beta \sqrt{h} Z_{k+1} \,.
\end{equation}
We also put $\bar{X}^c_{k} := \frac{1}{2}\left( X^{c-}_{k} + X^{c+}_{k} \right)$.

Following our calculations from the beginning of Section \ref{sectionAntitheticViaDiscretisation} we see that for any Lipschitz function $g \in C^2_b(\mathbb{R}^d; \mathbb{R})$,
\begin{equation}\label{eq:AntitheticSubsampling1}
\begin{split}
\mathbb{E}&\left| g(X^f_{k+1}) - \frac{1}{2}\left( g(X^{c-}_{k+1}) + g(X^{c+}_{k+1}) \right)\right|^2 \\
&\leq 2 \mathbb{E}\left| g(X^f_{k+1}) - g (\bar{X}^c_{k+1}) \right|^2 + \frac{1}{16} C^2_{g^{(2)}} \mathbb{E}\left| X^{c+}_{k+1} - X^{c-}_{k+1} \right|^4
\end{split}
\end{equation}
and we need a similar bound as in Lemma \ref{lem:term2}.

\begin{lemma}\label{thm:AntitheticSubsamplingAux}
	Let Assumptions \ref{as diss}, \ref{as ina}, \ref{as fourth moment} and \ref{as linear fourth} hold.
	For $X^{c+}_{k}$ and $X^{c-}_{k}$ defined in (\ref{AMLMCforSubsamplingDefinitions}), if $X^{c+}_{0} = X^{c-}_{0}$, then for any $k \geq 1$ and any $h \in (0,h_0)$ we have
	\begin{equation*}
	\mathbb{E}|X^{c+}_{k} - X^{c-}_{k}|^4 \leq \frac{C_1}{c_1}h^2 \,,
	\end{equation*}
	where
	$C_1 := 18 \frac{1}{\varepsilon} \left( 1 + 2 C^{(2h)}_{IEul} + C^{(4),(2h)}_{IEul} \right) \sigma^4 + 4 \left( 27 \sqrt{2} + 54h_0 \right) \sigma^{(4)} (1+ C^{(4),(2h)}_{IEul})$ 
	and $c_1$, $\varepsilon$ and $h_0$ are chosen so that
	\begin{equation*}
	- 4h_0K + 18h_0 \varepsilon + 18 h_0^2 L^2 + 27h_0^4 L^4 + 4h_0^3 \left( \left( \frac{1 + 27L^4}{2} 27 L^4 \right)^{1/2} + \frac{1 + 81 L^4 }{4} \right) \leq - c_1 h_0 \,.
	\end{equation*}
\end{lemma}

Using the Lipschitz property of $g$, we see that in order to deal with the first term on the right hand side of (\ref{eq:AntitheticSubsampling1}), we need to bound $\mathbb{E}|X^{f}_{k} - \bar{X}^{c}_{k}|^2$. Indeed, we have the following result.

\begin{lemma}\label{thm:AntitheticSubsampling}
	Let Assumptions \ref{as diss}, \ref{as ina}, \ref{as fourth moment}, \ref{as smoothnessEstimator} and \ref{as linear fourth} hold.
	For $X^f_{k}$ and $\bar{X}^{c}_{k}$ defined in (\ref{AMLMCforSubsamplingDefinitions}), if $X^{c+}_{0} = X^{c-}_{0}$, then for any $k \geq 1$ and any $h \in (0,h_0)$ we have
\begin{equation*}
\mathbb{E}|X^{f}_{k} - \bar{X}^{c}_{k}|^2 \leq \frac{C_2}{c_2}h^2 \,,
\end{equation*}
where
$C_2 := \frac{3}{4} \sigma^{(4)} (1 + C^{(4)}_{IEul}) + \frac{C_1}{c_1}\left( \frac{1}{4} d C_{a^{(2)}}  \frac{1}{\varepsilon_1} + \frac{3}{4} + \frac{3}{8} h_0 C_{b^{(2)}}^2 \right)$
with $C_1$ and $c_1$ given in Lemma \ref{thm:AntitheticSubsamplingAux}, whereas $c_2$, $\varepsilon_1$ and $h_0$ are chosen so that
$- h_0 K + \frac{1}{4} d C_{a^{(2)}} \varepsilon_1 h_0 + \frac{3}{2} h_0^2 \bar{L}^2 \leq - c_2 h_0$.
\end{lemma}

Note that both $C_1$ and $C_2$ in Lemma \ref{thm:AntitheticSubsamplingAux} and Lemma \ref{thm:AntitheticSubsampling} are of order $\mathcal{O}(s^{-2})$, which follows from the dependence of both these constants on the parameters $\sigma^{(4)}$ and $\sigma^2$ and Assumptions \ref{as ina} and \ref{as fourth moment}. The proofs of both these Lemmas can be found in Appendix \ref{section:ProofsForAMLMCviaSubsampling}.

\subsection{Proof of Lemma \ref{thm:crucialLemma}}

Similarly as in Sections \ref{sectionAntitheticViaDiscretisation} and \ref{sectionAntitheticViaSubsampling}, we will analyse our estimator step-by-step. To this end, we first need to define nine auxiliary Markov chains. In what follows, we will combine the ideas for antithetic estimators with respect to the discretisation parameter and with respect to the subsampling parameter. We will therefore need to consider fine and coarse chains with respect to both parameters. We use the notational convention $X^{\operatorname{subsampling},\operatorname{discretisation}}$, hence e.g.\ $X^{f,c}$ would be a chain that behaves as a fine chain with respect to the subsampling parameter and as a coarse chain with respect to the discretisation parameter. We define three chains that move as fine chains with respect to the discretisation parameter
\begin{equation*}
\begin{split}
\begin{cases}
\Delta X^{f,f}_{k+2} &= hb(X^{f,f}_{k+1}, U^f_{k+1}) + \beta \sqrt{h} Z_{k+2} \\
\Delta X^{f,f}_{k+1} &= hb(X^{f,f}_{k}, U^f_{k}) + \beta \sqrt{h} Z_{k+1} 
\end{cases} \, \,
\begin{cases}
\Delta X^{c-,f}_{k+2} &= hb(X^{c-,f}_{k+1}, U^{f,1}_{k+1}) + \beta \sqrt{h} Z_{k+2} \\
\Delta X^{c-,f}_{k+1} &= hb(X^{c-,f}_{k}, U^{f,1}_{k}) + \beta \sqrt{h} Z_{k+1}
\end{cases}\\
\begin{cases}
\Delta X^{c+,f}_{k+2} &= hb(X^{c+,f}_{k+1}, U^{f,2}_{k+1}) + \beta \sqrt{h} Z_{k+2} \\
\Delta X^{c+,f}_{k+1} &= hb(X^{c+,f}_{k}, U^{f,2}_{k}) + \beta \sqrt{h} Z_{k+1}
\end{cases}
\end{split}
\end{equation*}
and six chains that move as coarse chains
\begin{equation*}
\begin{split}
\Delta X^{f,c-}_{k+2} &= 2h b(X^{f,c-}_{k}, U^f_{k}) + \beta \sqrt{2h} \hat{Z}_{k+2} \,, \, \, \Delta X^{f,c+}_{k+2} = 2h b(X^{f,c+}_{k}, U^f_{k+1}) + \beta \sqrt{2h} \hat{Z}_{k+2} \\
\Delta X^{c-,c-}_{k+2} &= 2h b(X^{c-,c-}_{k}, U^{f,1}_{k}) + \beta \sqrt{2h} \hat{Z}_{k+2} \,, \, \, \Delta X^{c-,c+}_{k+2} = 2h b(X^{c-,c+}_{k}, U^{f,1}_{k+1}) + \beta \sqrt{2h} \hat{Z}_{k+2} \\
\Delta X^{c+,c-}_{k+2} &= 2h b(X^{c+,c-}_{k}, U^{f,2}_{k}) + \beta \sqrt{2h} \hat{Z}_{k+2} \,, \, \, \Delta X^{c+,c+}_{k+2} = 2h b(X^{c+,c+}_{k}, U^{f,2}_{k+1}) + \beta \sqrt{2h} \hat{Z}_{k+2} \,.
\end{split}
\end{equation*}
Here $\Delta X^{f,f}_{k+2} = X^{f,f}_{k+2} - X^{f,f}_{k+1}$, $\Delta X^{f,f}_{k+1} = X^{f,f}_{k+1} - X^{f,f}_{k}$, $\Delta X^{f,c-}_{k+2} = X^{f,c-}_{k+2} - X^{f,c-}_{k}$ and likewise for other chains, whereas $\sqrt{2h} \hat{Z}_{k+2} = \sqrt{h} Z_{k+1} + \sqrt{h} Z_{k+2}$.
In order to prove Lemma \ref{thm:crucialLemma}, we will first show that for any $k \geq 1$ we have
$\mathbb{E}|\Delta \operatorname{Ant} \Phi^{s, h}_{g,k}|^2 \leq \widetilde{C} \mathbb{E}|\Delta \operatorname{Ant} \Phi^{s, h}_{k}|^2 + \bar{C}h^2$,
where $\widetilde{C}$, $\bar{C}$ are positive constants and $\bar{C}$ is of order $\mathcal{O}(s^{-2})$. Here $\Delta \operatorname{Ant} \Phi^{s, h}_{k}$ corresponds to taking $\Delta \operatorname{Ant} \Phi^{s, h}_{g,k}$ with $g(x) = x$ the identity function.

Then we will show that for all $k \geq 1$ we have $\mathbb{E}|\Delta \operatorname{Ant} \Phi^{s, h}_{k+2}|^2 \leq (1-ch)\mathbb{E}|\Delta \operatorname{Ant} \Phi^{s, h}_{k}|^2 + Ch^3$ for some constants $c$, $C > 0$, where $C$ is of order $\mathcal{O}(s^{-2})$. Finally, we will conclude that for all $k \geq 1$ we have $\mathbb{E}|\Delta \operatorname{Ant} \Phi^{s, h}_{k}|^2 \leq C_1 h^2/s^2$ for some $C_1 > 0$, which will finish the proof.

In order to simplify the notation, we define here one step of the MASGA estimator as
\begin{equation*}
\begin{split}
\Psi_{k} &:= \left[ X^{f,f}_{k} - \frac{1}{2}\left( X^{f,c-}_{k} + X^{f,c+}_{k} \right)   \right] \\
&- \frac{1}{2} \left[ \left( X^{c-,f}_{k} - \frac{1}{2} \left( X^{c-,c-}_{k} + X^{c-,c+}_{k} \right) \right) + \left( X^{c+,f}_{k} - \frac{1}{2}\left( X^{c+,c-}_{k} + X^{c+,c+}_{k} \right) \right) \right] \,.
\end{split}
\end{equation*}

Let us first focus on the analysis of 
\begin{equation*}
\begin{split}
\mathbb{E}|\Psi^g_{k}|^2 &:= \mathbb{E} \Bigg| \left[ g(X^{f,f}_{k}) - \frac{1}{2}\left( g(X^{f,c-}_{k}) + g(X^{f,c+}_{k}) \right)   \right] \\
&- \frac{1}{2} \Big[ \left( g(X^{c-,f}_{k}) - \frac{1}{2} \left( g(X^{c-,c-}_{k}) + g(X^{c-,c+}_{k}) \right) \right) \\
&+ \left( g(X^{c+,f}_{k}) - \frac{1}{2}\left( g(X^{c+,c-}_{k}) + g(X^{c+,c+}_{k}) \right) \right) \Big] \Bigg|^2 
\end{split}
\end{equation*}
for $g : \mathbb{R}^d \to \mathbb{R}$ Lipschitz.
To this end, we will introduce three additional chains
\begin{equation*}
\bar{X}^{f,c}_{k} = \frac{1}{2}\left( X^{f,c-}_{k} + X^{f,c+}_{k} \right), \hspace{0.25em} \bar{X}^{c-,c}_{k} = \frac{1}{2}\left( X^{c-,c-}_{k} + X^{c-,c+}_{k} \right) \hspace{0.25em}
\text{and} \hspace{0.25em} \bar{X}^{c+,c}_{k} = \frac{1}{2}\left( X^{c+,c-}_{k} + X^{c+,c+}_{k} \right).
\end{equation*}

Observe that in the expression $\mathbb{E}|\Psi^g_{kh}|^2$ above we have three rows of the same structure as the antithetic estimator via discretisation, hence we can proceed exactly as in Section \ref{sectionAntitheticViaDiscretisation} by adding and subtracting $g\left( \bar{X}^{f,c}_{k} \right)$, $g\left( \bar{X}^{c-,c}_{k} \right)$ and $g\left( \bar{X}^{c+,c}_{k} \right)$, respectively in each row, and then applying Taylor's formula in points $\bar{X}^{f,c}_{k}$, $\bar{X}^{c-,c}_{k}$ and $\bar{X}^{c+,c}_{k}$, respectively in each row (note that the first order terms will cancel out) to get
\begin{equation*}
\begin{split}
\mathbb{E}|\Psi^g_{k}|^2 &= \mathbb{E} \Bigg| \left[ g(X^{f,f}_{k}) - g\left( \bar{X}^{f,c}_{k} \right) + T\left( X^{f,c-}_{k} - X^{f,c+}_{k} \right) \right] \\
&- \frac{1}{2} \Big[ \left( g(X^{c-,f}_{k}) - g\left( \bar{X}^{c-,c}_{k} \right) + T\left( X^{c-,c-}_{k} - X^{c-,c+}_{k} \right) \right) \\
&+ \left( g(X^{c+,f}_{k}) - g\left( \bar{X}^{c+,c}_{k} \right) + T\left( X^{c+,c-}_{k} - X^{c+,c+}_{k} \right) \right) \Big] \Bigg|^2 \,, 
\end{split}
\end{equation*}
where
\begin{equation*}
\begin{split}
T(X^{f,c-}_{k} - X^{f,c+}_{k}) &:= - \frac{1}{4}\sum_{|\alpha|=2} \int_0^1 (1-t) D^{\alpha} g\left( \bar{X}^{f,c}_{k} + t \left( X^{f,c-}_{k} - \bar{X}^{f,c}_{k} \right) \right) dt \left( X^{f,c-}_{k} - X^{f,c+}_{k} \right)^{\alpha} \\
&- \frac{1}{4}\sum_{|\alpha|=2} \int_0^1 (1-t) D^{\alpha} g\left( \bar{X}^{f,c}_{k} + t \left( X^{f,c+}_{k} - \bar{X}^{f,c}_{k} \right) \right) dt \left( X^{f,c+}_{k} - X^{f,c-}_{k} \right)^{\alpha} \,,
\end{split}
\end{equation*}
where we used
$X^{f,c+}_{k} - \bar{X}^{f,c}_{k} = 
\frac{1}{2} \left( X^{f,c+}_{k} - X^{f,c-}_{k} \right) = - \left( X^{f,c-}_{k} - \bar{X}^{f,c}_{k} \right)$
and $T$ is defined analogously for other terms, 
and hence $|T(x)|^2 \leq C|x|^4$ for some $C > 0$ and for any $x \in \mathbb{R}^d$. Using $(a+b)^2 \leq 2a^2 + 2b^2$ for $\mathbb{E}|\Psi^g_{kh}|^2$ we can separate the terms involving $T(\cdot)$ and, due to Lemma \ref{lem:term2}, we see that we have the correct order in $h$ and $s$ for all the terms expressed as $T(\cdot)$.
Hence the remaining term whose dependence on $s$ and $h$ we still need to check can be expressed as
\begin{equation*}
\mathbb{E} \left| \left[ g(X^{f,f}_{k}) - \frac{1}{2}\left(g(X^{c-,f}_{k}) + g(X^{c+,f}_{k}) \right) \right] - \left[ g\left( \bar{X}^{f,c}_{k} \right) - \frac{1}{2}\left(g\left( \bar{X}^{c-,c}_{k} \right) +  g\left( \bar{X}^{c+,c}_{k} \right) \right) \right] \right|^2 \,.
\end{equation*}

We will now need yet another auxiliary chain
$\bar{X}^{c,f}_{k} := \frac{1}{2}\left( X^{c-,f}_{k} + X^{c-,f}_{k} \right)$.
We repeat our argument from the previous step by  adding and subtracting $g\left( \bar{X}^{c,f}_{k} \right)$ and $g\left( \frac{1}{2}\left(\bar{X}^{c-,c}_{k} + \bar{X}^{c+,c}_{k} \right) \right)$, respectively, to the first and the second term in square brackets above, respectively, and applying Taylor's formula in points $\bar{X}^{c,f}_{k}$ and $\frac{1}{2}\left(\bar{X}^{c-,c}_{k} + \bar{X}^{c+,c}_{k} \right)$, respectively (the first order terms again cancel out), to obtain
\begin{equation*}
\begin{split}
\mathbb{E} &\Bigg| \left[ g(X^{f,f}_{k}) - g\left( \bar{X}^{c,f}_{k} \right) + T\left( X^{c-,f}_{k} - X^{c+,f}_{k} \right) \right] \\
&- \left[ g\left( \bar{X}^{f,c}_{k} \right) - g\left( \frac{1}{2}\left(\bar{X}^{c-,c}_{k} + \bar{X}^{c+,c}_{k} \right) \right) + T\left( \bar{X}^{c-,c}_{k} - \bar{X}^{c+,c}_{k} \right)  \right] \Bigg|^2  \,,
\end{split}
\end{equation*}
where again $|T(x)|^2 \leq C|x|^4$ for some $C > 0$ and for any $x \in \mathbb{R}^d$. Due to Lemma \ref{thm:AntitheticSubsamplingAux} we see that $\mathbb{E} \left| X^{c-,f}_{k} - X^{c+,f}_{k} \right|^4$ has the correct order in $s$ and $h$. Moreover, we have
$\bar{X}^{c-,c}_{k} - \bar{X}^{c+,c}_{k} = \frac{1}{2}\left( X^{c-,c-}_{k} - X^{c+,c-}_{k} \right) + \frac{1}{2}\left( X^{c-,c+}_{k} - X^{c+,c+}_{k} \right)$
and hence, after using $(a+b)^4 \leq 8a^4 + 8b^4$, Lemma \ref{thm:AntitheticSubsamplingAux} applies also to $\mathbb{E} \left| \bar{X}^{c-,c}_{k} - \bar{X}^{c+,c}_{k} \right|^4$.
Hence we only need to deal with
\begin{equation*}
I:= \mathbb{E} \left| g(X^{f,f}_{k}) - g\left( \bar{X}^{c,f}_{k} \right) - g\left( \bar{X}^{f,c}_{k} \right) + g\left( \frac{1}{2}\left(\bar{X}^{c-,c}_{k} + \bar{X}^{c+,c}_{k} \right) \right)  \right|^2 \,.
\end{equation*}
We can now add and subtract
$g\left(\bar{X}^{f,c}_{k} + \bar{X}^{c,f}_{k} - \frac{1}{2}\left(\bar{X}^{c-,c}_{k} + \bar{X}^{c+,c}_{k} \right) \right)$,
and, using the Lipschitz property of $g$ (which we assume it satisfies with a Lipschitz constant, say, $L_g > 0$), we see that
\begin{equation*}
I \leq 3 L_g \mathbb{E} \left| X^{f,f}_{k} - \bar{X}^{f,c}_{k} - \bar{X}^{c,f}_{k} + \frac{1}{2}\left(\bar{X}^{c-,c}_{k} + \bar{X}^{c+,c}_{k} \right) \right|^2 + 6 L_g \mathbb{E} \left| \bar{X}^{f,c}_{k} - \frac{1}{2}\left(\bar{X}^{c-,c}_{k} + \bar{X}^{c+,c}_{k} \right) \right|^2 \,.
\end{equation*}
However, observe that $\mathbb{E} \left| X^{f,f}_{k} - \bar{X}^{f,c}_{k} - \bar{X}^{c,f}_{k} + \frac{1}{2}\left(\bar{X}^{c-,c}_{k} + \bar{X}^{c+,c}_{k} \right) \right|^2 = \mathbb{E}|\Psi_{k}|^2$ and, moreover,
\begin{equation*}
\begin{split}
&\bar{X}^{f,c}_{k} - \frac{1}{2}\left( \bar{X}^{c-,c}_{k} + \bar{X}^{c+,c}_{k} \right) = \frac{1}{2} \left( X^{f,c-}_{k} + X^{f,c+}_{k} \right) - \frac{1}{4} \left( X^{c-,c-}_{k} + X^{c-,c+}_{k} + X^{c+,c-}_{k} + X^{c+,c+}_{k} \right) \\
&= \frac{1}{2} \left(X^{f,c-}_{k} - \frac{1}{2} \left( X^{c-,c-}_{k} + X^{c+,c-}_{k} \right)  \right) + \frac{1}{2} \left(X^{f,c+}_{k} - \frac{1}{2} \left( X^{c-,c+}_{k} + X^{c+,c+}_{k} \right)  \right) \,.
\end{split}
\end{equation*}
Note that both terms on the right hand side above correspond to antithetic estimators via subsampling, hence from Lemma \ref{thm:AntitheticSubsampling} we infer that $\mathbb{E} \left| \bar{X}^{f,c}_{k} - \frac{1}{2}\left(\bar{X}^{c-,c}_{k} + \bar{X}^{c+,c}_{k} \right) \right|^2$ has the correct order in $s$ and $h$. 
We have thus demonstrated that for any $k \geq 1$ we have $\mathbb{E}|\Psi^g_{k}|^2 \leq C_1\mathbb{E}|\Psi_{k}|^2 + C_2 h^2/s^2$ for some constants $C_1$, $C_2 > 0$. Therefore, in order to finish the proof, it remains to be shown that $\mathbb{E}|\Psi_{k}|^2$ has the correct order in $h$ and $s$. As explained above, this will be achieved by proving that there exist constants $c$, $C > 0$ (with $C$ being of order $\mathcal{O}(s^{-2})$) such that for any $k \geq 1$ we have $\mathbb{E}|\Psi_{k+2}|^2 \leq (1-ch)\mathbb{E}|\Psi_{k}|^2 + Ch^3$.
The idea for dealing with $\mathbb{E}|\Psi_{k+2}|^2$ is to group the terms in a specific way, add and subtract certain drifts in order to set up appropriate combinations of drifts for a Taylor's formula application and then to cancel out some first order terms. As we have already seen, the remaining second order terms should then be of the correct order in $h$ and $s$. To this end, we denote
\begin{equation*}
\begin{split}
\Xi^1_{k} &:= hb(X^{f,f}_{k}, U^f_{k}) - hb(X^{f,c-}_{k}, U^f_{k}) - hb(X^{f,c+}_{k}, U^f_{k+1}) \\
\Xi^2_{k} &:= hb(X^{c-,f}_{k}, U^{f,1}_{k}) - hb(X^{c-,c-}_{k}, U^{f,1}_{k}) - hb(X^{c-,c+}_{k}, U^{f,1}_{k+1}) \\
\Xi^3_{k} &:= hb(X^{c+,f}_{k}, U^{f,2}_{k}) - hb(X^{c+,c-}_{k}, U^{f,2}_{k}) - hb(X^{c+,c+}_{k}, U^{f,2}_{k+1}) \\
\Xi_{k} &:= \Xi^1_{k} - \frac{1}{2}\left( \Xi^2_{k} + \Xi^3_{k} \right) \text{ and } \Upsilon_{k} := hb(X^{f,f}_{k+1}, U^f_{k+1}) - \frac{1}{2}hb(X^{c-,f}_{k+1}, U^{f,1}_{k+1}) - \frac{1}{2}hb(X^{c+,f}_{k+1}, U^{f,2}_{k+1})
\end{split}
\end{equation*}
and hence, observing that all the noise variables cancel out, we can write
$\Psi_{k+2} = \Psi_{k} + \Xi_k + \Upsilon_{k}$.
Thus we have
$\mathbb{E}\left| \Psi_{k+2} \right|^2 = \mathbb{E} \left| \Psi_{k} \right|^2 + 2 \mathbb{E} \langle \Psi_{k} , \Xi_{k} \rangle + 2 \mathbb{E} \langle \Psi_{k} , \Upsilon_{k} \rangle + \mathbb{E} \left| \Xi_{k} + \Upsilon_{k} \right|^2$.
We will bound $\mathbb{E}|\Xi_{k} + \Upsilon_{k}|^2 \leq 2\mathbb{E}|\Xi_{k}|^2 + 2\mathbb{E}|\Upsilon_{k}|^2$ and we will first deal with the terms involving $\Upsilon_{k}$. We will need an additional auxiliary Markov chain (moving as a fine chain with respect to the discretisation parameter) defined as 
\begin{equation*}
X_{k+2} = X_{k+1} + ha(X_{k+1}) + \beta \sqrt{h} Z_{k+2} \,, \qquad
X_{k+1} = X_{kh} + ha(X_{k}) + \beta \sqrt{h} Z_{k+1} \,.
\end{equation*}
Using $b(x,U) = \frac{1}{2}b(x,U^1) + \frac{1}{2}b(x,U^2)$, we have
\begin{equation*}
\begin{split}
\Upsilon_{k} &= hb(X^{f,f}_{k+1}, U^f_{k+1}) - hb(X_{k+1}, U^f_{k+1}) - \frac{1}{2}h\left( b(X^{c-,f}_{k+1}, U^{f,1}_{k+1}) - b(X_{k+1}, U^{f,1}_{k+1})\right) \\
&- \frac{1}{2}h\left( b(X^{c+,f}_{k+1}, U^{f,2}_{k+1}) - b(X_{k+1}, U^{f,2}_{k+1})\right) \,.
\end{split}
\end{equation*}
First, in order to deal with $\mathbb{E} \langle \Psi_{k} , \Upsilon_{k} \rangle$, we use Taylor's formula to write
\begin{equation*}
\begin{split}
\Upsilon_{k} &= h \Bigg( \sum_{|\alpha|=1} D^{\alpha} b (X_{k+1}, U^f_{k+1}) \left(X^{f,f}_{k+1} - X_{k+1}\right)^{\alpha} \\
&+ \sum_{|\alpha|=2} \int_0^1 (1-t) D^{\alpha}b\left(X_{k+1} + t\left( X^{f,f}_{k+1} - X_{k+1} \right), U^f_{k+1}\right) dt \left(X^{f,f}_{k+1} - X_{k+1}\right)^{\alpha} \Bigg) \\
&-\frac{1}{2}h \Bigg( \sum_{|\alpha|=1} D^{\alpha} b (X_{k+1}, U^{f,1}_{k+1}) \left(X^{c-,f}_{k+1} - X_{k+1}\right)^{\alpha} \\
&+ \sum_{|\alpha|=2} \int_0^1 (1-t) D^{\alpha}b\left(X_{k+1} + t\left( X^{c-,f}_{k+1} - X_{k+1} \right), U^{f,1}_{k+1}\right) dt \left(X^{c-,f}_{k+1} - X_{k+1}\right)^{\alpha} \Bigg) \\
&-\frac{1}{2}h \Bigg( \sum_{|\alpha|=1} D^{\alpha} b (X_{k+1}, U^{f,2}_{k+1}) \left(X^{c+,f}_{k+1} - X_{k+1}\right)^{\alpha} \\
&+ \sum_{|\alpha|=2} \int_0^1 (1-t) D^{\alpha}b\left(X_{k+1} + t\left( X^{c+,f}_{k+1} - X_{k+1} \right), U^{f,2}_{k+1}\right) dt \left(X^{c+,f}_{k+1} - X_{k+1}\right)^{\alpha} \Bigg) \,.
\end{split}
\end{equation*}
Hence, using Assumptions \ref{as smoothnessEstimator} and \ref{as additionalAssumption}, we have
\begin{equation*}
\begin{split}
\mathbb{E} &\langle \Psi_{k} , \Upsilon_{k} \rangle = h \mathbb{E} \langle \Psi_{k} , \sum_{|\alpha|=1} D^{\alpha} a (X_{k+1}) \left(X^{f,f}_{k+1} - \frac{1}{2}X^{c-,f}_{k+1} -\frac{1}{2}X^{c+,f}_{k+1} \right)^{\alpha} \rangle \\
&+ h \langle \Psi_{k} , \sum_{|\alpha|=2} \int_0^1 (1-t) D^{\alpha}a\left(X_{k+1} + t\left( X^{f,f}_{k+1} - X_{k+1} \right)\right) dt \left(X^{f,f}_{k+1} - X_{k+1}\right)^{\alpha} \rangle \\
&+ h \langle \Psi_{k} , \sum_{|\alpha|=2} \int_0^1 (1-t) D^{\alpha}a\left(X_{k+1} + t\left( X^{c-,f}_{k+1} - X_{k+1} \right)\right) dt \left(X^{c-,f}_{k+1} - X_{k+1}\right)^{\alpha} \rangle \\
&+ h \langle \Psi_{k} , \sum_{|\alpha|=2} \int_0^1 (1-t) D^{\alpha}a\left(X_{k+1} + t\left( X^{c+,f}_{k+1} - X_{k+1} \right)\right) dt \left(X^{c+,f}_{k+1} - X_{k+1}\right)^{\alpha} \rangle \,.
\end{split}
\end{equation*}
Now we can use Young's inequality for each term above with some $\varepsilon_1$, $\varepsilon_2$, $\varepsilon_3$, $\varepsilon_4 > 0$ respectively and, using Assumption \ref{as diss} (condition (\ref{driftSmoothness})), we get
\begin{equation*}
\begin{split}
\mathbb{E} &\langle \Psi_{k} , \Upsilon_{k} \rangle \leq \frac{1}{2} h \varepsilon_1 C_{a^{(1)}} \mathbb{E} \left| \Psi_{k} \right|^2 + \frac{1}{2} h \frac{1}{\varepsilon_1} C_{a^{(1)}} \mathbb{E} \left| X^{f,f}_{k+1} - \frac{1}{2}X^{c-,f}_{k+1} -\frac{1}{2}X^{c+,f}_{k+1} \right|^2 \\
&+ \frac{1}{2} h \varepsilon_2 C_{a^{(2)}} \mathbb{E} \left| \Psi_{k} \right|^2 + \frac{1}{2} h \frac{1}{\varepsilon_2} C_{a^{(2)}} \mathbb{E}  \left| X^{f,f}_{k+1} - X_{k+1}\right|^4 \\
&+ \frac{1}{2} h \varepsilon_3 C_{a^{(2)}} \mathbb{E} \left| \Psi_{k} \right|^2 + \frac{1}{2} h \frac{1}{\varepsilon_3} C_{a^{(2)}} \mathbb{E}  \left| X^{c-,f}_{k+1} - X_{k+1}\right|^4 \\
&+ \frac{1}{2} h \varepsilon_4 C_{a^{(2)}} \mathbb{E} \left| \Psi_{k} \right|^2 + \frac{1}{2} h \frac{1}{\varepsilon_4} C_{a^{(2)}} \mathbb{E}  \left| X^{c+,f}_{k+1} - X_{k+1}\right|^4 \,.
\end{split}
\end{equation*}
Now note that the second term on the right hand side above is of order $\mathcal{O}(h^2/s^2)$ due to Lemma \ref{thm:AntitheticSubsampling}. For the other three terms, we need the following auxiliary result.

\begin{lemma}\label{thm:AntitheticMasterChain}
	Let Assumptions \ref{as diss}, \ref{as ina}, \ref{as fourth moment} and \ref{as linear fourth} hold. Assuming $X^{f,f}_0 = X_0$, there exists a constant $C > 0$ such that for all $k \geq 1$,
	\begin{equation*}
	\mathbb{E} \left| X^{f,f}_{k} - X_{k} \right|^4 \leq C\frac{1}{s^2}h^2 \,.
	\end{equation*}
\end{lemma}	

The proof is similar to the proof of Lemma \ref{thm:AntitheticSubsamplingAux} and can be found in Appendix \ref{sectionMIASGA}.

The reasoning in the proof of Lemma \ref{thm:AntitheticMasterChain} applies also to $X^{c-,f}_{k}$ and $X^{c+,f}_{k}$ in place of $X^{f,f}_{k}$. Hence we see that $\mathbb{E} \langle \Psi_{k} , \Upsilon_{k} \rangle$ is bounded from above by an expression of the form $C_1(\varepsilon_1 + \varepsilon_2 + \varepsilon_3 + \varepsilon_4)h \mathbb{E}|\Psi_k|^2 + C_2h$, where $\varepsilon_i$ for $i \in \{ 1, \ldots 4 \}$ can be chosen as small as necessary and the constant $C_2$ is of the correct order in $s$ and $h$, i.e., of order $\mathcal{O}(h^2/s^2)$. We will explain later how to handle the other terms and how to choose the values of $\varepsilon_i$ for $i \in \{ 1, \ldots 4 \}$. Now in order to deal with $\mathbb{E}|\Upsilon_{k}|^2$ we use a different decomposition for $\Upsilon_{k}$, namely we write
\begin{equation*}
\begin{split}
\Upsilon_{k} &= \frac{1}{2}hb(X^{f,f}_{k+1}, U^{f,1}_{k+1}) + \frac{1}{2}hb(X^{f,f}_{k+1}, U^{f,2}_{k+1}) - \frac{1}{2}hb(X^{c-,f}_{k+1}, U^{f,1}_{k+1}) - \frac{1}{2}hb(X^{c+,f}_{k+1}, U^{f,2}_{k+1}) \\
&= \frac{1}{2}hb(X^{f,f}_{k+1}, U^{f,1}_{k+1}) - \frac{1}{2}hb(\bar{X}^{c,f}_{k+1}, U^{f,1}_{k+1}) + \frac{1}{2}hb(\bar{X}^{c,f}_{k+1}, U^{f,1}_{k+1}) - \frac{1}{2}hb(X^{c-,f}_{k+1}, U^{f,1}_{k+1}) \\
&+ \frac{1}{2}hb(X^{f,f}_{k+1}, U^{f,2}_{k+1}) - \frac{1}{2}hb(\bar{X}^{c,f}_{k+1}, U^{f,2}_{k+1}) + \frac{1}{2}hb(\bar{X}^{c,f}_{k+1}, U^{f,2}_{k+1}) - \frac{1}{2}hb(X^{c+,f}_{k+1}, U^{f,2}_{k+1}) \,,
\end{split}
\end{equation*}
where $\bar{X}^{c,f}_{k+1} := \frac{1}{2} \left( X^{c-,f}_{k+1} + X^{c+,f}_{k+1} \right)$. Hence, using (\ref{ina driftLipschitz}) we obtain
\begin{equation*}
\begin{split}
&\mathbb{E} \left| \Upsilon_{k} \right|^2 \leq \frac{3}{4} h^2 \bar{L}^2 \mathbb{E} \left| X^{f,f}_{k+1} - \bar{X}^{c,f}_{k+1} \right|^2  + \frac{3}{4} h^2 \bar{L}^2 \mathbb{E} \left| X^{f,f}_{k+1} - \bar{X}^{c,f}_{k+1} \right|^2 \\
&+ \frac{3}{4} h^2 \mathbb{E} \left| b(\bar{X}^{c,f}_{k+1}, U^{f,1}_{k+1}) - b(X^{c-,f}_{k+1}, U^{f,1}_{k+1}) + b(\bar{X}^{c,f}_{k+1}, U^{f,2}_{k+1}) - b(X^{c+,f}_{k+1}, U^{f,2}_{k+1}) \right|^2 \,.
\end{split}
\end{equation*}
The first two terms on the right hand side above are identical and have the correct order in $s$ and $h$ due to Lemma \ref{thm:AntitheticSubsampling}. On the other hand, the third term can be dealt with by one more application of Taylor's formula (cf.\ the calculation for the term $J_{33}$ in the proof of Lemma \ref{thm:AntitheticSubsampling} in Appendix \ref{section:ProofsForAMLMCviaSubsampling} for more details).

The terms involving $\Xi_k$ can be handled using similar ideas as above. In particular, for $\mathbb{E}|\Xi_k|^2$ we have the following result.

\begin{lemma}\label{thm:lemmaXi1}
	Let Assumptions \ref{as diss}, \ref{as ina}, \ref{as fourth moment}, \ref{as smoothnessEstimator} and \ref{as linear fourth} hold. Assuming all the auxiliary chains introduced above are initiated at the same point, there exists a constant $C_{1,\Xi} > 0$ such that for all $k \geq 1$,
	\begin{equation*}
	\mathbb{E} |\Xi_k|^2 \leq C_{1,\Xi}\frac{1}{s^2}h^2 \,.
	\end{equation*}
\end{lemma}	

The proof of Lemma \ref{thm:lemmaXi1} can be found in Appendix \ref{sectionMIASGA}. The last term to deal with is $\mathbb{E} \langle \Psi_k, \Xi_k \rangle$. We have the following result.

\begin{lemma}\label{thm:lemmaXi2}
	Let Assumptions \ref{as diss}, \ref{as Lipschitz estimator}, \ref{as ina}, \ref{as fourth moment}, \ref{as smoothnessEstimator}, \ref{as additionalAssumption}, \ref{as linear fourth} and \ref{as splittingEstimator} hold. Assuming all the auxiliary chains introduced above are initiated at the same point, there exist constants $C_{2,\Xi}$, $C_{3,\Xi}$ and $C_{4,\Xi} > 0$ such that for all $k \geq 1$,
	\begin{equation*}
	\mathbb{E} \langle \Psi_k , \Xi_k \rangle \leq (C_{2,\Xi}\varepsilon - C_{3,\Xi})h \mathbb{E}|\Psi_k|^2 + C_{4,\Xi}\frac{1}{s^2}h^2 \,,
	\end{equation*}
	where $\varepsilon > 0$ can be chosen arbitrarily small.
\end{lemma}	 

The crucial insight about the bound in Lemma \ref{thm:lemmaXi2} is that, thanks to the special structure of the term $\Xi_k = \Xi_k^1 - \frac{1}{2}(\Xi_k^2 + \Xi_k^3)$, we can extract from $\mathbb{E}\langle \Psi_k , \Xi_k \rangle$, after a few applications of Taylor's formula, a term of the form $h\mathbb{E}\langle \Psi_k , \sum_{|\alpha|=1} D^{\alpha} a(X_k) (\Psi_k)^{\alpha} \rangle$, which, due to Assumptions \ref{as diss}, can be bounded from above by $-Kh\mathbb{E}|\Psi_k|^2$. This gives us the term with the constant $C_{3,\Xi}$ in Lemma \ref{thm:lemmaXi2}. Then, after combining all our estimates and choosing all the $\varepsilon_i > 0$ small enough, we can indeed conclude that for any $k \geq 1$ we have $\mathbb{E}|\Psi_{k+2}|^2 \leq (1-ch)\mathbb{E}|\Psi_{k}|^2 + Ch^3$ with $C$ of order $\mathcal{O}(s^{-2})$, which, as explained above, finishes the proof.

The proof of Lemma \ref{thm:lemmaXi2} is lengthy and tedious but elementary and hence is moved to Appendix \ref{sectionMIASGA}.

\section{Appendix: Bounds on moments of subsampling estimators}\label{sectionSubsampling}

In this section we present bounds both for subsampling with and without replacement \cite{welling2011bayesian, Shamir2016}. We fix $s$, $m \in \mathbb{N}$ such that $s < m$. Let $\theta_i\in \bR^d$, for $i = 1, \ldots, m$. Moreover, let $U = (U_i)_{i=1\ldots,s}$ be a collection of $s$ independent random variables, uniformly distributed over the set $\{1,\ldots,m\}$.
We define 
\begin{equation}\label{subsamplingDrift}
a(x)= \frac{1}{m}\sum_{i=1}^m \hat{b}(x,\theta_i)\,\,\,\,\text{and}\,\,\,\,b(x,U)= \frac{1}{s} \sum_{i=1}^s \hat{b}(x,\theta_{U_i})\,.
\end{equation}
Here $\hat{b} : \mathbb{R}^d \times \mathbb{R}^k \to \mathbb{R}^d$ is a kernel such that for any $x$, $y \in \mathbb{R}^d$ and any $\theta \in \mathbb{R}^k$ we have
\begin{equation}\label{bhatAssumptions}
|\hat{b}(x,\theta) - \hat{b}(y,\theta)| \leq L|x-y| \qquad \text{ and } \qquad \langle x - y , \hat{b}(x,\theta) - \hat{b}(y,\theta) \rangle \leq - K|x-y|^2 \,.
\end{equation}
for some $L$, $K > 0$. Hence $b$ is an unbiased estimator of $a$ that corresponds to sampling with replacement $s$ terms from the sum of $m$ terms defining $a$, cf. Example 2.15 in \cite{MajkaMijatovicSzpruch2018}. Moreover, Assumptions \ref{as diss} and \ref{as Lipschitz estimator} are satisfied with constants $L$, $K$ and $\bar{L} = L$. We will now verify Assumptions \ref{as ina} and \ref{as fourth moment}.

Note that related calculations for second moments of subsampling estimators were carried out in \cite{MajkaMijatovicSzpruch2018} (see Example 2.15 therein) for the drift $a$ and its estimator $b$ in (\ref{subsamplingDrift}) rescaled by $m$, that is, for $a(x)= \sum_{i=1}^m \hat{b}(x,\theta_i)$ and $b(x,U_k)= \frac{m}{s} \sum_{i=1}^s \hat{b}(x,\theta_{U_i})$. Hence, obviously, all the upper bounds on second moments obtained in \cite{MajkaMijatovicSzpruch2018} still hold for $a$ and $b$ given by (\ref{subsamplingDrift}), after rescaling by $1/m^2$.

Based on the calculations in \cite{MajkaMijatovicSzpruch2018}, we know that if we assume that for all $\theta$ and $x$ we have $|\hat{b}(x, \theta)|^2 \leq C(1 + |x|^2)$ with some constant $C >0$, then
$\mathbb{E}|b(x,U) - a(x)|^2 \leq \frac{1}{s} C(1 + |x|^2)$,
which verifies Assumption \ref{as ina} for the subsampling with replacement scheme.
Let us now define a new estimator
$b^{wor}(x,U) := \frac{1}{s} \sum_{j=1}^{m} \hat{b}(x,\theta_j)Z_j$,
where $(Z_j)_{j=1}^{m}$ are correlated random variables such that
$\mathbb{P}(Z_j = 1) = \frac{s}{m}$, $\mathbb{P}(Z_j = 0) = 1 - \frac{s}{m}$ and $\mathbb{P}(Z_i = 1, Z_j = 1) = \binom{m-2}{s-2}/\binom{m}{s}$ 
for any $i$, $j \in \{ 1, \ldots , m \}$ such that $i \neq j$. Note that this definition of $b^{wor}$ corresponds to sampling $s$ terms from the sum of $m$ terms defining $a$ in (\ref{subsamplingDrift}) without replacement, see Example 2.15 in \cite{MajkaMijatovicSzpruch2018}. It is immediate to check that this estimator of $a$ is indeed unbiased. In order to bound the variance, we can first check that $\operatorname{Cov}(Z_i,Z_j) = \frac{s(s-1)}{m(m-1)} - \frac{s^2}{m^2} = - \frac{s(1-\frac{s}{m})}{m(m-1)}$. We have
\begin{equation*}
\begin{split}
\mathbb{E}&|b^{wor}(x,U) - a(x)|^2 = \mathbb{E} \left| \frac{1}{s} \sum_{j=1}^{m} \hat{b}(x, \theta_j)Z_j  - \frac{1}{s} \sum_{j=1}^{m} \frac{s}{m} \hat{b}(x,\theta_j) \right|^2 = \frac{1}{s^2} \mathbb{E} \left| \sum_{j=1}^{m} \hat{b}(x, \theta_j) (Z_j - \frac{s}{m}) \right|^2 \\
&= \frac{1}{s^2} \left( \mathbb{E} \sum_{j=1}^{m} \hat{b}(x,\theta_j)^2 (Z_j - \frac{s}{m})^2 + \mathbb{E} \sum_{i, j = 1, i \neq j}^{m} \hat{b}(x, \theta_j)(Z_j - \frac{s}{m})\hat{b}(x, \theta_i)(Z_i - \frac{s}{m}) \right) \,.
\end{split}
\end{equation*}
Note now that we have
\begin{equation*}
\begin{split}
\mathbb{E}(Z_i - \frac{s}{m})(Z_j - \frac{s}{m}) &= \mathbb{E}\left(Z_i Z_j - \frac{s}{m}Z_i - \frac{s}{m}Z_j + \frac{s^2}{m^2}\right) = \frac{\binom{m-2}{s-2}}{\binom{m}{s}} - 2 \frac{s^2}{m^2} + \frac{s^2}{m^2} \\
&= \frac{s(s-1)}{m(m-1)} - \frac{s^2}{m^2} = - \frac{s(m-s)}{m^2(m-1)} \,.
\end{split}
\end{equation*}
Hence we can easily check that
$\mathbb{E}|b^{wor}(x,U) - a(x)|^2 \leq \frac{1}{s}(1 - \frac{s}{m})C(1 + |x|^2)$.
Thus we see that the upper bound on the variance of the estimator $b^{wor}$ that we obtained is equal to the upper bound on the variance of $b$ multiplied by $(1 - \frac{s}{m})$. In particular, this confirms that Assumption \ref{as ina} holds also for the subsampling without replacement scheme.

Let us now explain how to estimate the fourth centered moments required for Assumption \ref{as fourth moment}, based on the assumption that for all $\theta$ and $x$ we have $|\hat{b}(x,\theta)|^4 \leq C(1 + |x|^4)$. We have
\begin{equation*}
\begin{split}
&\mathbb{E}|b(x,U) - a(x)|^4 = \frac{1}{s^4} \mathbb{E} \left| \sum_{i=1}^{s} (\hat{b}(x, \theta_{U_i}) - a(x) )  \right|^4 = \frac{1}{s^4} \sum_{i=1}^{s} \mathbb{E} \left|\hat{b}(x, \theta_{U_i}) - a(x)  \right|^4 \\
&+ 4 \frac{1}{s^4} \sum_{i,j = 1, i < j}^s \mathbb{E} \left(\hat{b}(x, \theta_{U_i}) - a(x) \right)^3 \left(\hat{b}(x, \theta_{U_j}) - a(x) \right) \\
&+ 6 \frac{1}{s^4} \sum_{i,j = 1, i < j}^s \mathbb{E} \left(\hat{b}(x, \theta_{U_i}) - a(x) \right)^2 \left(\hat{b}(x, \theta_{U_j}) - a(x) \right)^2 \\
&+ 24 \frac{1}{s^4} \sum_{i,j,k,l = 1, i < j < k < l}^s \mathbb{E}\left(\hat{b}(x, \theta_{U_i}) - a(x) \right) \left(\hat{b}(x, \theta_{U_j}) - a(x) \right)\left(\hat{b}(x, \theta_{U_k}) - a(x) \right)\\
&\times \left(\hat{b}(x, \theta_{U_l}) - a(x) \right) \\
&+ 12 \frac{1}{s^4} \sum_{i,j,k = 1, i < j < k}^s \mathbb{E} \left(\hat{b}(x, \theta_{U_i}) - a(x) \right)^2 \left(\hat{b}(x, \theta_{U_j}) - a(x) \right) \left(\hat{b}(x, \theta_{U_k}) - a(x) \right) \,.
\end{split}
\end{equation*}
Since $(\hat{b}(x, \theta_{U_i}) - a(x))$ are mutually independent, centered random variables, we see that 
\begin{equation}\label{e:subsampling:aux1}
\begin{split}
\mathbb{E}|b(x,U) - a(x)|^4 &= \frac{1}{s^4} \sum_{i=1}^{s} \mathbb{E} \left|\hat{b}(x, \theta_{U_i}) - a(x)  \right|^4 \\
&+ 6 \frac{1}{s^4} \sum_{i,j = 1, i < j}^s \mathbb{E} \left(\hat{b}(x, \theta_{U_i}) - a(x) \right)^2 \left(\hat{b}(x, \theta_{U_j}) - a(x) \right)^2 \,.
\end{split}
\end{equation}
We can now compute for any $i = 1, \ldots, m$
\begin{equation*}
\begin{split}
\mathbb{E} &\left|\hat{b}(x, \theta_{U_i}) - a(x)  \right|^4 = \sum_{i=1}^{m} \left(\hat{b}(x, \theta_i) - a(x) \right)^4 \frac{1}{m} \\
&\leq \frac{1}{m} \sum_{i=1}^{m} \left(\hat{b}(x, \theta_i)^4 - 4 \hat{b}(x, \theta_i)^3 a(x) + 6 \hat{b}(x, \theta_i)^2 a(x)^2 \right) \leq C(1+|x|^4) \,,
\end{split}
\end{equation*}
where we used the linear growth conditions for $a(x)$ and for $\hat{b}(x, \theta)$. Hence we see that the first term on the right hand side of (\ref{e:subsampling:aux1}) can be bounded by $\frac{1}{s^3}C(1+|x|^4)$. On the other hand, using $\mathbb{E} \left( \hat{b}(x, \theta_{U_i}) - a(x) \right)^2 \leq C(1+|x|^2)$ we see that the second term on the right hand side of (\ref{e:subsampling:aux1}) can be bounded by $\frac{1}{s^4} \binom{s}{2} C(1+|x|^4)$ and hence we obtain
$\mathbb{E}|b(x,U) - a(x)|^4 \leq \frac{1}{s^2}C(1 + |x|^4)$.
By analogy, for the estimator without replacement $b^{wor}$ we have
\begin{equation}\label{e:subsampling:aux2}
\begin{split}
&\mathbb{E}|b^{wor}(x,U) - a(x)|^4 = \frac{1}{s^4} \mathbb{E} \left| \sum_{i=1}^{m} \hat{b}(x,\theta_i)(Z_i - \frac{s}{m})  \right|^4 = \frac{1}{s^4} \sum_{i=1}^{m} \mathbb{E} \left| \hat{b}(x,\theta_i)(Z_i - \frac{s}{m})  \right|^4 \\
&+ 4 \frac{1}{s^4} \sum_{i,j = 1, i < j}^m \mathbb{E} \left(\hat{b}(x,\theta_i)(Z_i - \frac{s}{m}) \right)^3 \left(\hat{b}(x,\theta_j)(Z_j - \frac{s}{m}) \right) \\
&+ 6 \frac{1}{s^4} \sum_{i,j = 1, i < j}^m \mathbb{E}\left(\hat{b}(x,\theta_i)(Z_i - \frac{s}{m}) \right)^2 \left(\hat{b}(x,\theta_j)(Z_j - \frac{s}{m}) \right)^2 \\
&+ 24 \frac{1}{s^4} \sum_{i,j,k,l = 1, i < j < k < l}^m \mathbb{E} \left(\hat{b}(x,\theta_i)(Z_i - \frac{s}{m}) \right) \left(\hat{b}(x,\theta_j)(Z_j - \frac{s}{m}) \right) \left(\hat{b}(x,\theta_k)(Z_k - \frac{s}{m}) \right)\\
&\times  \left(\hat{b}(x,\theta_l)(Z_l - \frac{s}{m}) \right) \\
&+ 12 \frac{1}{s^4} \sum_{i,j,k = 1, i < j < k}^m \mathbb{E} \left(\hat{b}(x,\theta_i)(Z_i - \frac{s}{m}) \right)^2 \left(\hat{b}(x,\theta_j)(Z_j - \frac{s}{m}) \right) \left(\hat{b}(x,\theta_k)(Z_k - \frac{s}{m}) \right) \,.
\end{split}
\end{equation}
Recall that the random variables $Z_i$ are not independent and hence we need to compute all the terms in the sum above. We have
\begin{equation*}
\begin{split}
\frac{1}{s^4} &\sum_{i=1}^{m} \mathbb{E} \left| \hat{b}(x,\theta_i)(Z_i - \frac{s}{m})  \right|^4 =  \frac{1}{s^4} \sum_{i=1}^{m}  |\hat{b}(x,\theta_i)|^4 \left[ \left(\frac{s}{m}\right)^4 \left( 1 - \frac{s}{m}\right) + \left( 1 - \frac{s}{m}\right)^4 \frac{s}{m}\right] \\
&= \frac{1}{s^4} \sum_{i=1}^{m}  |\hat{b}(x,\theta_i)|^4 \frac{s(m-s)}{m^4} (3s^2 - 3ms + m^2) \leq \frac{1}{s^3}(1 - \frac{s}{m}) C (1 + |x|^4) \,.
\end{split}
\end{equation*}
Let us now focus on the fourth term on the right hand side of (\ref{e:subsampling:aux2}). We have
\begin{equation*}
\begin{split}
&\left( Z_i - \frac{s}{m} \right)\left( Z_j - \frac{s}{m} \right)\left( Z_k - \frac{s}{m} \right)\left( Z_l - \frac{s}{m} \right) = Z_i Z_j Z_k Z_l \\
&- \frac{s}{m} \left( Z_i Z_j Z_k + Z_i Z_k Z_l + Z_j Z_k Z_l + Z_i Z_j Z_l \right) \\
&+ \frac{s^2}{m^2} \left( Z_i Z_k + Z_j Z_k + Z_j Z_l + Z_i Z_j + Z_i Z_l + Z_k Z_l \right) - \frac{s^3}{m^3} \left( Z_i + Z_j + Z_k + Z_l \right) + \frac{s^4}{m^4}
\end{split}
\end{equation*}
and hence, using the definition of the random variables $Z_i$,
\begin{equation*}
\begin{split}
&\mathbb{E} \left( Z_i - \frac{s}{m} \right)\left( Z_j - \frac{s}{m} \right)\left( Z_k - \frac{s}{m} \right)\left( Z_l - \frac{s}{m} \right) \\
&= \binom{m-4}{s-4} / \binom{m}{s} - 4 \frac{s}{m} \binom{m-3}{s-3} / \binom{m}{s} + 6 \frac{s^2}{m^2} \binom{m-2}{s-2} / \binom{m}{s} - 4 \frac{s^3}{m^3} \frac{s}{m} + \frac{s^4}{m^4} \\
&= \frac{(s-2)(s-1)s}{(m-2)(m-1)m} \left( \frac{s-3}{m-3} - \frac{s}{m}\right) + 3 \frac{(s-1)s^2}{(m-1)m^2} \left( \frac{s}{m} - \frac{s-2}{m-2}\right) + 3 \frac{s^3}{m^3} \left( \frac{s-1}{m-1} - \frac{s}{m} \right) \,.
\end{split}
\end{equation*}
Some straightforward computations allow us then to conclude that the fourth term on the right hand side of (\ref{e:subsampling:aux2}) is bounded by $\frac{1}{s^2}(1-\frac{s}{m})C(1+|x|^4)$. Dealing in a similar way with the remaining terms, by tedious by otherwise simple computations we can conclude that 
\begin{equation*}
\mathbb{E}|b^{wor}(x,U) - a(x)|^4 \leq \frac{1}{s^2}(1 - \frac{s}{m}) C (1 + |x|^4) \,.
\end{equation*}

\section{On the advisability of subsampling: A simple example}\label{section:Examples}

In this section we illustrate the issues that arise in the analysis of the dependence of the cost of SGAs on the parameters $m$ and $s$. Let us begin by discussing the MSE estimates \eqref{eq MSE decomp} in more detail.

In order to disentangle various approximation errors in our analysis, it is useful to consider the SDE
\begin{equation}\label{eq SDEm timechanged2} \textstyle
dY_{t} = - \frac{1}{2m}\nabla V(Y_{t},\nu^m) d{t} + \frac{1}{\sqrt{m}}dW_{t}\,, 
\end{equation}
where $(W_t)_{t \geq 0}$ is the standard Brownian motion in $\mathbb{R}^d$. We remark that (\ref{eq SDEm timechanged2}) is the time-changed SDE $d\bar{Y}_t = - \nabla V(\bar{Y}_{t},\nu^m) d{t} + \sqrt{2} dW_{t}$ and they both have the same limiting stationary distribution $\pi$, cf. \cite{DurmusRobertsVilmartZygalakis2017, Xifara2014}.
In the analysis of the mean square error, for $t = kh$, we can estimate
\begin{equation}
\begin{split}
\label{eq mse1}
&\operatorname{MSE}(\mathcal{A}^{f,k,N}) \leq     \left| (f,\pi) - (f,\mathcal{L}(Y_t))  \right|
+  \left|  (f,\mathcal{L}(Y_t))   -(f,\mathcal{L}(X_k))  \right|   +  \left( N^{-1}\mathbb V[f(X_k)] \right)^{1/2} \,,	
\end{split}
\end{equation}
where $X_{k+1} = X_k - \frac{h}{2m} \nabla V(X_k, \nu^s) + \sqrt{h/m} Z_{k+1}$ with i.i.d.\ $(Z_k)_{k=1}^{\infty}$ with the standard normal distribution.
The three terms above are, in order, bias (due to the simulation up to a finite time $t > 0$), weak time discretisation error and the Monte Carlo variance. 
We choose to work with the SDE (\ref{eq SDEm timechanged2}), to mitigate the effect of $m$ on the Lipschitz and convexity constants that play the key role in the first two errors in \eqref{eq mse1}, see the discussion in \cite{TrueCost}. Consequently, we focus on the last term in \eqref{eq mse1}, i.e., the variance of $\mathcal{A}^{f,k,N}$, in our analysis.

For convenience we assume that $h = 1/n$ for some $n \geq 1$, which corresponds to taking $n$ steps in each unit time interval. There are numerous results in the stochastic analysis literature for bounding the first term on the right hand side of (\ref{eq mse1}) by a quantity of order $\mathcal{O}(e^{-t})$, under fairly general assumptions on $\nabla V$, see e.g.\ \cite{Eberle2016, EberleGuillinZimmer2019}. Moreover, in our previous paper \cite{MajkaMijatovicSzpruch2018} we carried out the weak error analysis (see Theorem 1.5 therein) that provides an upper bound for the second term in (\ref{eq mse1}) of order $\mathcal{O}(h)$. Hence, for any algorithm $\mathcal{A}^{f,k,N}$ based on the chain $(X_k)_{k=0}^{\infty}$ given above, we have
\begin{equation}\label{eq mse2} \textstyle
\operatorname{MSE}(\mathcal{A}^{f,k,N} )\lesssim e^{{-}\lambda t} + \frac{1}{n} +\frac{1}{\sqrt{\Lambda(s,m) N}} \,,
\end{equation}
for some $\lambda > 0$ (which is the exponential rate of convergence in the $L^1$-Wasserstein distance of the SDE (\ref{eq SDEm timechanged2}) to $\pi$). Here $\Lambda(s,m)$ is a quantity whose exact value depends on the properties of $V$ and the function $f$, cf.\ the discussion below. Fix $\varepsilon >0$ and set 
$
\operatorname{MSE}(\mathcal{A}^{f,k,N}) \lesssim \varepsilon.
$
This enforces the following choice of the parameters 
$
t \approx \lambda^{-1} \log(\varepsilon^{-1}), \quad \Lambda(s,m) N \approx \varepsilon^{-2}, \quad n \approx {\varepsilon}^{-1}. 
$
The cost of simulation of our algorithm is defined as the product of the number of paths $N$, the number of iterations $k$ of each path and the number of data points $s$ in each iteration. Since $t = kh$ and $h = 1/n$, we have 
\begin{equation}\label{eq cost} \textstyle
\operatorname{cost}(\mathcal{A}^{f,k,N}) = t n N s \approx \frac{s}{\Lambda(s,m)} \log(\varepsilon^{-1}) \varepsilon^{-3}. 
\end{equation}
The main difficulty in quantifying the cost of such Monte Carlo algorithms stems from the fact that the value of $\Lambda(s,m)$ is problem-specific and depends substantially on the interplay between parameters $m$, $s$ and $h$. Hence, one may obtain different costs (and thus different answers to the question of profitability of using mini-batching) for different models and different data regimes \cite{TrueCost}.

In order to gain some insight into possible values of $\Lambda(s,m)$,
 we consider a simple example of an Ornstein-Uhlenbeck Markov chain $(X_k)_{k=0}^{\infty}$ given by 
\begin{equation}\label{eq Ornsteinm}
X_{k+1} = X_k - \alpha X_k h + \left( \frac{1}{m} \sum_{i=1}^{m} \xi_i \right) h + \sqrt{h/m} Z_{k+1} \,,
\end{equation}
where $\alpha > 0$ and $(\xi_i)_{i=1}^{m}$ are data points in $\mathbb{R}^k$, and its stochastic gradient counterpart $(\bar{X}_k)_{k=0}^{\infty}$ given by
\begin{equation}\label{eq Ornsteins}
\bar{X}_{k+1} = \bar{X}_k - \alpha \bar{X}_k h + \left( \frac{1}{s} \sum_{i=1}^{s} \xi_{U_i^k} \right) h + \sqrt{h/m} Z_{k+1} \,,
\end{equation}
where for all $i \in \{ 1, \ldots , s \}$ and for all $k \geq 0$ we have $U_i^k \sim \operatorname{Unif}(\{ 1, \ldots , m \})$ and all the random variables $U_i^k$ are mutually independent. 
Denoting $b:= \frac{1}{m} \sum_{i=1}^{m} \xi_i$ and $\bar{b}^k := \frac{1}{s} \sum_{i=1}^{s} \xi_{U_i^k}$, we easily see that
\begin{equation*}
X_k = (1-\alpha h)^k X_0 + \sum_{j=1}^{k} (1-\alpha h)^{k-j} (bh + \sqrt{h/m}Z_j)
\end{equation*}
and
$\bar{X}_k = (1-\alpha h)^k \bar{X}_0 + \sum_{j=1}^{k} (1-\alpha h)^{k-j} (\bar{b}^{j-1}h + \sqrt{h/m}Z_j)$.
Since $\mathbb{V}[Z_j] = 1$ for all $j \geq 1$, we observe that 
\begin{equation}\label{eq Ornsteinm var}
\mathbb{V}[X_k] = (1-\alpha h)^{2k} \mathbb{V}[X_0] + \frac{h}{m} \sum_{j=1}^{k} (1-\alpha h)^{2(k-j)} \,.
\end{equation}
Moreover,
$\mathbb{V}[\bar{X}_k] = (1-\alpha h)^{2k} \mathbb{V} [\bar{X}_0] + \sum_{j=1}^{k} (1-\alpha h)^{2(k-j)} \left(h^2 \mathbb{V}[\bar{b}^{j-1}] + \frac{h}{m}\right)$.
Following the calculations in Example 2.15 in \cite{MajkaMijatovicSzpruch2018}, we see that for any $j \geq 1$
\begin{equation*}
\mathbb{V}[\bar{b}^j] = \frac{1}{s} \left[ \frac{1}{m} \sum_{j=1}^{m} \xi_j^2 - \left( \frac{1}{m} \sum_{j=1}^{m} \xi_j \right)^2 \right] \,,
\end{equation*}
which, assuming $\bar{X}_0 = X_0$, shows
\begin{equation}\label{eq Ornsteins var}
\mathbb{V}[\bar{X}_k] = \mathbb{V}[X_k] + \frac{h^2}{s} \left[ \frac{1}{m} \sum_{j=1}^{m} \xi_j^2 - \left( \frac{1}{m} \sum_{j=1}^{m} \xi_j \right)^2 \right] \sum_{j=1}^{k} (1-\alpha h)^{2(k-j)} \,.
\end{equation}
Since the sum $\sum_{j=1}^{k} (1-\alpha h)^{2(k-j)}$ is of order $1/h$, we infer that $\mathbb{V}[X_k]$ is of order $1/m$, whereas $\mathbb{V}[\bar{X}_k]$ is of order $1/m + h/s$. Note that this corresponds to taking $f(x) = x$ in $\mathcal{A}^{f,k,N}$ and demonstrates that even in this simple case it is not clear whether the algorithm $\mathcal{A}^{f,k,N}$ based on $\bar{X}_k$ is more efficient than the one based on $X_k$, since the exact values of their costs (\ref{eq cost}) depend on the relation between $m$, $s$ and $h$. Note that our analysis in this case is exact, since we used equalities everywhere.

Let us also consider the case of $f(x) = x^2$, which turns out to be more cumbersome. We first assume that $\bar{X}_0 = X_0$ and observe that then $\mathbb{E}[\bar{X}_k] = \mathbb{E}[X_k]$ for all $k \geq 0$ and hence it is sufficient to compare the variances of the centered versions of $\bar{X}_k$ and $X_k$. More precisely, in our analysis of the algorithm $\cA^{f,k,N}$ we want to look at $\mathbb{V}[f(X_k - \mathbb{E}[X_k])]$ and $\mathbb{V}[f(\bar{X}_k - \mathbb{E}[\bar{X}_k])]$ with $f(x) = x^2$ and hence we will compare their respective upper bounds $\mathbb{E}|X_k - \mathbb{E}[X_k]|^4$ and $\mathbb{E}|\bar{X}_k - \mathbb{E}[\bar{X}_k]|^4$. First we observe that
\begin{equation}\label{eq Ornsteinm var aux1}
\mathbb{E}|X_k - \mathbb{E}[X_k]|^4 = \mathbb{E} \left| (1-\alpha h)^k (X_0 - \mathbb{E}[X_0]) + \sum_{j=1}^{k} (1-\alpha h)^{k-j} \frac{\sqrt{h}}{\sqrt{m}} Z_j \right|^4 \,.
\end{equation}
Hence we can expand the fourth power of the sum as in Section \ref{sectionSubsampling} and, after taking into account all the cross-terms, we see that $\mathbb{E}|X_k - \mathbb{E}[X_k]|^4$ is of order $h/m^2$. On the other hand, using $(a+b)^4 \leq 8a^4 + 8b^4$ we get
\begin{equation}\label{eq Ornsteins var aux1}
\mathbb{E}|\bar{X}_k - \mathbb{E}[\bar{X}_k]|^4 \leq 8\mathbb{E}|X_k - \mathbb{E}[X_k]|^4 + 8\mathbb{E} \left| \sum_{j=1}^{k} (1-\alpha h)^{k-j} (\bar{b}^{j-1} - \mathbb{E}[\bar{b}^{j-1}]) h \right|^4 \,.
\end{equation}
Now the analysis follows again Section \ref{sectionSubsampling} in expanding the fourth power of the sum. Note that similarly as in \eqref{e:subsampling:aux1} the terms involving $\mathbb{E}(\bar{b}^{j-1} - \mathbb{E}[\bar{b}^{j-1}])$ and $\mathbb{E}(\bar{b}^{j-1} - \mathbb{E}[\bar{b}^{j-1}])^3$ vanish and hence we are left with the terms involving $\mathbb{E}(\bar{b}^{j-1} - \mathbb{E}[\bar{b}^{j-1}])^2$ and $\mathbb{E}(\bar{b}^{j-1} - \mathbb{E}[\bar{b}^{j-1}])^4$, for which we can use the bounds obtained in Example 2.15 in \cite{MajkaMijatovicSzpruch2018} and in Section \ref{sectionSubsampling}, to conclude that the second term on the right hand side of (\ref{eq Ornsteins var aux1}) is of order $h^3/s^2$.
Hence we conclude that for the case $f(x) = x^2$
the algorithm based on $X_k$ has the variance of order $h/m^2$, while the algorithm based on $\bar{X}_k$ has the variance of order $h/m^2 + h^3/s^2$.

Finally, let us analyse $f(x) = \sqrt{x}$. To this end, we again use centered versions and compare $\mathbb{E}|X_k - \mathbb{E}[X_k]|$ with $\mathbb{E}|\bar{X}_k - \mathbb{E}[\bar{X}_k]|$. Similarly as in (\ref{eq Ornsteinm var aux1}) we observe that $\mathbb{E}|X_k - \mathbb{E}[X_k]|$ is of order $1/\sqrt{hm}$ (remembering that $\sum_{j=1}^{k} (1-\alpha h)^{k-j}$ is of order $1/h$). Moreover, similarly as in (\ref{eq Ornsteins var aux1}) we have
\begin{equation}\label{eq Ornsteins var aux2}
\mathbb{E}|\bar{X}_k - \mathbb{E}[\bar{X}_k]| \leq \mathbb{E}|X_k - \mathbb{E}[X_k]| + \mathbb{E} \left| \sum_{j=1}^{k} (1-\alpha h)^{k-j} (\bar{b}^{j-1} - \mathbb{E}[\bar{b}^{j-1}]) h \right| \,.
\end{equation}
Hence, recalling once again from Example 2.15 in \cite{MajkaMijatovicSzpruch2018} that $\mathbb{E}(\bar{b}^{j} - \mathbb{E}[\bar{b}^{j}])^2$ is of order $1/s$ for any $j \geq 1$, we can use Jensen's inequality to conclude that $\mathbb{E}(\bar{b}^{j} - \mathbb{E}[\bar{b}^{j}])$ is of order $1/\sqrt{s}$ and, consequently, that the second term on the right hand side of (\ref{eq Ornsteins var aux2}) is also of order $1/\sqrt{s}$ (after cancellation of $h$ with the sum $\sum_{j=1}^{k} (1-\alpha h)^{k-j}$ which, as we already pointed out, is of order $1/h$).
Hence for $f(x) = \sqrt{x}$ we observe that $X_k$ leads to the variance of order $1/\sqrt{hm}$, whereas $\bar{X}_k$ leads to the variance of order $1/\sqrt{hm} + 1/\sqrt{s}$. Again, in all these cases, determining which term is the leading one depends on the interplay between $m$, $s$ and $h$.

\section{Appendix: Uniform bounds on fourth moments of SGLD}

\begin{lemma}\label{lem:fourthmoment}
	Let Assumptions \ref{as diss}, \ref{as fourth moment} and \ref{as linear fourth} hold. Then there exists a constant $C_{IEul}^{(4)} > 0$ such that for the Markov chain $(X_k)_{k=0}^{\infty}$ given by
	\begin{equation*}
	X_{k+1} = X_k + hb(X_k, U_k) + \beta \sqrt{h} Z_{k+1} 
	\end{equation*}
	with pairwise independent $(U_k)_{k = 0}^{\infty}$ satisfying (\ref{estimator}) and $(Z_k)_{k = 0}^{\infty}$ i.i.d.\,, $Z_k \sim N(0,I)$ and independent of $(U_k)_{k = 0}^{\infty}$, we have
	\begin{equation*}
	\mathbb{E}|X_k|^4 \leq C_{IEul}^{(4)}
	\end{equation*}
	for all $k \geq 1$ and $h \in (0, h_0)$, where
	$C_{IEul}^{(4)} := \mathbb{E}|X_0|^4 + \frac{C^{(4)}_{ult}}{c}$
	with $C^{(4)}_{ult} > 0$ given by (\ref{e:lemFM2}) and $c$, $h_0 > 0$ determined by (\ref{e:lemFM1}).
	\begin{proof}
		By a standard computation, we have
		\begin{equation*}
		\begin{split}
		\mathbb{E}|X_{k+1}|^4 &\leq \mathbb{E}|X_k|^4 + h^4 \mathbb{E} |b(X_k, U_k)|^4 + \beta^4 h^2 \mathbb{E}|Z_{k+1}|^4 + 6h^2 \mathbb{E}|X_k|^2 |b(X_k, U_k)|^2 \\
		&+ 6 \beta^2 h \mathbb{E}|X_k|^2 |Z_{k+1}|^2 + 6\beta^2 h^3 \mathbb{E}|b(X_k, U_k)|^2 |Z_{k+1}|^2 + 4h \mathbb{E}|X_k|^2 \langle X_k, b(X_k, U_k) \rangle \\
		&+ 4 \beta \sqrt{h} \mathbb{E}|X_k|^2 \langle X_k, Z_{k+1} \rangle + 4 h^3 \mathbb{E}|b(X_k, U_k)|^2 \langle b(X_k, U_k) , X_k \rangle \\
		&+ 4 \beta h^3 \mathbb{E}|b(X_k, U_k)|^2 \langle b(X_k, U_k) , Z_{k+1} \rangle + 4\beta^3 h^{3/2} \mathbb{E} |Z_{k+1}|^2 \langle Z_{k+1} , X_k \rangle \\
		&+ 4\beta^3 h^{3/2} \mathbb{E} |Z_{k+1}|^2 \langle Z_{k+1} , b(X_k, U_k) \rangle + 6 \beta h^{3/2} \mathbb{E}|X_k|^2 \langle b(X_k, U_k) , Z_{k+1} \rangle \\
		&+ 6 \beta h^{5/2} \mathbb{E}|b(X_k, U_k)|^2 \langle X_k, Z_{k+1} \rangle + 6 \beta^2 h^2 \mathbb{E}|Z_{k+1}|^2 \langle X_k ,  b(X_k, U_k) \rangle =: \sum_{j=1}^{15} \tilde{I}_j \,.
		\end{split}
		\end{equation*}
		By conditioning on $X_k$ and $U_k$ and using properties of the multivariate normal distribution, we see that
		$\tilde{I}_8 = \tilde{I}_{10} = \tilde{I}_{11} = \tilde{I}_{12} = \tilde{I}_{13} = \tilde{I}_{14} = 0$.
		Hence we have
		\begin{equation*}
		\begin{split}
		\mathbb{E}|X_{k+1}|^4 &\leq \mathbb{E}|X_k|^4 + h^4 \mathbb{E} |b(X_k, U_k)|^4 + \beta^4 h^2 \mathbb{E}|Z_{k+1}|^4 + 6h^2 \mathbb{E}|X_k|^2 |b(X_k, U_k)|^2 \\
		&+ 6\beta^2 h \mathbb{E}|X_k|^2 |Z_{k+1}|^2 + 6\beta^2 h^3 \mathbb{E}|b(X_k, U_k)|^2 |Z_{k+1}|^2 + 4h \mathbb{E}|X_k|^2 \langle X_k, b(X_k, U_k) \rangle \\
		&+ 4 h^3 \mathbb{E}|b(X_k, U_k)|^2 \langle b(X_k, U_k) , X_k \rangle + 6 \beta^2 h^2 \mathbb{E}|Z_{k+1}|^2 \langle X_k ,  b(X_k, U_k) \rangle =: \sum_{j=1}^{9} I_j \,.
		\end{split}
		\end{equation*}
		Now observe that due to Assumptions \ref{as linear fourth} and \ref{as fourth moment} we have 
		$\mathbb{E}|b(x,U)|^4 \leq \bar{L}_0^{(4)} (1 + |x|^4)$,
		where
		\begin{equation*}
		\bar{L}_0^{(4)} := 8(\sigma^{(4)} + L_0^{(4)}) \,.
		\end{equation*}
		Indeed, we have
		\begin{equation*}
		\mathbb{E}|b(x,U)|^4 = \mathbb{E}|b(x,U) - a(x) + a(x)|^4 \leq 8 \mathbb{E}|b(x,U) - a(x)|^4 + 8 \mathbb{E}|a(x)|^4 \,.
		\end{equation*}
		Moreover, we have
	$\left( \mathbb{E}|b(x,U)|^3 \right)^{4/3} \leq \mathbb{E} \left( |b(x,U)|^3 \right)^{4/3} = \mathbb{E}|b(x,U)|^4 \leq \bar{L}_0^{(4)}(1 + |x|^4)$ 
		and hence
		\begin{equation*}
		\mathbb{E} |b(x,U)|^3 \leq \left( \bar{L}_0^{(4)} \right)^{3/4} \left( 1 + |x|^4 \right)^{3/4} \leq \left( \bar{L}_0^{(4)} \right)^{3/4} \left( 1 + |x|^3 \right) \,.
		\end{equation*}
		These auxiliary estimates allow us to bound
		$I_2 \leq h^4 \bar{L}_0^{(4)} \left( 1 + \mathbb{E}|X_k|^4 \right)$
		and
		\begin{equation*}
		I_8 \leq 4h^3 \mathbb{E} |b(X_k, U_k)|^3 |X_k| \leq 4h^3 \left( \bar{L}_0^{(4)} \right)^{3/4} \left( C_{IEul}^{1/2} + \mathbb{E}|X_k|^4 \right) \,.
		\end{equation*}
		Moreover, by conditioning, we get
		\begin{equation*}
		\begin{split}
		I_3 &= \beta^4 h^2(d^2 + 2d) \\
		I_4 &\leq 6h^2 \bar{L}_0 \left( \mathbb{E}|X_k|^2 + \mathbb{E}|X_k|^4 \right) \leq 6h^2 \bar{L}_0 C_{IEul} + 6h^2 \bar{L}_0 \mathbb{E}|X_k|^4 \\
		I_5 &= 6\beta^2 hd \mathbb{E}|X_k|^2 \leq 6\beta^2 hd C_{IEul} \\
		I_6 &= 6\beta^2 h^3d \mathbb{E} |b(X_k, U_k)|^2 \leq 6\beta^2 h^3 d \bar{L}_0 (1 + \mathbb{E}|X_k|^2) \leq 6\beta^2 h^3 d \bar{L}_0 (1 + C_{IEul}) \\
		I_7 &= 4h \mathbb{E}|X_k|^2 \langle X_k, a(X_k) \rangle \leq 4h M_2 \mathbb{E}|X_k|^2 - 4hM_1 \mathbb{E}|X_k|^4 \\
		I_9 &\leq 6\beta^2 h^2 d M_2 - 6\beta^2 h^2 d M_1 \mathbb{E}|X_k|^2 \,,
		\end{split}
		\end{equation*}
		where in $I_4$ and $I_6$ we used (\ref{estimatorGrowth}) and in $I_7$ and $I_9$ we used (\ref{driftLyapunov}). Hence we obtain
		\begin{equation*}
		\begin{split}
		\mathbb{E}|X_{k+1}|^4 &\leq - 4hM_1 \mathbb{E}|X_k|^4 + \left( 1 + h^4 \bar{L}_0^{(4)} + 6h^2 \bar{L}_0 + 4h^3 \left( \bar{L}_0^{(4)} \right)^{3/4} \right)\mathbb{E}|X_k|^4 \\
		&+ \beta^4 h^2(d^2 + 2d) + 6\beta^2 h^2 d M_2 + 4h^3 \left( \bar{L}_0^{(4)} \right)^{3/4} C_{IEul}^{1/2} + 4hM_2 C_{IEul} \\
		&+ 6\beta^2 h^3 d \bar{L}_0 (1 + C_{IEul}) + 6\beta^2 hd C_{IEul} + 6h^2 \bar{L}_0 C_{IEul} + h^4 \bar{L}_0^{(4)} \,.
		\end{split}
		\end{equation*}
		We can now choose constants $c$, $h_0 > 0$ such that for all $h \in (0, h_0)$ we have
		\begin{equation}\label{e:lemFM1}
		h^4 \bar{L}_0^{(4)} + 6h^2 \bar{L}_0 + 4h^3 \left( \bar{L}_0^{(4)} \right)^{3/4} - 4hM_1 \leq - ch
		\end{equation}
		holds for all $h \in (0, h_0)$. Then, putting
		\begin{equation}\label{e:lemFM2}
		\begin{split}
		C^{(4)}_{ult} &:= \beta^4 h_0(d^2 + 2d) + 6\beta^2 h_0 d M_2 + 4h_0^2 \left( \bar{L}_0^{(4)} \right)^{3/4} C_{IEul}^{1/2} + 4M_2 C_{IEul} \\
		&+ 6\beta^2 h_0^2 d \bar{L}_0 (1 + C_{IEul}) + 6\beta^2 d C_{IEul} + 6h_0 \bar{L}_0 C_{IEul} + h_0^3 \bar{L}_0^{(4)}
		\end{split}
		\end{equation}
		we get
	$\mathbb{E}|X_{k+1}|^4 \leq (1-ch)\mathbb{E}|X_k|^4 + C^{(4)}_{ult}h$
		for all $h \in (0, h_0)$, and hence
		\begin{equation*}
		\mathbb{E}|X_{k+1}|^4 \leq (1-ch)^{k+1}\mathbb{E}|X_0|^4 + \sum_{j=0}^{k}C^{(4)}_{ult}(1-ch)^j h \leq (1-ch)^{k+1}\mathbb{E}|X_0|^4 + \frac{C^{(4)}_{ult}}{c} \,.
		\end{equation*}
	\end{proof}
\end{lemma}

\section{Appendix: Proofs for AMLMC via subsampling}\label{section:ProofsForAMLMCviaSubsampling}

\subsection{Proof of Lemma \ref{thm:AntitheticSubsamplingAux}}

\begin{proof}	
We have
\begin{equation*}
\begin{split}
&\mathbb{E} \left| X^{c+}_{k+1} - X^{c-}_{k+1} \right|^4 = \mathbb{E} \left| X^{c+}_{k} + hb(X^{c+}_{k}, U^{f,2}_{k}) - X^{c-}_{k} - hb(X^{c-}_{k}, U^{f,1}_{k}) \right|^4 \\
&\leq \mathbb{E} \left| X^{c+}_{k} - X^{c-}_{k} \right|^4 \\
&+ 4 \mathbb{E} \left| X^{c+}_{k} - X^{c-}_{k} \right|^2 \left \langle X^{c+}_{k} - X^{c-}_{k} , hb(X^{c+}_{k}, U^{f,2}_{k}) - hb(X^{c-}_{k}, U^{f,1}_{k}) \right \rangle \\
&+ 6 \mathbb{E} \left[ \left| X^{c+}_{k} - X^{c-}_{k} \right|^2 \left| hb(X^{c+}_{k}, U^{f,2}_{k}) - hb(X^{c-}_{k}, U^{f,1}_{k})\right|^2 \right] \\
&+ 4 \mathbb{E} \left \langle X^{c+}_{k} - X^{c-}_{k} , hb(X^{c+}_{k}, U^{f,2}_{k}) - hb(X^{c-}_{k}, U^{f,1}_{k}) \right \rangle \left| hb(X^{c+}_{k}, U^{f,2}_{k}) - hb(X^{c-}_{k}, U^{f,1}_{k})\right|^2 \\
&+ \mathbb{E} \left| hb(X^{c+}_{k}, U^{f,2}_{k}) - hb(X^{c-}_{k}, U^{f,1}_{k})\right|^4 =: B_1 + B_2 + B_3 + B_4 + B_5 \,.
\end{split}
\end{equation*}
We obtain
$B_2 \leq - 4hK \mathbb{E} |X^{c+}_{k} - X^{c-}_{k}|^4$ by conditioning on $X^{c+}_{k}$ and $X^{c-}_{k}$ and using Assumption \ref{as diss}(ii).
In order to deal with $B_5$ we write
\begin{equation}\label{e:antithetic:decomposition2}
\begin{split}
b(X^{c+}_{k}, U^{f,2}_{k}) - b(X^{c-}_{k}, U^{f,1}_{k}) &= b(X^{c+}_{k}, U^{f,2}_{k}) - a(X^{c+}_{k}) + a(X^{c-}_{k}) - a(X^{c+}_{k}) \\
&+ a(X^{c-}_{k}) - b(X^{c-}_{k}, U^{f,1}_{k}) \,.
\end{split}
\end{equation}
We will now use the inequality
$\left( \sum_{j=1}^n a_j \right)^k \leq n^{k-1} \left( \sum_{j=1}^n a_j^k \right)$,
which holds for all $a_j \geq 0$ and all integers $k \geq 2$ and $n \geq 1$ due to the H\"{o}lder inequality for sums. Hence we have $(a+b+c)^4 \leq 27(a^4 + b^4 + c^4)$ and we obtain
\begin{equation*}
\begin{split}
B_5 \leq 27 h^4 \Big(  &\mathbb{E} \left| b(X^{c+}_{k}, U^{f,2}_{k}) -  a(X^{c+}_{k}) \right|^4 +  \mathbb{E} \left| a(X^{c+}_{k}) -  a(X^{c-}_{k}) \right|^4 \\
&+  \mathbb{E} \left| b(X^{c-}_{k}, U^{f,1}_{k}) -  a(X^{c-}_{k}) \right|^4 \Big) \,.
\end{split}
\end{equation*}
Hence, due to Assumption \ref{as fourth moment}, we get
$B_5 \leq 27 h^4 \left( 2\sigma^{(4)} (1+ C^{(4),(2h)}_{IEul})  +  L^4 \mathbb{E} \left| X^{c+}_{k} -  X^{c-}_{k} \right|^4 \right)$.
Using the Cauchy-Schwarz inequality, we now have
\begin{equation*}
\begin{split}
B_4 &\leq 4 h^3 \left( \mathbb{E} \left| X^{c+}_{k} - X^{c-}_{k} \right|^2 \left| b(X^{c+}_{k}, U^{f,2}_{k}) - b(X^{c-}_{k}, U^{f,1}_{k}) \right|^2 \right)^{1/2} \\
&\times \left( \mathbb{E} \left| b(X^{c+}_{k}, U^{f,2}_{k}) - b(X^{c-}_{k}, U^{f,1}_{k})\right|^4 \right)^{1/2} \\
&\leq 4 h^3 \left( \left( \mathbb{E} \left| X^{c+}_{k} - X^{c-}_{k} \right|^4 \right)^{1/2} \left( \mathbb{E} \left| b(X^{c+}_{k}, U^{f,2}_{k}) - b(X^{c-}_{k}, U^{f,1}_{k}) \right|^4 \right)^{1/2} \right)^{1/2} \\
&\times \left( 27 \left( 2\sigma^{(4)} (1+ C^{(4),(2h)}_{IEul})  +  L^4 \mathbb{E} \left| X^{c+}_{k} -  X^{c-}_{k} \right|^4 \right) \right)^{1/2} \,.
\end{split}
\end{equation*}
Hence, after applying the inequalities $ab \leq \frac{1}{2}a^2 + \frac{1}{2}b^2$ and $(a+b)^{1/2} \leq a^{1/2} + b^{1/2}$ several times, we obtain
\begin{equation*}
\begin{split}
B_4 &\leq  4 h^3 \left( \frac{1}{2} \mathbb{E} \left| X^{c+}_{k} - X^{c-}_{k} \right|^4 + \frac{1}{2} \mathbb{E} \left| b(X^{c+}_{k}, U^{f,2}_{k}) - b(X^{c-}_{k}, U^{f,1}_{k}) \right|^4 \right)^{1/2} \\
&\times \left(  \left( 54 \sigma^{(4)} (1+ C^{(4),(2h)}_{IEul}) \right)^{1/2} + \left( 27 L^4 \mathbb{E} \left| X^{c+}_{k} -  X^{c-}_{k} \right|^4 \right)^{1/2} \right) \\
&\leq  4 h^3 \left( \frac{1}{2} \mathbb{E} \left| X^{c+}_{k} - X^{c-}_{k} \right|^4 + \frac{27}{2}L^4 \mathbb{E} \left| X^{c+}_{k} - X^{c-}_{k} \right|^4 + \frac{54}{2} \sigma^{(4)} (1+ C^{(4),(2h)}_{IEul}) \right)^{1/2} \\
&\times \left(  \left( 54 \sigma^{(4)} (1+ C^{(4),(2h)}_{IEul}) \right)^{1/2} + \left( 27 L^4 \mathbb{E} \left| X^{c+}_{k} -  X^{c-}_{k} \right|^4 \right)^{1/2} \right) \\
&\leq  4 h^3 \Big( \left( \frac{1 + 27L^4}{2} 27 L^4 \right)^{1/2} \mathbb{E} \left| X^{c+}_{k} - X^{c-}_{k} \right|^4 \\
&+ \left( \frac{1 + 27L^4}{2}\mathbb{E} \left| X^{c+}_{k} - X^{c-}_{k} \right|^4 \right)^{1/2}\left( 54 \sigma^{(4)} (1+ C^{(4),(2h)}_{IEul}) \right)^{1/2} \\
&+ 27 \sqrt{2} \sigma^{(4)} (1+ C^{(4),(2h)}_{IEul}) + \left( 27 \sigma^{(4)} (1+ C^{(4),(2h)}_{IEul}) \right)^{1/2} \left( 27 L^4 \mathbb{E} \left| X^{c+}_{k} -  X^{c-}_{k} \right|^4 \right)^{1/2} \Big) \\
&\leq 4h^3 \left( \left( \frac{1 + 27L^4}{2} 27 L^4 \right)^{1/2} + \frac{1 + 81 L^4 }{4} \right) \mathbb{E} \left| X^{c+}_{k} - X^{c-}_{k} \right|^4 \\
&+ 4h^3 \left( 27 \sqrt{2} + \frac{81}{2} \right) \sigma^{(4)} (1+ C^{(4),(2h)}_{IEul}) \,.
\end{split}
\end{equation*}
In order to deal with $B_3$, we again use the decomposition (\ref{e:antithetic:decomposition2}). Then, conditioning on $X^{c+}_{k}$ and $X^{c-}_{k}$, and using Assumption \ref{as ina}, we get
\begin{equation*}
\begin{split}
B_3 &\leq 6 h^2 \Big( 3 \mathbb{E} \left( \left| X^{c+}_{k} - X^{c-}_{k} \right|^2 \sigma^2 (1+ |X^{c+}_{k}|^2) \right) + 3L^2  \mathbb{E}\left| X^{c+}_{k} - X^{c-}_{k} \right|^4 \\
&+ 3 \mathbb{E} \left( \left| X^{c+}_{k} - X^{c-}_{k} \right|^2 \sigma^2 (1+ |X^{c-}_{k}|^2) \right) \Big) \,.
\end{split}
\end{equation*}
Now, using the Cauchy-Schwarz inequality and $ab \leq \frac{1}{2}a^2 + \frac{1}{2}b^2$, we have
\begin{equation*}
\begin{split}
\mathbb{E} &\left( \left| X^{c+}_{k} - X^{c-}_{k} \right|^2 \sigma^2 (1+ |X^{c+}_{k}|^2)h^{2} \right) \\
&\leq \left( \mathbb{E} \left| X^{c+}_{k} - X^{c-}_{k} \right|^4 h \varepsilon \right)^{1/2} \left( \sigma^4 \mathbb{E} (1+|X^{c+}_{k}|^2)^2 h^{3} \frac{1}{\varepsilon} \right)^{1/2} \\
&\leq \frac{1}{2} \mathbb{E} \left| X^{c+}_{k} - X^{c-}_{k} \right|^4 h \varepsilon + \frac{1}{2 \varepsilon} h^{3} \mathbb{E} \left(1 + |X^{c+}_{k}|^2 \right)^2 \sigma^4 \\
&\leq \frac{1}{2} \mathbb{E} \left| X^{c+}_{k} - X^{c-}_{k} \right|^4 h \varepsilon +  \frac{1}{2 \varepsilon} h^{3} \left( 1 + 2 C^{(2h)}_{IEul} + C^{(4),(2h)}_{IEul} \right) \sigma^4
\end{split}
\end{equation*}
for some $\varepsilon > 0$ to be specified later. Hence
\begin{equation*}
B_3 \leq 18 \mathbb{E} \left| X^{c+}_{k} - X^{c-}_{k} \right|^4 h \varepsilon + 18 \frac{1}{\varepsilon} h^{3} \left( 1 + 2 C^{(2h)}_{IEul} + C^{(4),(2h)}_{IEul} \right) \sigma^4 + 18 h^2 L^2  \mathbb{E}\left| X^{c+}_{k} - X^{c-}_{k} \right|^4 \,.
\end{equation*}
Note that here $\sigma^4 = (\sigma^2)^2$, where $\sigma^2$ is given in Assumption \ref{as ina}, whereas $\sigma^{(4)}$ appearing in the bounds on $B_4$ and $B_5$ is possibly a different quantity, given in Assumption \ref{as fourth moment}. Combining all our estimates together, we obtain
\begin{equation*}
\begin{split}
&\mathbb{E} \left| X^{c+}_{k+1} - X^{c-}_{k+1} \right|^4 \leq C_1 h^3 \\
&+ \left( 1 - 4hK + 18h \varepsilon + 18 h^2 L^2 + 27h^4 L^4 + 4h^3 \left( \left( \frac{1 + 27L^4}{2} 27 L^4 \right)^{1/2} + \frac{1 + 81 L^4 }{4} \right) \right)\\
&\times \mathbb{E} \left| X^{c+}_{k} - X^{c-}_{k} \right|^4 \,,
\end{split}
\end{equation*}
where 
\begin{equation}\label{defC1AMLMCforSubsampling}
C_1 := 18 \frac{1}{\varepsilon} \left( 1 + 2 C^{(2h)}_{IEul} + C^{(4),(2h)}_{IEul} \right) \sigma^4 + 4 \left( 27 \sqrt{2} + 54h_0 \right) \sigma^{(4)} (1+ C^{(4),(2h)}_{IEul}) \,.
\end{equation}
Hence, choosing $h$, $c_1$, $\varepsilon$ such that
\begin{equation}\label{defc1AMLMCforSubsampling}
- 4hK + 18h \varepsilon + 18 h^2 L^2 + 27h^4 L^4 + 4h^3 \left( \left( \frac{1 + 27L^4}{2} 27 L^4 \right)^{1/2} + \frac{1 + 81 L^4 }{4} \right) \leq - c_1 h \,,
\end{equation}
we obtain, for any $k \geq 1$,
$\mathbb{E} \left| X^{c+}_{k+1} - X^{c-}_{k+1} \right|^4 \leq (1-c_1h)^{k} 	\mathbb{E} \left| X^{c+}_{0} - X^{c-}_{0} \right|^4 + \sum_{j=0}^{k} C_1 (1-c_1h)^j h^3$.
Taking $X^{c+}_{0} = X^{c-}_{0}$ and bounding the sum above by an infinite sum gives
$\mathbb{E} \left| X^{c+}_{k} - X^{c-}_{k} \right|^4 \leq (C_1/c_1) h^2$ 
for all $k \geq 1$ and finishes the proof.
\end{proof}

\subsection{Proof of Lemma \ref{thm:AntitheticSubsampling}}

	\begin{proof}
		Using $b(X^f_k,U^f_k) = \frac{1}{2}b(X^f_k,U^{f,1}_k) + \frac{1}{2}b(X^f_k,U^{f,2}_k)$, we have
		\begin{equation}\label{e:antithetic2}
		\begin{split}
		\mathbb{E}&\left| X^f_{k+1} - \bar{X}^c_{k+1} \right|^2 = \mathbb{E}\left| X^f_{k} - \bar{X}^c_{k} \right|^2 \\
		&+ 2h \mathbb{E} \left\langle X^f_{k} - \bar{X}^c_{k} , \frac{1}{2}b(X^f_{k}, U^{f,1}_{k}) + \frac{1}{2}b(X^f_{k}, U^{f,2}_{k}) - \frac{1}{2}b(X^{c-}_{k}, U^{f,1}_{k}) - \frac{1}{2}b(X^{c+}_{k}, U^{f,2}_{k}) \right\rangle \\
		&+ h^2 \mathbb{E} \left|\frac{1}{2}b(X^f_{k}, U^{f,1}_{k}) + \frac{1}{2}b(X^f_{k}, U^{f,2}_{k}) - \frac{1}{2}b(X^{c-}_{k}, U^{f,1}_{k}) - \frac{1}{2}b(X^{c+}_{k}, U^{f,2}_{k})\right|^2 =: J_1 + J_2 + J_3 \,.
		\end{split}
		\end{equation}
		We begin by bounding $J_2$. We have
		\begin{equation*}
		\begin{split}
		J_2 &= h \mathbb{E} \left\langle X^f_{k} - \bar{X}^c_{k} , b(X^f_{k}, U^{f,1}_{k}) - b(\bar{X}^c_{k}, U^{f,1}_{k}) \right\rangle + h \mathbb{E} \left\langle X^f_{k} - \bar{X}^c_{k} , b(\bar{X}^c_{k}, U^{f,1}_{k}) - b(X^{c-}_{k}, U^{f,1}_{k}) \right\rangle \\
		&+ h \mathbb{E} \left\langle X^f_{k} - \bar{X}^c_{k} , b(X^f_{k}, U^{f,2}_{k}) - b(\bar{X}^c_{k}, U^{f,2}_{k}) \right\rangle + h \mathbb{E} \left\langle X^f_{k} - \bar{X}^c_{k} , b(\bar{X}^c_{k}, U^{f,2}_{k}) - b(X^{c+}_{k}, U^{f,2}_{k}) \right\rangle \\
		&=: I_1 + I_2 + I_3 + I_4 \,.
		\end{split}
		\end{equation*}
		By conditioning on $X^f_{k}$, $X^{c+}_{k}$ and $X^{c-}_{k}$ and using Assumption \ref{as diss} (specifically condition (\ref{driftDissipativity})) we get
$I_1 = I_3 = h \mathbb{E} \left\langle X^f_{k} - \bar{X}^c_{k} , a(X^f_{k}) - a(\bar{X}^{c}_{k}) \right\rangle \leq - hK \mathbb{E} \left| X^f_{k} - \bar{X}^c_{k} \right|^2$,
		while for the other terms we have
		$I_2 = h \mathbb{E} \left\langle X^f_{k} - \bar{X}^c_{k} , a(\bar{X}^{c}_{k}) - a(X^{c-}_{k}) \right\rangle$ and
		$I_4 = h \mathbb{E} \left\langle X^f_{k} - \bar{X}^c_{k} , a(\bar{X}^{c}_{k}) - a(X^{c+}_{k}) \right\rangle$.
		We now use the Taylor formula for $a$ and (\ref{antisymmetry}) to write
		\begin{equation*}
		\begin{split}
		I_2 + I_4 &= h \mathbb{E} \Big\langle X^f_{k} - \bar{X}^c_{k} , - \sum_{|\alpha| = 2} \int_0^1 (1-t) D^{\alpha} a\left(\bar{X}^c_{k} + t\left(\bar{X}^c_{k} - X^{c-}_{k}\right)\right) dt \left( \bar{X}^c_{k} - X^{c-}_{k} \right)^{\alpha} \\
		&- \sum_{|\alpha| = 2} \int_0^1 (1-t) D^{\alpha} a\left(\bar{X}^c_{k} + t\left(\bar{X}^c_{k} - X^{c+}_{k}\right)\right) dt \left( \bar{X}^c_{k} - X^{c+}_{k} \right)^{\alpha} \Big\rangle \\
		&\leq h C_{a^{(2)}} \mathbb{E} |X^f_{k} - \bar{X}^c_{k}| \cdot |\bar{X}^c_{k} - X^{c-}_{k}|^2 + h C_{a^{(2)}} \mathbb{E} |X^f_{k} - \bar{X}^c_{k}| \cdot |\bar{X}^c_{k} - X^{c+}_{k}|^2 \\
		&= \frac{1}{2} C_{a^{(2)}} \mathbb{E}\left[  |X^f_{k} - \bar{X}^c_{k}|h^{1/2} \cdot |X^{c+}_{k} - X^{c-}_{k}|^2 h^{1/2} \right] \\
		&\leq \frac{1}{2} C_{a^{(2)}} \left( \mathbb{E} |X^f_{k} - \bar{X}^c_{k}|^2 \varepsilon_1 h \right)^{1/2} \cdot \left(  \mathbb{E} |X^{c+}_{k} - X^{c-}_{k}|^4 \frac{1}{\varepsilon_1} h \right)^{1/2} \\
		&\leq \frac{1}{4} C_{a^{(2)}} \mathbb{E} |X^f_{k} - \bar{X}^c_{k}|^2 \varepsilon_1 h + \frac{1}{4} C_{a^{(2)}} \mathbb{E} |X^{c+}_{k} - X^{c-}_{k}|^4 \frac{1}{\varepsilon_1} h \,,
		\end{split}
		\end{equation*}
		for some $\varepsilon_1 > 0$ whose exact value will be specified later, where we used the Cauchy-Schwarz inequality and the elementary inequality $ab \leq \frac{1}{2}(a^2 + b^2)$.
		
		We now come back to (\ref{e:antithetic2}) and deal with $J_3$. We have
		\begin{equation*}
		\begin{split}
		&J_3 = \frac{1}{4} h^2 \mathbb{E} \Big|  b(X^f_{k}, U^{f,1}_{k}) - b(\bar{X}^c_{k}, U^{f,1}_{k}) + b(\bar{X}^c_{k}, U^{f,1}_{k}) - b(X^{c-}_{k}, U^{f,1}_{k}) \\ 
		&+ b(X^f_{k}, U^{f,2}_{k}) - b(\bar{X}^c_{k}, U^{f,2}_{k}) + b(\bar{X}^c_{k}, U^{f,2}_{k}) - b(X^{c+}_{k}, U^{f,2}_{k}) \Big|^2 \\
		&\leq \frac{3}{4} h^2 \bar{L}^2 \mathbb{E} \left| X^f_{k} - \bar{X}^c_{k} \right|^2 + \frac{3}{4} h^2 \bar{L}^2 \mathbb{E} \left| X^f_{k} - \bar{X}^c_{k} \right|^2 \\
		&+ \frac{3}{4} h^2 \mathbb{E} \left| b(\bar{X}^c_{k}, U^{f,1}_{k}) - b(X^{c-}_{k}, U^{f,1}_{k}) + b(\bar{X}^c_{k}, U^{f,2}_{k}) - b(X^{c+}_{k}, U^{f,2}_{k}) \right|^2 =: J_{31} + J_{32} + J_{33} \,,
		\end{split}
		\end{equation*}
		where we used the Lipschitz condition (\ref{ina driftLipschitz}). Note that we have
		$J_{31} + J_{32} = \frac{3}{2} h^2 \bar{L}^2 \mathbb{E} \left| X^f_{k} - \bar{X}^c_{k} \right|^2$.
		On the other hand, in order to deal with $J_{33}$, we use the Taylor theorem to write
		\begin{equation*}
		\begin{split}
		&b(\bar{X}^c_{k}, U^{f,1}_{k}) - b(X^{c-}_{k}, U^{f,1}_{k}) = - \Big[ \sum_{|\alpha| = 1} D^{\alpha} b(\bar{X}^c_{k}, U^{f,1}_{k}) \left( \bar{X}^c_{k} - X^{c-}_{k} \right)^{\alpha} \\
		&+ \sum_{|\alpha| = 2} \int_0^1 (1-t) D^{\alpha} b\left(\bar{X}^c_{k} + t\left(\bar{X}^c_{k} - X^{c-}_{k}\right), U^{f,1}_{k} \right) dt \left( \bar{X}^c_{k} - X^{c-}_{k} \right)^{\alpha} \Big] \,.
		\end{split}
		\end{equation*}
		Hence we have
		\begin{equation*}
		\begin{split}
		&b(\bar{X}^c_{k}, U^{f,1}_{k}) - b(X^{c-}_{k}, U^{f,1}_{k}) + b(\bar{X}^c_{k}, U^{f,2}_{k}) - b(X^{c+}_{k}, U^{f,2}_{k}) \\
		&= \frac{1}{2} \sum_{|\alpha| = 1} \left( D^{\alpha} b(\bar{X}^c_{k}, U^{f,1}_{k}) - D^{\alpha}a(\bar{X}^c_{k}) + D^{\alpha}a(\bar{X}^c_{k}) - D^{\alpha}b(\bar{X}^c_{k}, U^{f,2}_{k}) \right) \left( X^{c+}_{k} - X^{c-}_{k} \right)^{\alpha} \\
		&+ \frac{1}{4} \sum_{|\alpha| = 2} \int_0^1 (1-t) D^{\alpha} b\left(\bar{X}^c_{k} + t\left(\bar{X}^c_{k} - X^{c-}_{k}\right), U^{f,1}_{k} \right) dt \left( X^{c+}_{k} - X^{c-}_{k} \right)^{\alpha} \\
		&- \frac{1}{4} \sum_{|\alpha| = 2} \int_0^1 (1-t) D^{\alpha} b\left(\bar{X}^c_{k} + t\left(\bar{X}^c_{k} - X^{c+}_{k}\right), U^{f,2}_{k} \right) dt \left( X^{c+}_{k} - X^{c-}_{k} \right)^{\alpha} \,.
		\end{split}
		\end{equation*}
		Recall that we assume (in Assumption \ref{as smoothnessEstimator}) that
		$\mathbb{E} \left| \nabla b(x, U) - \nabla a(x) \right|^4 \leq \sigma^{(4)}(1 + |x|^4)$
		and that
		$|D^{\alpha} b(x,U)| \leq C_{b^{(2)}}$
		for all multiindices $\alpha$ with $|\alpha| = 2$. Hence we have
		\begin{equation*}
		\begin{split}
		J_{33} &\leq \frac{3}{4} h^2 \mathbb{E} \left[ \left| \nabla b(\bar{X}^c_{k}, U^{f,1}_{k}) - \nabla a (\bar{X}^c_{k}) \right|^2 \left| X^{c+}_{k} - X^{c-}_{k} \right|^2 \right] \\
		&+ \frac{3}{4} h^2 \mathbb{E} \left[ \left|  \nabla a (\bar{X}^c_{k}) - \nabla b(\bar{X}^c_{k}, U^{f,2}_{k}) \right|^2 \left| X^{c+}_{k} - X^{c-}_{k} \right|^2 \right] + \frac{3}{8} h^2 C_{b^{(2)}}^2 \mathbb{E} \left| X^{c+}_{k} - X^{c-}_{k} \right|^4 \\
		&\leq \frac{3}{4} h^3 \sigma^{(4)} (1 + \mathbb{E}|\bar{X}^c_{k}|^4) + \frac{3}{4} h \mathbb{E} \left| X^{c+}_{k} - X^{c-}_{k} \right|^4 + \frac{3}{8} h^2 C_{b^{(2)}}^2 \mathbb{E} \left| X^{c+}_{k} - X^{c-}_{k} \right|^4 \,,
		\end{split}
		\end{equation*}
		where in the second inequality we used Young's inequality. Combining all our estimates together, we see that if we choose $h$, $c_2$ and $\varepsilon_1 > 0$ such that
		\begin{equation}\label{defc2AMLMCforSubsampling}
		- h K + \frac{1}{4} d C_{a^{(2)}} \varepsilon_1 h + \frac{3}{2} h^2 \bar{L}^2 \leq - c_2 h \,,
		\end{equation}
		then we obtain
		$\mathbb{E}\left| X^f_{k+1} - \bar{X}^c_{k+1} \right|^2 \leq (1-c_2h) \mathbb{E}\left| X^f_{k} - \bar{X}^c_{k} \right|^2 + C_2 h^3$,
		where
		\begin{equation}\label{defC2AMLMCforSubsampling}
		C_2 := \frac{3}{4} \sigma^{(4)} (1 + C^{(4)}_{IEul}) + \frac{C_1}{c_1}\left( \frac{1}{4} d C_{a^{(2)}}  \frac{1}{\varepsilon_1} + \frac{3}{4} + \frac{3}{8} h_0 C_{b^{(2)}}^2 \right) \,.
		\end{equation}
		We can now finish the proof exactly as we did in Lemma \ref{thm:AntitheticSubsamplingAux}.
	\end{proof}

\section{Appendix: Proofs for MASGA}\label{sectionMIASGA}

	\begin{proof}[Proof of Lemma \ref{thm:AntitheticMasterChain}]
		The argument is very similar to the proof of Lemma \ref{thm:AntitheticSubsamplingAux} and in fact even simpler as here we have only one inaccurate drift. For completeness, we give here an outline of the proof anyway. We have
		\begin{equation*}
		\begin{split}
		\mathbb{E} \left| X^{f,f}_{k+1} - X_{k+1} \right|^4 &= \mathbb{E} \left| X^{f,f}_{k} - X_{k} + hb(X^{f,f}_{k}, U^f_{k}) - ha(X_{k}) \right|^4 \leq \mathbb{E} \left| X^{f,f}_{k} - X_{k} \right|^4 \\
		&+ 4 \mathbb{E} \left| X^{f,f}_{k} - X_{k} \right|^2 \left\langle X^{f,f}_{k} - X_{k} , hb(X^{f,f}_{k}, U^f_{k}) - ha(X_{k}) \right\rangle \\
		&+ 6 \mathbb{E} \left| X^{f,f}_{k} - X_{k} \right|^2 \left| hb(X^{f,f}_{k}, U^f_{k}) - ha(X_{k}) \right|^2 \\
		&+ 4 \mathbb{E} \left\langle X^{f,f}_{k} - X_{k} , hb(X^{f,f}_{k}, U^f_{k}) - ha(X_{k}) \right\rangle \left| hb(X^{f,f}_{k}, U^f_{k}) - ha(X_{k}) \right|^2 \\
		&+ \mathbb{E} \left| hb(X^{f,f}_{k}, U^f_{k}) - ha(X_{k}) \right|^4 =: B_1 + B_2 + B_3 + B_4 + B_5 \,.
		\end{split}
		\end{equation*}
		By conditioning and Assumption \ref{as diss}, we have $B_2 \leq - 4hK \mathbb{E} \left| X^{f,f}_{k} - X_{k} \right|^4$. Furthermore,
		\begin{equation*}
		\begin{split}
		B_5 &\leq h^4 \mathbb{E} \left| b(X^{f,f}_{k}, U^f_{k}) - a(X^{f,f}_{k}) + a(X^{f,f}_{k}) - a(X_{k}) \right|^4 \\
		&\leq 8h^4 \mathbb{E} \left| b(X^{f,f}_{k}, U^f_{k}) - a(X^{f,f}_{k}) \right|^4 + 8h^4 \mathbb{E} \left| a(X^{f,f}_{k}) - a(X^{f,f}_{k}) \right|^4 \\
		&\leq 8h^4 \sigma^{(4)} (1 + C^{(4)}_{IEul}) + 8h^4 L^4 \mathbb{E} \left| X^{f,f}_{k} - X_{k} \right|^4 \,,
		\end{split}
		\end{equation*}
		where we used Assumptions \ref{as fourth moment}, \ref{as diss} and Lemma \ref{lem:fourthmoment}. It is now clear that the terms $B_3$ and $B_4$ can be dealt with exactly as the corresponding terms in the proof of Lemma \ref{thm:AntitheticSubsamplingAux} and we obtain essentially the same estimates with sligthly different constants, which are, however, of the same order in $s$ and $h$.
	\end{proof}

\begin{proof}[Proof of Lemma \ref{thm:lemmaXi1}] We denote
	\begin{equation*}
	\begin{split}
	\Xi^A_{k} &:= h \left( b(X^{f,f}_{k}, U^f_{k}) - \frac{1}{2}\left( b(X^{c-,f}_{k}, U^{f,1}_{k}) + b(X^{c+,f}_{k}, U^{f,2}_{k}) \right) \right) \\
	\Xi^B_{k} &:= h \left( b(X^{f,c-}_{k}, U^f_{k}) - \frac{1}{2}\left( b(X^{c-,c-}_{k}, U^{f,1}_{k}) + b(X^{c+,c-}_{k}, U^{f,2}_{k}) \right) \right) \\
	\Xi^C_{k} &:= h \left( b(X^{f,c+}_{k}, U^f_{k+1}) - \frac{1}{2}\left( b(X^{c-,c+}_{k}, U^{f,1}_{k+1}) + b(X^{c+,c+}_{k}, U^{f,2}_{k+1}) \right) \right) 
	\end{split}
	\end{equation*}
	and we have $\Xi_{k} = \Xi^A_{k} - \Xi^B_{k} - \Xi^C_{k}$ . Then, using $b(x,U) = \frac{1}{2}b(x,U^1) + \frac{1}{2}b(x,U^2)$, we have
	\begin{equation*}
	\begin{split}
	\Xi^A_{k} &= \frac{1}{2}h b(X^{f,f}_{k}, U^{f,1}_{k}) - \frac{1}{2}h b(\bar{X}^{c,f}_{k}, U^{f,1}_{k}) + \frac{1}{2}h b(\bar{X}^{c,f}_{k}, U^{f,1}_{k}) - \frac{1}{2}h b(X^{c-,f}_{k}, U^{f,1}_{k}) \\
	&+ \frac{1}{2}h b(X^{f,f}_{k}, U^{f,2}_{k}) - \frac{1}{2}h b(\bar{X}^{c,f}_{k}, U^{f,2}_{k}) + \frac{1}{2}h b(\bar{X}^{c,f}_{k}, U^{f,2}_{k}) -  \frac{1}{2}h b(X^{c+,f}_{k}, U^{f,2}_{k})
	\end{split}
	\end{equation*}
	and hence
	\begin{equation*}
	\begin{split}
	\mathbb{E}|\Xi^A_{k}|^2 &\leq \frac{3}{4}h \bar{L} \mathbb{E} \left| X^{f,f}_{k} - \bar{X}^{c,f}_{k} \right|^2 + \frac{3}{4}h \bar{L} \mathbb{E} \left| X^{f,f}_{k} - \bar{X}^{c,f}_{k} \right|^2 \\
	&+\frac{3}{4}h \mathbb{E} \left| b(\bar{X}^{c,f}_{k}, U^{f,1}_{k}) - b(X^{c-,f}_{k}, U^{f,1}_{k}) +  b(\bar{X}^{c,f}_{k}, U^{f,2}_{k}) -  b(X^{c+,f}_{k}, U^{f,2}_{k}) \right|^2 \,.
	\end{split}
	\end{equation*}
	Note that the first two terms on the right hand side above are identical and have the correct order in $s$ and $h$ due to Lemma \ref{thm:AntitheticSubsampling}. Furthermore, the last term can be dealt with by applying Taylor's formula twice in $\bar{X}^{c,f}_{k}$ and using the argument from the proof of Lemma \ref{thm:AntitheticSubsampling} for the term $J_{33}$ therein. Bounds for $\mathbb{E}|\Xi^B_{k}|^2$ and $\mathbb{E}|\Xi^C_{k}|^2$ can be obtained in exactly the same way. 
\end{proof}

\begin{proof}[Proof of Lemma \ref{thm:lemmaXi2}]
We need to introduce an auxiliary chain
\begin{equation*}
X^c_{k+2} = X^c_{k} + 2h a(X^c_{k}) + \beta \sqrt{2h} \hat{Z}_{k+2} \,.
\end{equation*}
Let us begin with bounding $\mathbb{E} \langle \Psi_{k} , \Xi^1_{k} \rangle$. Recall that $\bar{X}^{f,c}_{k} = \frac{1}{2} \left( X^{f,c-}_{k} + X^{f,c+}_{k} \right)$. We denote 
\begin{equation*}
\begin{split}
\Xi^{1,1}_{k} &:= hb(X^{f,c-}_{k}, U^f_{k}) - hb(\bar{X}^{f,c}_{k}, U^f_{k}) + hb(X^{f,c+}_{k}, U^f_{k}) - hb(\bar{X}^{f,c}_{k}, U^f_{k}) \\
\Xi^{1,2}_{k} &:= hb(X^{f,f}_{k}, U^f_{k}) - hb(X_{k}, U^f_{k}) \,, \, \,
\Xi^{1,3}_{k} := hb(\bar{X}^{f,c}_{k}, U^f_{k}) - hb(X^c_{k}, U^f_{k}) \\
\Xi^{1,4}_{k} &:= hb(\bar{X}^{f,c}_{k}, U^f_{k+1}) - hb(X^c_{k}, U^f_{k+1})
\end{split}
\end{equation*}
and we see that
$\Xi^1_{k} = - \Xi^{1,1}_{k} + \Xi^{1,2}_{k} - \Xi^{1,3}_{k} - \Xi^{1,4}_{k} + hb(X_{k}, U^f_{k}) - hb(X^c_{k}, U^f_{k}) - hb(X^c_{k}, U^f_{k+1})$.
By analogy, we define
\begin{equation*}
\begin{split}
\Xi^{2,1}_{k} &:= hb(X^{c-,c-}_{k}, U^{f,1}_{k}) - hb(\bar{X}^{c-,c}_{k}, U^{f,1}_{k}) + hb(X^{c-,c+}_{k}, U^{f,1}_{k}) - hb(\bar{X}^{c-,c}_{k}, U^{f,1}_{k}) \\
\Xi^{2,2}_{k} &:= hb(X^{c-,f}_{k}, U^{f,1}_{k}) - hb(X_{k}, U^{f,1}_{k}) \,, \, \,
\Xi^{2,3}_{k} := hb(\bar{X}^{c-,c}_{k}, U^{f,1}_{k}) - hb(X^c_{k}, U^{f,1}_{k}) \\
\Xi^{2,4}_{k} &:= hb(\bar{X}^{c-,c}_{k}, U^{f,1}_{k+1}) - hb(X^c_{k}, U^{f,1}_{k+1})
\end{split}
\end{equation*}
\begin{equation*}
\begin{split}
\Xi^{3,1}_{k} &:= hb(X^{c+,c-}_{k}, U^{f,2}_{k}) - hb(\bar{X}^{c+,c}_{k}, U^{f,2}_{k}) + hb(X^{c+,c+}_{k}, U^{f,2}_{k}) - hb(\bar{X}^{c+,c}_{k}, U^{f,2}_{k}) \\
\Xi^{3,2}_{k} &:= hb(X^{c+,f}_{k}, U^{f,2}_{k}) - hb(X_{k}, U^{f,2}_{k}) \,, \, \,
\Xi^{3,3}_{k} := hb(\bar{X}^{c+,c}_{k}, U^{f,2}_{k}) - hb(X^c_{k}, U^{f,2}_{k}) \\
\Xi^{3,4}_{k} &:= hb(\bar{X}^{c+,c}_{k}, U^{f,2}_{k+1}) - hb(X^c_{k}, U^{f,2}_{k+1})
\end{split}
\end{equation*}
and hence, since $\Xi_{k} = \Xi^1_{k} - \frac{1}{2}\left( \Xi^2_{k} + \Xi^3_{k} \right)$, we see that
\begin{equation*}
\begin{split}
\Xi_{k} &= - \left( \Xi^{1,1}_{k} - \frac{1}{2} \left(\Xi^{2,1}_{k} + \Xi^{3,1}_{k} \right) \right) + \left( \Xi^{1,2}_{k} - \frac{1}{2} \left(\Xi^{2,2}_{k} + \Xi^{3,2}_{k} \right) \right) \\
&- \left( \Xi^{1,3}_{k} - \frac{1}{2} \left(\Xi^{2,3}_{k} + \Xi^{3,3}_{k} \right) \right) - \left( \Xi^{1,4}_{k} - \frac{1}{2} \left(\Xi^{2,4}_{k} + \Xi^{3,4}_{k} \right) \right) + Rh \,,
\end{split}
\end{equation*}
\begin{equation*}
\begin{split}
R &= b(X_{k}, U^f_{k}) - \frac{1}{2}\left( b(X_{k}, U^{f,1}_{k}) + b(X_{k}, U^{f,2}_{k})\right) - b(X^c_{k}, U^f_{k}) + \frac{1}{2}\left(  b(X^c_{k}, U^{f,1}_{k}) +  b(X^c_{k}, U^{f,2}_{k})\right) \\
&- b(X^c_{k}, U^f_{k+1}) + \frac{1}{2}\left( b(X^c_{k}, U^{f,1}_{k+1}) + b(X^c_{k}, U^{f,2}_{k+1})\right) = 0 \,,
\end{split}
\end{equation*}
since $b(x,U) = \frac{1}{2}b(x,U^1) + \frac{1}{2}b(x,U^2)$. We now write
\begin{equation*}
\begin{split}
\Xi^{1,1}_{k} &= h \sum_{|\alpha|=1} D^{\alpha} b(\bar{X}^{f,c}_{k}, U^f_{k}) \left( X^{f,c-}_{k} - \bar{X}^{f,c}_{k} \right)^{\alpha} \\
&+ h \sum_{|\alpha|=2} \int_0^1 (1-t) D^{\alpha} b\left(\bar{X}^{f,c}_{k} + t \left( X^{f,c-}_{k} - \bar{X}^{f,c}_{k} \right), U^f_{k}\right) dt \left( X^{f,c-}_{k} - \bar{X}^{f,c}_{k} \right)^{\alpha} \\
&+ h \sum_{|\alpha|=1} D^{\alpha} b(\bar{X}^{f,c}_{k}, U^f_{k+1}) \left( X^{f,c+}_{k} - \bar{X}^{f,c}_{k} \right)^{\alpha} \\
&+ h \sum_{|\alpha|=2} \int_0^1 (1-t) D^{\alpha} b\left(\bar{X}^{f,c}_{k} + t \left( X^{f,c+}_{k} - \bar{X}^{f,c}_{k} \right), U^f_{k+1}\right) dt \left( X^{f,c+}_{k} - \bar{X}^{f,c}_{k} \right)^{\alpha} \,.
\end{split}
\end{equation*}
Note that $X^{f,c-}_{k} - \bar{X}^{f,c}_{k} = - \left( X^{f,c+}_{k} - \bar{X}^{f,c}_{k} \right) = \frac{1}{2} \left( X^{f,c+}_{k} - X^{f,c-}_{k} \right)$. Hence
\begin{equation*}
\begin{split}
\mathbb{E} \langle \Psi_{k}, \Xi^{1,1}_{k} \rangle &= h \mathbb{E} \langle \Psi_{k}, \sum_{|\alpha|=1} D^{\alpha} a(\bar{X}^{f,c}_{k}) \left( X^{f,c-}_{k} - \bar{X}^{f,c}_{k} \right)^{\alpha} \\
&+ \sum_{|\alpha|=2} \int_0^1 (1-t) D^{\alpha} a\left(\bar{X}^{f,c}_{k} + t \left( X^{f,c-}_{k} - \bar{X}^{f,c}_{k} \right) \right) dt \left( X^{f,c-}_{k} - \bar{X}^{f,c}_{k} \right)^{\alpha} \\
&+ \sum_{|\alpha|=1} D^{\alpha} a(\bar{X}^{f,c}_{k}) \left( X^{f,c+}_{k} - \bar{X}^{f,c}_{k} \right)^{\alpha} \\
&+ \sum_{|\alpha|=2} \int_0^1 (1-t) D^{\alpha} a\left(\bar{X}^{f,c}_{k} + t \left( X^{f,c+}_{k} - \bar{X}^{f,c}_{k} \right)\right) dt \left( X^{f,c+}_{k} - \bar{X}^{f,c}_{k} \right)^{\alpha} \rangle \\
&= h \mathbb{E} \langle \Psi_{k}, \sum_{|\alpha|=2} \int_0^1 (1-t) D^{\alpha} a\left(\bar{X}^{f,c}_{k} + t \left( X^{f,c-}_{k} - \bar{X}^{f,c}_{k} \right) \right) dt \left( X^{f,c-}_{k} - \bar{X}^{f,c}_{k} \right)^{\alpha} \\
&+ \sum_{|\alpha|=2} \int_0^1 (1-t) D^{\alpha} a\left(\bar{X}^{f,c}_{k} + t \left( X^{f,c+}_{k} - \bar{X}^{f,c}_{k} \right)\right) dt \left( X^{f,c+}_{k} - \bar{X}^{f,c}_{k} \right)^{\alpha} \rangle \\
&\leq 2h C_{a^{(2)}} \varepsilon_5 \mathbb{E} |\Psi_{k}|^2 + \frac{1}{8} h C_{a^{(2)}} \frac{1}{\varepsilon_5} \mathbb{E} \left| X^{f,c+}_{k} - X^{f,c-}_{k} \right|^4 \,,
\end{split}
\end{equation*}
where we used Young's inequality with some $\varepsilon_5 > 0$ to be specified later. From Lemma \ref{lem:term2} we know that $\mathbb{E} \left| X^{f,c+}_{k} - X^{f,c-}_{k} \right|^4$ has the correct order in $s$ and $h$. Similarly, we can show
\begin{equation*}
\begin{split}
\mathbb{E} \langle \Psi_{k}, \Xi^{2,1}_{k} \rangle &\leq 2h C_{a^{(2)}} \varepsilon_6 \mathbb{E} |\Psi_{k}|^2 + \frac{1}{8} h C_{a^{(2)}} \frac{1}{\varepsilon_6} \mathbb{E} \left| X^{c-,c+}_{k} - X^{c-,c-}_{k} \right|^4 \\
\mathbb{E} \langle \Psi_{k}, \Xi^{3,1}_{k} \rangle &\leq 2h C_{a^{(2)}} \varepsilon_7 \mathbb{E} |\Psi_{k}|^2 + \frac{1}{8} h C_{a^{(2)}} \frac{1}{\varepsilon_7} \mathbb{E} \left| X^{c+,c+}_{k} - X^{c+,c-}_{k} \right|^4
\end{split}
\end{equation*}
for some $\varepsilon_6$, $\varepsilon_7 > 0$ and we also conclude that the second terms on the right hand side above have the correct order in $s$ and $h$. Note that in order to deal with the terms $\Xi^{1,1}_{k}$, $\Xi^{2,1}_{k}$ and $\Xi^{3,1}_{k}$ we did not use the structure of our estimator and we just dealt with each of them separately. This will be different in the case of the expression $\left( \Xi^{1,2}_{k} - \frac{1}{2} \left(\Xi^{2,2}_{k} + \Xi^{3,2}_{k} \right) \right)$, where we will use its structure in order to produce an additional antithetic term $X^{f,f}_{k} - \frac{1}{2}\left( X^{c-,f}_{k} + X^{c+,f}_{k} \right)$ on the right hand side of $\mathbb{E} \langle \Psi_{k}, \Xi^{1,2}_{k} - \frac{1}{2} \left(\Xi^{2,2}_{k} + \Xi^{3,2}_{k} \right)  \rangle$ below. Indeed, we first write
\begin{equation*}
\begin{split}
\Xi^{1,2}_{k} &= h \sum_{|\alpha|=1} D^{\alpha} b(X_{k}, U^f_{k}) \left( X^{f,f}_{k} - X_{k} \right)^{\alpha} \\
&+ h \sum_{|\alpha|=2} \int_0^1 (1-t) D^{\alpha} b\left(X_{k} + t\left( X^{f,f}_{k} - X_{k} \right), U^f_{k} \right) dt \left( X^{f,f}_{k} - X_{k} \right)^{\alpha} \,.
\end{split}
\end{equation*}
Then, expanding $\Xi^{2,2}_{k}$ and $\Xi^{3,2}_{k}$ in an analogous way, we see that
\begin{equation*}
\begin{split}
\mathbb{E} &\langle \Psi_{k}, \Xi^{1,2}_{k} - \frac{1}{2} \left(\Xi^{2,2}_{k} + \Xi^{3,2}_{k} \right)  \rangle = h \mathbb{E} \langle \Psi_{k} , \sum_{|\alpha|=1} D^{\alpha} a(X_{k}) \left( X^{f,f}_{k} - \frac{1}{2}\left( X^{c-,f}_{k} + X^{c+,f}_{k} \right) \right)^{\alpha} \\&+ \sum_{|\alpha|=2} \int_0^1 (1-t) D^{\alpha} a\left(X_{k} + t\left( X^{f,f}_{k} - X_{k} \right) \right) dt \left( X^{f,f}_{k} - X_{k} \right)^{\alpha} \\
&- \frac{1}{2} \sum_{|\alpha|=2} \int_0^1 (1-t) D^{\alpha} a\left(X_{k} + t\left( X^{c-,f}_{k} - X_{k} \right) \right) dt \left( X^{c-,f}_{k} - X_{k} \right)^{\alpha} \\
&- \frac{1}{2} \sum_{|\alpha|=2} \int_0^1 (1-t) D^{\alpha} a\left(X_{k} + t\left( X^{c+,f}_{k} - X_{k} \right) \right) dt \left( X^{c+,f}_{k} - X_{k} \right)^{\alpha} \rangle \\
&=: h \mathbb{E} \langle \Psi_{k} , \sum_{|\alpha|=1} D^{\alpha} a(X_{k}) \left( X^{f,f}_{k} - \frac{1}{2}\left( X^{c-,f}_{k} + X^{c+,f}_{k} \right) \right)^{\alpha} \rangle + h \mathbb{E} \langle \Psi_{k} , \hat{\Xi}^2_{k} \rangle
\,.
\end{split}
\end{equation*}
Using Young's inequality with some $\varepsilon_8 > 0$, we can now bound 
\begin{equation*}
\mathbb{E} \langle \Psi_{k} , \hat{\Xi}^2_{k} \rangle \leq
3 \varepsilon_8 C_{a^{(2)}} \mathbb{E} |\Psi_{k}|^2 + \frac{1}{\varepsilon_8} C_{a^{(2)}} \left[ \mathbb{E} \left| X^{f,f}_{k} - X_{k} \right|^4 + \mathbb{E} \left| X^{c-,f}_{k} - X_{k} \right|^4 + \mathbb{E} \left| X^{c+,f}_{k} - X_{k} \right|^4 \right] \,,
\end{equation*}
whereas the remaining antithetic term will be used later.
Note that all the fourth moments above have the correct order in $s$ and $h$ due to Lemma \ref{thm:AntitheticMasterChain}.

Now we turn our attention to $\left( \Xi^{1,3}_{k} - \frac{1}{2} \left(\Xi^{2,3}_{k} + \Xi^{3,3}_{k} \right) \right)$. We start with $\Xi^{1,3}_{k}$ by writing
\begin{equation*}
\begin{split}
\Xi^{1,3}_{k} &= h \sum_{|\alpha| = 1} D^{\alpha} b(X^c_{k}, U^f_{k}) \left( \bar{X}^{f,c}_{k} - X^c_{k} \right)^{\alpha} \\
&+ h \sum_{|\alpha|=2} \int_0^1 (1-t) D^{\alpha} b\left(  X^c_{k} + t\left(\bar{X}^{f,c}_{k} - X^c_{k} \right), U^f_{k} \right) dt \left( \bar{X}^{f,c}_{k} - X^c_{k} \right)^{\alpha} =: \Xi^{1,3,1}_{k} + \Xi^{1,3,2}_{k} \,.
\end{split}
\end{equation*}
Note that
$\mathbb{E} \langle \Psi_{k} , \Xi^{1,3,2}_{k} \rangle \leq h \varepsilon_9 C_{a^{(2)}} \mathbb{E} |\Psi_{k}|^2 + h \varepsilon_9^{-1} C_{a^{(2)}} \mathbb{E} \left| \bar{X}^{f,c}_{k} - X^c_{k} \right|^4$ for some $\varepsilon_9 > 0$.
We have
\begin{lemma}\label{thm:AntitheticMean}
	Under the assumptions of Lemma \ref{thm:AntitheticMasterChain}, there is a constant $C > 0$ such that for all $k \geq 1$,
	\begin{equation*}
	\mathbb{E} \left| \bar{X}^{f,c}_{k} - X^c_{k} \right|^4 \leq C \frac{1}{s^2}h^2 \,.
	\end{equation*}
	\begin{proof}
		We notice that 
		\begin{equation*}
		\mathbb{E} \left| \bar{X}^{f,c}_{k} - X^c_{k} \right|^4 = \mathbb{E} \left| \frac{1}{2}\left(X^{f,c-}_{k} - X^c_{k} + X^{f,c+}_{k} - X^c_{k} \right) \right|^4 \leq \frac{1}{2}\mathbb{E}\left| X^{f,c-}_{k} - X^c_{k} \right|^4 + \frac{1}{2}\mathbb{E}\left| X^{f,c+}_{k} - X^c_{k} \right|^4
		\end{equation*}
		and then use an analogue of Lemma \ref{thm:AntitheticMasterChain} for the coarse chain.
	\end{proof}
\end{lemma}

On the other hand,
\begin{equation*}
\begin{split}
\mathbb{E} &\langle \Psi_{k}, \Xi^{1,3,1}_{k} \rangle = h \mathbb{E} \langle \Psi_{k} , \sum_{|\alpha| = 1} D^{\alpha} a(X^c_{k}) \left( \bar{X}^{f,c}_{k} - X^c_{k} \right)^{\alpha} \rangle = h \mathbb{E} \langle \Psi_{k} , \sum_{|\alpha| = 1} D^{\alpha} a(X_{k}) \left( \bar{X}^{f,c}_{k} - X_{k} \right)^{\alpha} \rangle \\
&+ h \mathbb{E} \langle \Psi_{k} , \sum_{|\alpha| = 1} \left( D^{\alpha} a(X^c_{k}) \left( \bar{X}^{f,c}_{k} - X^c_{k} \right)^{\alpha} - D^{\alpha} a(X_{k}) \left( \bar{X}^{f,c}_{k} - X_{k} \right)^{\alpha} \right) \rangle \\
&=: \mathbb{E} \langle \Psi_{k}, \Xi^{1,3,1,1}_{k} \rangle + \mathbb{E} \langle \Psi_{k}, \Xi^{1,3,1,2}_{k} \rangle \,.
\end{split}
\end{equation*}
Now observe that
\begin{equation}\label{eq auxMASGAproof0}
\begin{split}
\mathbb{E} \langle \Psi_{k}, \Xi^{1,3,1,2}_{k} \rangle &= h \mathbb{E} \langle \Psi_{k} , \sum_{|\alpha| = 1} \left( D^{\alpha} a(X^c_{k}) \left( \bar{X}^{f,c}_{k} - X^c_{k} \right)^{\alpha} - D^{\alpha} a(X_{k}) \left( \bar{X}^{f,c}_{k} - X^c_{k} \right)^{\alpha} \right) \rangle \\
&+ h \mathbb{E} \langle \Psi_{k} , \sum_{|\alpha| = 1} \left( D^{\alpha} a(X_{k}) \left( \bar{X}^{f,c}_{k} - X^c_{k} \right)^{\alpha} - D^{\alpha} a(X_{k}) \left( \bar{X}^{f,c}_{k} - X_{k} \right)^{\alpha} \right) \rangle \\
&= h \mathbb{E} \langle \Psi_{k} , \sum_{|\alpha| = 1} \left( D^{\alpha} a(X^c_{k}) - D^{\alpha} a(X_{k}) \right) \left( \bar{X}^{f,c}_{k} - X^c_{k} \right)^{\alpha}  \rangle \\
&+ h \mathbb{E} \langle \Psi_{k} , \sum_{|\alpha| = 1} D^{\alpha} a(X_{k}) \left( X_{k} - X^c_{k} \right)^{\alpha} \rangle \,.
\end{split}
\end{equation}
Recall that we are dealing now with the group $- \left( \Xi^{1,3}_{k} - \frac{1}{2} \left(\Xi^{2,3}_{k} + \Xi^{3,3}_{k} \right) \right)$. Similarly as above, for $\Xi^{2,3}_{k}$ and $\Xi^{3,3}_{k}$, we have the terms 
\begin{equation*}
\begin{split}
\Xi^{2,3,2}_{k} &= h \sum_{|\alpha|=2} \int_0^1 (1-t) D^{\alpha} b\left(  X^c_{k} + t\left(\bar{X}^{c-,c}_{k} - X^c_{k} \right), U^{f,1}_{k} \right) dt \left( \bar{X}^{c-,c}_{k} - X^c_{k} \right)^{\alpha} \\
\Xi^{3,3,2}_{k} &= h \sum_{|\alpha|=2} \int_0^1 (1-t) D^{\alpha} b\left(  X^c_{k} + t\left(\bar{X}^{c+,c}_{k} - X^c_{k} \right), U^{f,2}_{k} \right) dt \left( \bar{X}^{c+,c}_{k} - X^c_{k} \right)^{\alpha} 
\end{split}
\end{equation*}
for which
\begin{equation*}
\begin{split}
\mathbb{E} \langle \Psi_{k} , \Xi^{2,3,2}_{k} \rangle &\leq h \varepsilon_{10} C_{a^{(2)}} \mathbb{E} |\Psi_{k}|^2 + h \frac{1}{\varepsilon_{10}} C_{a^{(2)}} \mathbb{E} \left| \bar{X}^{c-,c}_{k} - X^c_{k} \right|^4 \\
\mathbb{E} \langle \Psi_{k} , \Xi^{3,3,2}_{k} \rangle &\leq h \varepsilon_{11} C_{a^{(2)}} \mathbb{E} |\Psi_{k}|^2 + h \frac{1}{\varepsilon_{11}} C_{a^{(2)}} \mathbb{E} \left| \bar{X}^{c+,c}_{k} - X^c_{k} \right|^4 
\end{split}
\end{equation*}
for some $\varepsilon_{10}$, $\varepsilon_{11} > 0$, and we can again apply Lemma \ref{thm:AntitheticMean} to conclude that the fourth moments above have the correct order in $s$ and $h$. On the other hand, repeating the analysis for $\Xi^{1,3,1}_{k}$ above, we see that
\begin{equation*}
\begin{split}
\mathbb{E} &\langle \Psi_{k}, \Xi^{2,3,1,1}_{k} \rangle = h \mathbb{E} \langle \Psi_{k} , \sum_{|\alpha| = 1} D^{\alpha} a(X_{k}) \left( \bar{X}^{c-,c}_{k} - X_{k} \right)^{\alpha} \rangle \\
\mathbb{E} &\langle \Psi_{k}, \Xi^{3,3,1,1}_{k} \rangle = h \mathbb{E} \langle \Psi_{k} , \sum_{|\alpha| = 1} D^{\alpha} a(X_{k}) \left( \bar{X}^{c+,c}_{k} - X_{k} \right)^{\alpha} \rangle \,,
\end{split}
\end{equation*}
whereas
\begin{equation}\label{eq:proofMASGAaux}
\begin{split}
\mathbb{E} \langle \Psi_{k}, \Xi^{2,3,1,2}_{k} \rangle &= h \mathbb{E} \langle \Psi_{k} , \sum_{|\alpha| = 1} \left( D^{\alpha} a(X^c_{k}) - D^{\alpha} a(X_{k}) \right) \left( \bar{X}^{c-,c}_{k} - X^c_{k} \right)^{\alpha}  \rangle \\
&+ h \mathbb{E} \langle \Psi_{k} , \sum_{|\alpha| = 1} D^{\alpha} a(X_{k}) \left( X_{k} - X^c_{k} \right)^{\alpha} \rangle \\
\mathbb{E} \langle \Psi_{k}, \Xi^{3,3,1,2}_{k} \rangle &= h \mathbb{E} \langle \Psi_{k} , \sum_{|\alpha| = 1} \left( D^{\alpha} a(X^c_{k}) - D^{\alpha} a(X_{k}) \right) \left( \bar{X}^{c+,c}_{k} - X^c_{k} \right)^{\alpha}  \rangle \\
&+ h \mathbb{E} \langle \Psi_{k} , \sum_{|\alpha| = 1} D^{\alpha} a(X_{k}) \left( X_{k} - X^c_{k} \right)^{\alpha} \rangle  \,.
\end{split}
\end{equation}
Recall that
\begin{equation*}
\begin{split}
\mathbb{E} \langle \Psi_{k}, - \left( \Xi^{1,3}_{k} - \frac{1}{2} \left(\Xi^{2,3}_{k} + \Xi^{3,3}_{k} \right) \right) \rangle &= \mathbb{E} \langle \Psi_{k}, - \left( \Xi^{1,3,1}_{k} - \frac{1}{2} \left(\Xi^{2,3,1}_{k} + \Xi^{3,3,1}_{k} \right) \right) \rangle \\
&+ \mathbb{E} \langle \Psi_{k}, - \left( \Xi^{1,3,2}_{k} - \frac{1}{2} \left(\Xi^{2,3,2}_{k} + \Xi^{3,3,2}_{k} \right) \right) \rangle
\end{split}
\end{equation*}
and, due to our discussion above, $\mathbb{E} \langle \Psi_{k}, - \left( \Xi^{1,3,2}_{k} - \frac{1}{2} \left(\Xi^{2,3,2}_{k} + \Xi^{3,3,2}_{k} \right) \right) \rangle$ is of the correct order in $s$ and $h$. Hence we focus on
\begin{equation*}
\begin{split}
\mathbb{E} &\langle \Psi_{k}, - \left( \Xi^{1,3,1}_{k} - \frac{1}{2} \left(\Xi^{2,3,1}_{k} + \Xi^{3,3,1}_{k} \right) \right) \rangle \\
&= - h \mathbb{E} \langle \Psi_{k} , \sum_{|\alpha| = 1} D^{\alpha} a(X_{k}) \left( \bar{X}^{f,c}_{k} - \frac{1}{2}\left( \bar{X}^{c-,c}_{k} + \bar{X}^{c+,c}_{k} \right) \right)^{\alpha} \rangle \\
&- h \mathbb{E} \langle \Psi_{k} , \sum_{|\alpha| = 1} \left( D^{\alpha} a(X^c_{k}) - D^{\alpha} a(X_{k}) \right) \left( \bar{X}^{f,c}_{k} - \frac{1}{2}\left( \bar{X}^{c-,c}_{k} + \bar{X}^{c+,c}_{k} \right) \right)^{\alpha}  \rangle \,,
\end{split}
\end{equation*}
where the first term on the right hand side above comes from the expression $\mathbb{E} \langle \Psi_{k}, \Xi^{1,3,1,1}_{k} \rangle - \frac{1}{2} \left( \mathbb{E} \langle \Psi_{k}, \Xi^{2,3,1,1}_{k} \rangle + \mathbb{E} \langle \Psi_{k}, \Xi^{3,3,1,1}_{k} \rangle \right)$ and the second term comes from 
\begin{equation}\label{eq auxMASGAproof}
\mathbb{E} \langle \Psi_{k}, \Xi^{1,3,1,2}_{k} \rangle - \frac{1}{2} \left( \mathbb{E} \langle \Psi_{k}, \Xi^{2,3,1,2}_{k} \rangle + \mathbb{E} \langle \Psi_{k}, \Xi^{3,3,1,2}_{k} \rangle \right) \,.
\end{equation}
Note that each term in (\ref{eq auxMASGAproof}) was a sum of two terms, however, all the second terms in (\ref{eq auxMASGAproof}) cancelled out, since they were all of the same form, cf.\ (\ref{eq auxMASGAproof0}) and (\ref{eq:proofMASGAaux}). Now we will combine the first term on the right hand side of $\mathbb{E} \langle \Psi_{k}, - \left( \Xi^{1,3,1}_{k} - \frac{1}{2} \left(\Xi^{2,3,1}_{k} + \Xi^{3,3,1}_{k} \right) \right) \rangle$ with a term from a previous group. To this end, recall that
\begin{equation*}
\begin{split}
\mathbb{E} &\langle \Psi_{k}, \Xi^{1,2}_{k} - \frac{1}{2} \left(\Xi^{2,2}_{k} + \Xi^{3,2}_{k} \right)  \rangle \\
&= h \mathbb{E} \langle \Psi_{k} , \sum_{|\alpha|=1} D^{\alpha} a(X_{k}) \left( X^{f,f}_{k} - \frac{1}{2}\left( X^{c-,f}_{k} + X^{c+,f}_{k} \right) \right)^{\alpha} \rangle + h \mathbb{E} \langle \Psi_{k} , \hat{\Xi}^2_{k} \rangle \,,
\end{split}
\end{equation*}
where $h \mathbb{E} \langle \Psi_{k} , \hat{\Xi}^2_{k} \rangle$ is of the correct order in $s$ and $h$, and notice that we have
\begin{equation*}
\begin{split}
h \mathbb{E} &\langle \Psi_{k} , \sum_{|\alpha|=1} D^{\alpha} a(X_{k}) \left( X^{f,f}_{k} - \frac{1}{2}\left( X^{c-,f}_{k} + X^{c+,f}_{k} \right) \right)^{\alpha} \rangle  \\
&- h \mathbb{E} \langle \Psi_{k} , \sum_{|\alpha| = 1} D^{\alpha} a(X_{k}) \left( \bar{X}^{f,c}_{k} - \frac{1}{2}\left( \bar{X}^{c-,c}_{k} + \bar{X}^{c+,c}_{k} \right) \right)^{\alpha} \rangle \\
&= h \mathbb{E} \langle \Psi_{k} , \sum_{|\alpha| = 1} D^{\alpha} a(X_{k}) \left( X^{f,f}_{k} - \frac{1}{2}\left( X^{c-,f}_{k} + X^{c+,f}_{k} \right) - \left( \bar{X}^{f,c}_{k} - \frac{1}{2}\left( \bar{X}^{c-,c}_{k} + \bar{X}^{c+,c}_{k} \right) \right) \right)^{\alpha} \rangle \\
&= h \mathbb{E} \langle \Psi_{k} , \sum_{|\alpha| = 1} D^{\alpha} a(X_{k}) \left( \Psi_{k} \right)^{\alpha} \rangle \leq -Kh \mathbb{E} |\Psi_{k}|^2 \,. 
\end{split}
\end{equation*}
In the inequality above we used the fact that Assumption \ref{as diss} implies that for all $x$, $y \in \mathbb{R}^d$ we have $\langle y, \sum_{|\alpha| = 1}D^{\alpha}a(x)y^{\alpha} \rangle \leq - K|y|^2$. On the other hand, the first terms in (\ref{eq auxMASGAproof}) give
\begin{equation*}
\begin{split}
h \mathbb{E} &\langle \Psi_{k} , \sum_{|\alpha| = 1} \left( D^{\alpha} a(X^c_{k}) - D^{\alpha} a(X_{k}) \right) \left( \bar{X}^{f,c}_{k} - \frac{1}{2}\left( \bar{X}^{c-,c}_{k} + \bar{X}^{c+,c}_{k} \right) \right)^{\alpha}  \rangle \\
&\leq 2 h \varepsilon_{12} C_{a^{(1)}} \mathbb{E} |\Psi_{k}|^2 +  2 h \frac{1}{\varepsilon_{12}} C_{a^{(1)}} \mathbb{E} \left| \bar{X}^{f,c}_{k} - \frac{1}{2}\left( \bar{X}^{c-,c}_{k} + \bar{X}^{c+,c}_{k} \right) \right|^2 
\end{split}
\end{equation*}
for some $\varepsilon_{12} > 0$ to be chosen later. Now we use
\begin{lemma}\label{thm:AntitheticMeanMean}
	Under the assumptions of Lemma \ref{thm:AntitheticSubsampling}, there is a $C > 0$ such that for all $k \geq 1$ we have
	$\mathbb{E} \left| \bar{X}^{f,c}_{k} - \frac{1}{2}\left( \bar{X}^{c-,c}_{k} + \bar{X}^{c+,c}_{k} \right) \right|^2 \leq Ch^2/s^2$.
	\begin{proof}
		Notice that
		\begin{equation*}
		\begin{split}
		\mathbb{E} &\left| \bar{X}^{f,c}_{k} - \frac{1}{2}\left( \bar{X}^{c-,c}_{k} + \bar{X}^{c+,c}_{k} \right) \right|^2 \\
		&\leq \frac{1}{2} \mathbb{E} \left| X^{f,c-}_{k} - \frac{1}{2}\left( X^{c-,c-}_{k} + X^{c+,c-}_{k} \right) \right|^2 + \frac{1}{2} \mathbb{E} \left| X^{f,c+}_{k} - \frac{1}{2}\left( X^{c-,c+}_{k} + X^{c+,c+}_{k} \right) \right|^2
		\end{split}
		\end{equation*}
		and that both terms above correspond to antithetic estimators with respect to subsampling for coarse chains, hence Lemma \ref{thm:AntitheticSubsampling} applies.
	\end{proof}
\end{lemma}

Hence it only remains to deal with $\mathbb{E} \langle \Psi_{k} , - \left( \Xi^{1,4}_{k} - \frac{1}{2} \left(\Xi^{2,4}_{k} + \Xi^{3,4}_{k} \right) \right) \rangle$. We have
\begin{equation*}
\begin{split}
\Xi^{1,4}_{k} &= h \sum_{|\alpha|=1} D^{\alpha} b(X^c_{k}, U^f_{k+1}) \left( \bar{X}^{f,c}_{k} - X^c_{k} \right)^{\alpha} \\
&+ h \sum_{|\alpha|=2} \int_0^1 (1-t) D^{\alpha} b\left( X^c_{k} + t\left( \bar{X}^{f,c}_{k} - X^c_{k}  \right) , U^f_{k+1}\right) dt \left( \bar{X}^{f,c}_{k} - X^c_{k} \right)^{\alpha}
\end{split}
\end{equation*}
\begin{equation*}
\begin{split}
\Xi^{2,4}_{k} &= h \sum_{|\alpha|=1} D^{\alpha} b(X^c_{k}, U^{f,1}_{k+1}) \left( \bar{X}^{c-,c}_{k} - X^c_{k} \right)^{\alpha} \\
&+ h \sum_{|\alpha|=2} \int_0^1 (1-t) D^{\alpha} b\left( X^c_{k} + t\left( \bar{X}^{c-,c}_{k} - X^c_{k}  \right) , U^{f,1}_{k+1}\right) dt \left( \bar{X}^{c-,c}_{k} - X^c_{k} \right)^{\alpha}
\end{split}
\end{equation*}
\begin{equation*}
\begin{split}	
\Xi^{3,4}_{k} &= h \sum_{|\alpha|=1} D^{\alpha} b(X^c_{k}, U^{f,2}_{k+1}) \left( \bar{X}^{c+,c}_{k} - X^c_{k} \right)^{\alpha} \\
&+ h \sum_{|\alpha|=2} \int_0^1 (1-t) D^{\alpha} b\left( X^c_{k} + t\left( \bar{X}^{c+,c}_{k} - X^c_{k}  \right) , U^{f,2}_{k+1}\right) dt \left( \bar{X}^{c+,c}_{k} - X^c_{k} \right)^{\alpha} 
\end{split}
\end{equation*}
and hence
\begin{equation*}
\begin{split}
&\mathbb{E} \langle \Psi_{k} , - \left( \Xi^{1,4}_{k} - \frac{1}{2} \left(\Xi^{2,4}_{k} + \Xi^{3,4}_{k} \right) \right) \rangle = - h \mathbb{E} \langle \Psi_{k} , \sum_{|\alpha|=1} D^{\alpha}a(X^c_{k}) \left( \bar{X}^{f,c}_{k} - \frac{1}{2}\left( \bar{X}^{c-,c}_{k} + \bar{X}^{c+,c}_{k} \right) \right)^{\alpha} \rangle \\
&- h \mathbb{E} \langle \Psi_{k} , \sum_{|\alpha|=2} \int_0^1 (1-t) D^{\alpha} a\left( X^c_{k} + t\left( \bar{X}^{f,c}_{k} - X^c_{k}  \right) \right) dt \left( \bar{X}^{f,c}_{k} - X^c_{k} \right)^{\alpha} \rangle \\
&- h \mathbb{E} \langle \Psi_{k} , \sum_{|\alpha|=2} \int_0^1 (1-t) D^{\alpha} a\left( X^c_{k} + t\left( \bar{X}^{c-,c}_{k} - X^c_{k}  \right) \right) dt \left( \bar{X}^{c-,c}_{k} - X^c_{k} \right)^{\alpha} \rangle  \\
&- h \mathbb{E} \langle \Psi_{k} , \sum_{|\alpha|=2} \int_0^1 (1-t) D^{\alpha} a\left( X^c_{k} + t\left( \bar{X}^{c+,c}_{k} - X^c_{k}  \right) \right) dt \left( \bar{X}^{c+,c}_{k} - X^c_{k} \right)^{\alpha} \rangle \\
&\leq  h \varepsilon_{13} C_{a^{(1)}} \mathbb{E} |\Psi_{k}|^2 +   h \frac{1}{\varepsilon_{13}} C_{a^{(1)}} \mathbb{E} \left| \bar{X}^{f,c}_{k} - \frac{1}{2}\left( \bar{X}^{c-,c}_{k} + \bar{X}^{c+,c}_{k} \right) \right|^2 \\
&+ h \varepsilon_{14} C_{a^{(2)}} \mathbb{E} |\Psi_{k}|^2 +   h \frac{1}{\varepsilon_{14}} C_{a^{(2)}} \mathbb{E} \left| \bar{X}^{f,c}_{k} - X^c_{k} \right|^4 \\
&+ h \varepsilon_{15} C_{a^{(2)}} \mathbb{E} |\Psi_{k}|^2 +   h \frac{1}{\varepsilon_{15}} C_{a^{(2)}} \mathbb{E} \left| \bar{X}^{c-,c}_{k} - X^c_{k} \right|^4 + h \varepsilon_{16} C_{a^{(2)}} \mathbb{E} |\Psi_{k}|^2 +   h \frac{1}{\varepsilon_{16}} C_{a^{(2)}} \mathbb{E} \left| \bar{X}^{c+,c}_{k} - X^c_{k} \right|^4 \,,
\end{split}
\end{equation*}
for some $\varepsilon_{13}$, $\varepsilon_{14}$, $\varepsilon_{15}$, $\varepsilon_{16} > 0$. Using Lemmas \ref{thm:AntitheticMean} and \ref{thm:AntitheticMeanMean}, we see that all the fourth moments above have the correct order in $s$ and $h$. This concludes our estimates for $\mathbb{E} \langle \Psi_{k} , \Xi_{k} \rangle$. 
\end{proof}

\section*{Acknowledgement}
A part of this work was completed while MBM was affiliated to the University of Warwick and supported by the EPSRC grant EP/P003818/1.

\bibliographystyle{abbrv}

\bibliography{Matbib_Antithetic} 
\end{document}